\documentclass[11 pt]{article}
\usepackage{style}

\title{\LARGE\bfseries Achieving $\epsilon^{-2}$ Sample Complexity for Single-Loop Actor-Critic under Minimal Assumptions}
\author{Ishaq Hamza\thanks{Much of this work was carried out while the author was a research intern at the Edwardson School of Industrial Engineering, Purdue University.}\; and Zaiwei Chen\textsuperscript{$\dagger$} \\ 
\small \textsuperscript{$*$}\textit{IISc,} \href{mailto:ishaqhamza@iisc.ac.in}{\textit{ishaqhamza@iisc.ac.in}}\,, \ \ 
\textsuperscript{$\dagger$}\textit{Purdue IE,} \href{mailto:chen5252@purdue.edu}{\textit{chen5252@purdue.edu}}}

\date{\vspace{-0.4 in}}
\begin{document}
\maketitle
\allowdisplaybreaks

\begin{abstract}
In this paper, we establish last-iterate convergence rates for off-policy actor--critic methods in reinforcement learning. In particular, under a \textit{single-loop, single-timescale implementation} and a broad class of policy updates, including approximate policy iteration and natural policy gradient methods, we prove the first $\smash{\tilde{\mathcal{O}}(\epsilon^{-2})}$ sample complexity guarantee for finding an $\epsilon$-optimal policy under \textit{minimal assumptions}, namely, the existence of a policy that induces an irreducible Markov chain. This stands in stark contrast to the existing literature, where an $\smash{\tilde{\mathcal{O}}(\epsilon^{-2})}$ sample complexity is achieved only through nested-loop updates and/or under strong, algorithm-dependent assumptions on the policies, such as uniform mixing and uniform exploration.

Technically, to address the challenges posed by the coupled update equations arising from the single-loop implementation, as well as the potentially unbounded iterates induced by off-policy learning, our analysis is based on a coupled Lyapunov drift framework. Specifically, we establish a geometric convergence rate for the actor and an $\smash{\tilde{\mathcal{O}}(1/T)}$ convergence rate for the critic, and combine the two Lyapunov drift inequalities through a cross-domination property. We believe this analytical framework is of independent interest and may be applicable to other coupled iterative algorithms with unbounded iterates.
\end{abstract}

\section{Introduction}\label{sec:intro}
Reinforcement learning (RL) has become increasingly impactful in solving sequential decision-making problems \cite{sutton2018reinforcement}, ranging from game-playing AI \cite{silver2017mastering} to large language models \cite{ouyang2022training}. Mathematically, RL can be viewed as a data-driven framework for solving Markov decision processes (MDPs) \cite{puterman2014markov}. Two classical approaches to solving MDPs are value iteration and policy iteration, which, in model-free RL, correspond to value-space methods such as $Q$-learning \cite{watkins1992q} and policy-space methods such as actor--critic \cite{konda1999actor}, respectively.

For $Q$-learning, it has long been established that an $\smash{\tilde{\mathcal{O}}(\epsilon^{-2})}$ sample complexity is achievable as long as the stationary behavior policy induces an irreducible and aperiodic Markov chain; see \cite{chen2024lyapunov,li2020sample,qu2020finite}, among many others. The aperiodicity assumption has recently been relaxed in \cite{nanda2025minimal,chandak2022concentration,haque2024stochastic}. Irreducibility is essentially minimal for exploration in the sense that, if this assumption is violated, then there exists at least one state that is visited only finitely many times during the learning process.

For actor--critic methods, existing results achieving an $\tilde{\mathcal{O}}(\epsilon^{-2})$ sample complexity mostly rely on a nested-loop implementation \cite{ju2025auto,xiao2022convergence,kumar2023sample,li2025policy}, where the actor is updated in an outer loop while being held fixed in an inner loop for critic updates. In contrast, practical actor--critic algorithms are typically implemented using a more natural single-loop structure, in which both the actor and the critic are updated simultaneously at each iteration \cite{mnih2016asynchronous,haarnoja2018sac,lillicrap2015continuous}. Moreover, even under a nested-loop implementation, existing results often impose strong, algorithm-dependent assumptions, such as requiring all policies encountered along the algorithm trajectory, or simply all policies, to induce irreducible and aperiodic Markov chains, with mixing factors uniformly bounded above and stationary distributions uniformly bounded away from zero \cite{qiu2019finite,alacaoglu2022natural,chen2023actorcritic,olshevsky2023small}, with only a few exceptions \cite{ju2025auto,li2025policy,chen2025approximate}. A more detailed literature review is presented in Section~\ref{subsec:literature}.

Since value-space and policy-space methods are two fundamental pillars of RL, there is little reason to believe that the latter are inherently less sample efficient or require stronger assumptions than the former. This motivates the following question: \textit{Can actor--critic methods achieve $\tilde{\mathcal{O}}(\epsilon^{-2})$ sample complexity with a single-loop implementation under minimal assumptions, in particular, the existence of a policy that induces an irreducible Markov chain on the state space?} In this paper, we answer this question in the affirmative.

\subsection{Main Contributions} 
We consider a single-loop, single-timescale actor--critic algorithm in which the actor employs either incremental approximate policy iteration or incremental natural policy gradient updates, while the critic employs off-policy temporal-difference (TD) learning. 

$\bullet$ \textbf{Last-Iterate Convergence Rates under Minimal Assumptions}: We assume the existence of a policy that induces an irreducible Markov chain over the state trajectory. This policy need not be visited by the algorithm's trajectory and thus constitutes a purely existential assumption on the underlying MDP. Under this assumption, we show that our single-loop, single-timescale actor--critic algorithm achieves an $\tilde{\mathcal{O}}(\epsilon^{-2})$ sample complexity. To the best of our knowledge, this is the first result attaining near-minimax sample complexity (with respect to $\epsilon$) without relying on a nested-loop implementation, timescale separation, or strong, potentially algorithm-dependent assumptions.

$\bullet$ \textbf{Technical Contributions}: 
The main challenges in the analysis arise from the coupling between the actor and critic updates, due to the single-loop structure of the algorithm, and from the potentially unbounded iterates induced by off-policy learning in the critic. To overcome these challenges, we employ a coupled Lyapunov approach. Specifically, we construct one Lyapunov function for the actor and another for the critic, and establish a negative drift inequality for each while ensuring a cross-domination property. This structure allows the two drift inequalities to be combined to obtain the desired convergence rates. More details are provided in Section~\ref{sec:Proof}. While this approach is inspired by small-gain theory in the analysis of dynamical systems \cite{khalil2002nonlinear}, to the best of our knowledge, the analytical framework we develop is the first to handle Markovian noise, time-varying targets, unbounded iterates, and coupled systems in a unified manner.

\subsection{Related Literature}
\label{subsec:literature}

Actor--critic methods \cite{konda1999actor} are policy-space RL algorithms that iteratively perform policy evaluation and policy improvement. Policy evaluation is typically carried out using Monte Carlo methods or TD-learning \cite{sutton1988learning}, while policy improvement is performed via variants of policy gradient methods \cite{sutton1999policy,kakade2001natural} or approximate policy iteration \cite{bertsekas2011approximate}. Asymptotic convergence guarantees for actor--critic methods have been established for decades \cite{konda1999actor,borkar1997actor}. Motivated by the need for sample efficiency in modern large-scale applications, recent work has increasingly focused on finite-time and sample-complexity guarantees. Existing non-asymptotic analyses can be broadly divided into two categories: nested-loop implementations, where the actor remains fixed while the critic performs multiple updates, and single-loop implementations, where the actor and critic are updated simultaneously at each iteration.

\textbf{Finite-Time Analysis of Nested-Loop Actor--Critic.}
For nested-loop actor--critic methods, $\tilde{\mathcal{O}}(\epsilon^{-2})$ sample complexity has been achieved in both tabular and function approximation settings \cite{xu2020improving,chen2025approximate,alacaoglu2022natural,kumar2023sample,lan2021policy,xiao2022convergence,ganesh25a}. However, most existing results rely on strong uniform ergodicity assumptions on the policies generated by the algorithm, or simply on all policies \cite{xu2020improving,alacaoglu2022natural,kumar2023sample,ganesh25a,tian2023neural}, with only a few exceptions \cite{li2025policy,chen2025approximate,ju2025auto}. Such algorithm-dependent assumptions introduce a form of circularity: verifying them typically depends on properties of the policy sequence itself, which is generated by the algorithm whose convergence analysis relies on these conditions. We do not make such assumptions.

\textbf{Finite-Time Analysis of Single-Loop Actor--Critic.} Motivated by empirical successes, we study actor--critic methods with a single-loop implementation. Depending on whether the actor and critic stepsizes are of different orders, such methods may have either a two-timescale structure \cite{wu2020finite,khodadadian2021finite,wang2024singleloop} or a single-timescale structure \cite{chen2023actorcritic,olshevsky2023small,chen2023games}. Regardless of timescale separation, existing results either lack global convergence guarantees \cite{olshevsky2023small,wu2020finite}, rely on strong assumptions \cite{kumar2026single,kumarconvergence}, or yield suboptimal sample complexity bounds \cite{khodadadian2021finite,kumarconvergence}. Our work establishes near-optimal sample complexity guarantees with respect to $\epsilon$ without timescale separation and under substantially weaker assumptions. A detailed comparison is provided in Table~\ref{table}.

Due to the computational limitations of large state-action spaces, popular RL algorithms typically rely on function approximation \cite{sutton2018reinforcement}. While the tabular setting is a special case of function approximation, we emphasize that existing results on actor--critic with function approximation \cite{wu2020finite,xu2021sample,olshevsky2023small,ganesh25a} are not applicable in our setting, as they require spectral assumptions on the feature map that do not hold in the tabular case; see, e.g., \cite[Assumption~4.1]{wu2020finite}, \cite[Assumption~6]{olshevsky2023small}, and \cite[Assumption~4]{ganesh25a}. These assumptions and their incompatibility with the tabular setting are detailed in Appendix~\ref{app:comparison}.

\begin{table*}[t]
  \centering
  \caption{Sample complexity results for single-loop actor--critic algorithms.}
  \label{table}

  {\small
  \setlength{\tabcolsep}{2.3pt}
  \renewcommand{\arraystretch}{1.1}
  \renewcommand{\tabularxcolumn}[1]{m{#1}} 

  \begin{tabularx}{\textwidth}{|
    >{\hsize=1.00\hsize\linewidth=\hsize\centering\arraybackslash}X|
    >{\hsize=1.15\hsize\linewidth=\hsize\centering\arraybackslash}X|
    >{\hsize=1.05\hsize\linewidth=\hsize\centering\arraybackslash}X|
    >{\hsize=0.80\hsize\linewidth=\hsize\centering\arraybackslash}X|
    >{\hsize=1.05\hsize\linewidth=\hsize\centering\arraybackslash}X|
    >{\hsize=1.25\hsize\linewidth=\hsize\centering\arraybackslash}X|
    >{\hsize=0.70\hsize\linewidth=\hsize\centering\arraybackslash}X|}
    \hline
    Work &
    Complexity &
    \makecell[c]{Global\\Convergence} &
    \makecell[c]{Last\\Iterate} &
    \makecell[c]{Markovian\\Sampling} &
    Assumptions &
    \makecell[c]{Single\\Timescale} \\
    \hline

    \cite{khodadadian2021finite}
    & $\tilde{\mathcal{O}}(\epsilon^{-6})$
    & \ding{51}
    & \ding{51}
    & \ding{51}
    & Ergodicity
    & \ding{55} \\
    \hline

    \cite{xu2020twotimescaleNAC}
    & $\tilde{\mathcal{O}}(\epsilon^{-2.5})$
    & \ding{55}
    & \ding{51}
    & \ding{51}
    & Unif. mix. rates
    & \ding{55} \\
    \hline

    \cite{olshevsky2023small}
    & $\mathcal{O}(\epsilon^{-2})$
    & \ding{55}
    & \ding{55}
    & \ding{55}
    & \makecell[c]{Unif. mix. rates\\Exploration}
    & \ding{55} \\
    \hline

    \cite{fu2021singletimescale}
    & $\tilde{\mathcal{O}}(\epsilon^{-4})$
    & \ding{51}
    & \ding{55}
    & \ding{55}
    & Unif. mix. rates
    & \ding{51} \\
    \hline

    \cite{kumarconvergence}
    & $\mathcal{O}(\epsilon^{-3})$
    & \ding{51}
    & \ding{51}
    & \ding{55}
    & Exploration
    & \ding{51} \\
    \hline

    \cite{kumar2026single}
    & $\mathcal{O}(\epsilon^{-2})$
    & \ding{51}
    & \ding{51}
    & \ding{55}
    & Exploration
    & \ding{51} \\
    \hline

    \textbf{This work}
    & $\tilde{\mathcal{O}}(\epsilon^{-2})$
    & \ding{51}
    & \ding{51}
    & \ding{51}
    & Irreducibility
    & \ding{51} \\
    \hline
  \end{tabularx}
  }

  \vspace{4pt}
  \begin{minipage}{\textwidth}
  \footnotesize
  \raggedright
  Some works in the table study the function approximation setting. Although the tabular setting is a special case of function approximation, their results do not directly transfer to the tabular setting because they impose assumptions (labeled ``Exploration'' in the table) that are not satisfied in this regime. ``Unif. mix. rates'' refers to a strong ergodicity assumption requiring the same mixing rate across a class of policies. See Appendix \ref{app:comparison} for details on these assumptions.
  \end{minipage}

\end{table*}

\textbf{Off-Policy Learning.}
Off-policy TD-learning refers to TD-learning where the target policy, whose value function we aim to estimate, differs from the behavior policy used to collect samples \cite{degris2012off}. Importance sampling is typically used to correct this discrepancy \cite{precup2001,espeholt2018impala,munos2016safe}. Off-policy actor--critic methods have been studied in the literature \cite{OffPolicyNAC,khodadadian2021finite2,chen2025approximate}, achieving $\tilde{\mathcal{O}}(\epsilon^{-2})$ sample complexity under nested-loop implementations. In contrast, in our single-loop actor--critic method, the off-policy TD-learning component must track a time-varying target, introducing challenges not addressed in existing off-policy actor--critic analyses.

\textbf{In summary,} the existing literature can be broadly divided into two categories: (i) analyses that achieve strong finite-time guarantees but rely on nested-loop implementations, and (ii) single-loop analyses that typically yield weaker guarantees, such as convergence to stationary points or global convergence with suboptimal rates. Moreover, most existing results rely on restrictive sampling or approximation assumptions. In contrast, we establish near order-optimal last-iterate finite-time guarantees for a single-loop algorithm under substantially weaker assumptions.

\subsection{Preliminaries}
An infinite-horizon discounted MDP \cite{puterman2014markov} is defined by a tuple
$(\mathcal{S}, \mathcal{A}, p, \mathcal{R}, \gamma)$, where $\mathcal{S}$ is the
state space of size $n$, $\mathcal{A}$ is the action space of size $m$, $p$ is the transition
probability kernel, $\mathcal{R}:\mathcal{S}\times\mathcal{A}\to[0,1]$ is the reward
function, and $\gamma\in(0,1)$ is the discount factor. Given a policy $\pi$, its $Q$-function
$Q^\pi:\mathcal{S}\times\mathcal{A}\to\mathbb{R}$ is defined as
$Q^\pi(s,a)=\mathbb{E}\left[\sum_{t=0}^\infty \gamma^t \mathcal{R}(S_t,A_t)
\mid S_0=s, A_0=a\right]$ for all $(s,a)\in\mathcal{S}\times\mathcal{A}$. Note that $Q^\pi$ can be alternatively viewed as a vector in $\mathbb{R}^{mn}$. The goal
is to learn an optimal policy $\pi^*$ such that its associated $Q$-function, denoted by $Q^*$, is maximized.

To solve an MDP, existing methods can be broadly divided into two categories: value-space methods, such as value iteration, and policy-space methods, such as policy iteration, policy gradient \cite{sutton1999policy}, and natural policy gradient \cite{kakade2001natural}. When the model parameters, e.g., the transition kernel and the reward function, are unknown, one must develop data-driven methods to learn an optimal policy, which corresponds to model-free RL. In this case, value-space methods translate to $Q$-learning \cite{watkins1992q}, and policy-space methods translate to the popular actor--critic framework, in which an agent iteratively performs policy evaluation and policy improvement \cite{konda1999actor}.

\section{Main Results}\label{sec:main_results}

This section presents our main results. Specifically, our single-loop actor--critic algorithm and its finite-time analysis are presented in Sections \ref{subsec:algo} and \ref{subsec:theory}. The proof of the main theorem is presented in Section \ref{sec:Proof}, with the proofs of technical lemmas deferred to the appendix.

\subsection{Algorithm}\label{subsec:algo}

Our single-loop, single-timescale off-policy actor--critic method is presented in Algorithm~\ref{alg}, where the actor and critic are updated simultaneously based on a single trajectory of Markovian samples. The actor incrementally updates the policy toward $\tilde{\pi}_t$, which is computed from the current policy $\pi_t$, the current $Q$-function estimate $Q_t$, and a temperature parameter $\tau_t$ through a generic update map $G(\cdot)$. This update map can correspond to either approximate policy iteration or natural policy gradient, as discussed shortly.

The critic performs off-policy TD-learning. To correct the distribution mismatch between the behavior and target policies, we propose two approaches. The first uses importance sampling (IS), where $\rho_{t+1}$ denotes the importance-sampling ratio. The second directly computes the expected $Q$-function at the next state $S_{t+1}$; we refer to this variant as expected temporal difference (ETD). In both cases, $\Delta_t$ denotes the temporal-difference error, which is then used in the TD-learning update. We next describe the actor and critic in more detail.

\begin{algorithm}[h]
  \caption{Single-Loop Off-Policy Actor--Critic}
  \label{alg}
  \begin{algorithmic}[1]
    \STATE {\bfseries Input:} Initialization: $\pi_0$ and $Q_0=\mathbf{0}$, a single trajectory $\{(S_t,A_t)\}_{t\geq 0}$ generated by a behavior policy $\pi_b$, a sequence $\{(\tau_t, \omega_t, \alpha_t)\}$ of update parameters, and a choice of critic $\in\left\{\mathrm{IS}, \mathrm{ETD}\right\}$.
    
    \FOR{$t=0,1,2,\dots$}
        \STATE $\pi_{t+1} = (1-\omega_t)\pi_t + \omega_t \tilde{\pi}_t$, where $\tilde{\pi}_{t} = G(\pi_t, Q_t,\tau_t)$\hfill{$\triangleright$ Actor}
        
        \IF{critic = $\mathrm{IS}$}
            \STATE $\rho_{t+1} = {\pi_{t+1}(A_{t+1}\mid S_{t+1})}/{\pi_b(A_{t+1}\mid S_{t+1})}$
            \STATE $\Delta_t = \mathcal{R}(S_t,A_t) + \gamma \rho_{t+1} Q_t(S_{t+1}, A_{t+1}) - Q_t(S_t, A_t)
            $
        
        \ELSIF{critic = $\mathrm{ETD}$}
            \STATE $\Delta_t = \mathcal{R}(S_t,A_t) + \gamma \sum_{a\in\mathcal{A}}\pi_{t+1}(a \mid S_{t+1}) Q_t(S_{t+1}, a) - Q_t(S_t, A_t)$
        \ENDIF
        \STATE Update $Q_t$ according to 
        \begin{align*}
        Q_{t+1}(s,a)=
        \begin{dcases}
        Q_t(s,a)+\alpha_t \Delta_t, & \text{if } (s,a)=(S_t,A_t),\\
        Q_t(s,a), & \text{otherwise}.\tag*{$\triangleright$ Critic}
        \end{dcases}
        \end{align*}
    \ENDFOR
  \end{algorithmic}
\end{algorithm}

\textbf{The Actor for Policy Improvement.}
In each iteration of Algorithm~\ref{alg}, the agent computes $\tilde{\pi}_t$ via the update map $G(\cdot)$ and then updates its policy by taking an incremental step toward $\tilde{\pi}_t$ with stepsize $\omega_t$. As for the function $G$, popular choices include natural policy gradient and approximate policy iteration with $\epsilon$-greedy or softmax updates, as presented below.

\textit{Natural Policy Gradient:} Define the function $G$ as
    \begin{align*}
        [G(\pi,Q,\tau)](s,a)
        =\frac{\pi(a| s)\exp(Q(s,a)/\tau)}
        {\sum_{a'}\pi(a'|s)\exp(Q(s,a')/\tau)},\quad\forall\,(s,a).
    \end{align*}
    Note that if the critic estimate is accurate, i.e.,
    $Q_t=Q^{\pi_t}$, the update $\tilde{\pi}_t=G(\pi_t,Q_t,\tau_t)$ coincides with
    the natural policy gradient update \cite{agarwal2021theory}. 

    \textit{Approximate Policy Iteration with Softmax Update:} Define the function $G$ as
    \begin{align*}
        [G(\pi,Q,\tau)](s,a)
        =\frac{\exp(Q(s,a)/\tau)}
        {\sum_{a'}\exp(Q(s,a')/\tau)},\quad\forall\,(s,a),
    \end{align*}
    which corresponds to a ``soft'' version of policy iteration.

\textit{Approximate Policy Iteration with $\epsilon$-Greedy Update:}
    Define the function $G$ as
\begin{align*}
    [G(\pi,Q,\tau)](s,a)=\begin{dcases}
        \tau/m+1-\tau,& a=\arg\max_{a'\in\mathcal{A}} Q(s,a'),\\
        \tau/m,&\text{otherwise,}
    \end{dcases}
\end{align*}
where the tie-breaking rule can be arbitrary. This is the usual $\epsilon$-greedy update with exploration level $\epsilon=\tau$. We use the notation $\tau$ only to keep the temperature parameter consistent across the three actor updates. Although $\epsilon$ does not appear in the definition, the name $\epsilon$-greedy is retained for consistency with the literature \cite{mnih2015human}.

Note that as the tunable parameter $\tau$ approaches zero, all three updates reduce to policy iteration \cite{kakade2001natural}, thereby enabling a unified analysis. For ease of terminology, we will uniformly refer to $\tau_t$ in Algorithm~\ref{alg} as the temperature parameter.

\textbf{The Critic for Policy Evaluation.}
Our critic update is an off-policy variant of TD-learning \cite{sutton1988learning}, motivated by both practical and theoretical considerations. From a practical perspective, in safety-critical applications such as autonomous driving, healthcare, and clinical trials, data collection can be costly and/or high-risk \cite{munos2016safe,sutton2016emphatic}. In such settings, one must rely on historical data collected under a different policy. From a theoretical perspective, off-policy learning facilitates exploration by avoiding algorithm-dependent assumptions, such as requiring all policies encountered along the algorithm trajectory to satisfy uniform exploration properties \cite{khodadadian2021finite,wu2020finite,chen2023actorcritic}. While off-policy learning is both practically and theoretically justified, it introduces a mismatch between the behavior policy $\pi_b$ and the target policy $\pi_{t+1}$. Recall that the goal of the critic is to estimate the $Q$-function $Q^{\pi_{t+1}}$ associated with $\pi_{t+1}$, while samples are generated from $\pi_b$. To correct this mismatch, we consider two approaches: one based on importance sampling and the other based on directly computing the expectation.

Given the current state-action pair $(S_t,A_t)$, the temporal difference $\Delta_t$ is constructed as a conditionally unbiased estimator of the Bellman error
\begin{align*}
    \mathcal{R}(S_t,A_t)
    + \gamma 
    \underbrace{
    \sum_{s'\in\mathcal{S}} p(s'\mid S_t,A_t)
    \sum_{a'\in\mathcal{A}} \pi_{t+1}(a'\mid s') Q_t(s',a')
    }_{:=E_1}
    - Q_t(S_t,A_t).
\end{align*}
The key quantity to estimate is $E_1$. Since $S_{t+1}\sim p(\cdot\mid S_t,A_t)$, one can construct an unbiased estimator of $E_1$ using $A_{t+1}\sim \pi_b(\cdot\mid S_{t+1})$ via importance sampling, which corrects the mismatch between $\pi_{t+1}$ and $\pi_b$ \cite{munos2016safe, espeholt2018impala}. This yields the IS-based approach in Lines 5--6 of Algorithm~\ref{alg}. Alternatively, since $\pi_{t+1}$ is known, one can directly compute an unbiased estimator of $E_1$ given $S_{t+1}$ as $\sum_{a'\in\mathcal{A}} \pi_{t+1}(a'\mid S_{t+1}) Q_t(S_{t+1},a')$, resulting in the ETD-based approach in Line 8 of Algorithm~\ref{alg}.

Comparing the two approaches, the IS-based method is computationally more efficient, as it avoids explicitly computing expectations over the action space, which can be costly when the action space is large. However, its analysis is more challenging due to nonlinear updates induced by the product of the importance sampling ratio and the $Q$-function, as well as the possibility of unbounded iterates. In contrast, the ETD-based approach is easier to analyze, as directly computing the expectation avoids the complications introduced by importance sampling. The trade-off is a higher computational cost.

\subsection{Convergence Rates}
\label{subsec:theory}

We now state the only exploration assumption used in our analysis. 

\begin{assumption}
\label{assum:behavior}
There exists a policy $\pi$ such that the Markov chain $\{S_t\}$ induced by $\pi$ is irreducible.
\end{assumption}

\begin{remark}
By choosing the behavior policy to have full support, i.e.,
$\pi_{b,\min}:=\min_{s,a}\pi_b(a\mid s)>0$, Assumption~\ref{assum:behavior}
implies that the Markov chain $\{S_t\}$ induced by $\pi_b$ is irreducible.
Hence, it admits a unique stationary distribution $\mu_b\in\Delta^n$ (where $\Delta^d$ denotes the $d$-dimensional probability simplex) satisfying
$\mu_{b,\min}:=\min_{s\in\mathcal{S}}\mu_b(s)>0$. A proof of this claim is
provided in Appendix~\ref{app:assum}. Moreover, by applying an aperiodicity
transformation, also known as the Schweitzer transformation
\cite{schweitzer1971iterative}, we may assume \textit{without loss of generality}
that the Markov chain $\{S_t\}$ induced by $\pi_b$ is aperiodic. Consequently, the chain $\{S_t\}$ mixes at a geometric rate \cite[Theorem 4.9]{levin2017markov}: there exists a constant $\sigma_b\in(0,1)$ such that $\max_{s\in\mathcal{S}} \mathrm{d}_{\mathrm{TV}}\big(P_{\pi_b}^t(s,\cdot),\mu_b(\cdot)\big) \leq 2\sigma_b^t$, where $P_{\pi_b}$ denotes the state transition matrix.\footnote{The total variation distance between two probability distributions $p_1,p_2$ on a finite sample space $\Omega$ is defined as $\mathrm{d}_{\mathrm{TV}}(p_1,p_2)=\frac{1}{2}\sum_{\omega\in\Omega}\left|p_1(\omega)-p_2(\omega)\right|$.}
\end{remark}

Assumption~\ref{assum:behavior} is purely existential: the policy $\pi$ need not be known, computed, or visited by the algorithm. Thus, the assumption is independent of the algorithmic trajectory and instead captures an intrinsic explorability property of the underlying MDP. In Appendix~\ref{app:assum}, we show that this condition is also necessary for exploration in the following sense: if it fails, then no algorithm can generate a single sample trajectory that visits every state infinitely often. Therefore, Assumption~\ref{assum:behavior} is minimal for state-space exploration.

This contrasts sharply with the assumptions commonly imposed in the actor--critic literature, for both nested-loop and single-loop implementations. Existing works often require all policies generated along the algorithm trajectory, or even all policies in the policy class, to induce irreducible and aperiodic Markov chains with uniformly bounded mixing constants, together with some form of uniform exploration \cite{wu2020finite,chen2023actorcritic,olshevsky2023small}. Such conditions are not only stronger than Assumption~\ref{assum:behavior}, but also algorithm dependent: verifying them requires controlling the very policy sequence whose convergence is being analyzed. By imposing only the existential condition above, we decouple exploration from the algorithm trajectory and avoid this circularity. A detailed comparison with common assumptions is provided in Appendix~\ref{app:comparison}.

For our main result, we consider stepsize sequences $\left\{\alpha_t,\omega_t\right\}$ of the form
$\alpha_t = \alpha/(t+h)^\eta$ and $\omega_t = \omega/(t+h)^\eta$, where
$\alpha,\omega,h>0$ and $\eta\in[0,1]$ are constants. Let $C_r=\omega/\alpha\in(0,1)$ denote the ratio between the actor and critic stepsizes. Note that the actor and critic stepsizes are of the same order, reflecting the single-timescale nature of our method. This feature offers numerical stability and simplicity in implementation, making single-timescale actor--critic methods widely popular in practice \cite{peters2008natural,schulman2015trust,schulman2017ppo,mnih2016asynchronous,haarnoja2018sac}.

Next, we state a condition on the choice of the temperature parameter $\tau_t$ in
Algorithm~\ref{alg}.

\begin{condition}
\label{cond:temp}
Let $\tau \ge 0$. Depending on the specific actor update rules and the stepsizes $\{\alpha_t,\omega_t\}$, the sequence
$\{\tau_t\}$ is chosen such that the following conditions are satisfied.
\begin{enumerate}[(1)]
\item \textit{Natural Policy Gradient:}
$
    \tau_t \leq \tau(t+h)^{-\eta/2}/\log(1/\min_s \max_a \pi_t(a| s)).
$
\item \textit{Exponential Softmax:} $\tau_t \leq \tau(t+h)^{-\eta/2} / \log(m)$.
    \item \textit{$\epsilon$-Greedy:} $\tau_t \le \tau(t+h)^{-\eta/2} / (2\gamma\infnorm{Q_t})$.
\end{enumerate}
\end{condition}

Let $z_t$ be the mixing time of the Markov chain $\{S_t\}$ under $\pi_b$ with precision $\omega_t$, defined as $z_t := \min\left\{k\geq 1\mid \max_{s\in\mathcal{S}}\mathrm{d}_{\mathrm{TV}}\big(P_{\pi_b}^{k-1}(s,\cdot),\mu_b(\cdot)\big) \leq \omega_t/2\right\}$. When using a constant stepsize $\omega_t\equiv \omega$, we simply write $z_\omega$. We also define a related quantity $K = K_{\omega,h,\eta} := \min\{t\in\mathbb{N}\mid t\geq z_t\}$. Due to geometric mixing, this quantity is well defined and finite for our choice of stepsizes.

\begin{theorem}
\label{thm:main}
Consider $\{\pi_t\}$ generated by Algorithm~\ref{alg}. Under Assumption~\ref{assum:behavior}, define
\begin{align}\label{def:M_constant}
    M_\mathrm{critic}:=
\begin{dcases}
mn(1-\gamma)^{-3}\pi_{b,\min}^{-3}\mu_{b,\min}^{-1}, & \mathrm{critic}=\mathrm{IS},\\
(1-\gamma)^{-2}, & \mathrm{critic}=\mathrm{ETD}.
\end{dcases}
\end{align}
Then there exists a constant threshold $\tilde{C}_{r}=\tilde{C}_{r}(n,m,p,\gamma,\pi_b)>0$ such that when using small enough stepsizes, for any $C_r\leq\tilde{C}_{r}$, and any temperature sequence $\{\tau_t\}$ satisfying Condition~\ref{cond:temp}, the following bounds hold for all $T\geq K$, where $\mathrm{MSE}_T:=\mathbb{E}\|Q^* - Q^{\pi_T}\|_\infty^2$.
\begin{enumerate}[(1)]
\item Under constant stepsizes, i.e., $\eta=0$, we have
\begin{align*}
\mathrm{MSE}_T
\leq& \underbrace{
\frac{3}{(1-\gamma)^2}\left(1 - \frac{\omega(1-\gamma)}{2}\right)^{T-K}
}_{\text{optimization bias}}
+\underbrace{\frac{12\tau^2}{(1-\gamma)^4}}_{\text{temperature error}}
+\underbrace{
\frac{432M_\mathrm{critic}\omega z_\omega}
{(1-\gamma)C_r^2}
}_{\text{stochastic error}}.
\end{align*}

\item Under harmonic stepsizes, i.e., $\eta=1$, we have
\begin{align*}
\mathrm{MSE}_T
\leq\frac{3}{(1-\gamma)^2}\left(\frac{K+h}{T+h}\right)^{\frac{\omega(1-\gamma)}{2}}
+\begin{dcases}
\frac{8M'_{\mathrm{critic},T}\omega}{[2-\omega(1-\gamma)](T+h)^{\frac{\omega(1-\gamma)}{2}}}, & \omega < \frac{2}{1-\gamma},\\
\frac{M'_{\mathrm{critic},T}\omega\log(T+h)}{T+h}, & \omega = \frac{2}{1-\gamma},\\
\frac{8eM'_{\mathrm{critic},T}\omega}{[\omega(1-\gamma)-2](T+h)}, & \omega > \frac{2}{1-\gamma},
\end{dcases}
\end{align*}
where $M'_{\mathrm{critic},T}:={6\tau^2}/{(1-\gamma)^3}+{216M_\mathrm{critic}z_T\omega}/{C_r^2}$.

\item Under polynomial stepsizes, i.e., $\eta\in(0,1)$, we have
\begin{align*}
\mathrm{MSE}_T
\leq& \frac{3\exp\left[-\frac{\omega(1-\gamma)}{2(1-\eta)}
\Big((T+h)^{1-\eta}-(K+h)^{1-\eta}\Big)\right]}{(1-\gamma)^2}
+\frac{4M'_{\mathrm{critic},T}}{(1-\gamma)(T+h)^\eta}.
\end{align*}
\end{enumerate}
\end{theorem}

Recall that $Q^*$ is the $Q$-function of an optimal policy, whereas $Q^{\pi_T}$ is the $Q$-function of the last policy iterate $\pi_T$. Thus, the left-hand side measures the mean-square optimality gap of $\pi_T$. On the right-hand side, the bounds decompose into three terms (for the diminishing-stepsize cases, after expanding $M'_{\mathrm{critic},T}$): the optimization bias, the error induced by a nonzero temperature, and the stochastic error. For approximate policy iteration with either exponential softmax or $\epsilon$-greedy updates, setting $\tau=0$ yields the strongest form of the bound. This is possible because exploration is handled by off-policy TD-learning. At the same time, our analysis allows for nonzero $\tau$, which is important in practice, as smoother regularized or trust-region updates, such as NPG and its variants TRPO and PPO, are often preferred for stabilizing policy updates \cite{kakade2001natural,schulman2015trust,schulman2017ppo}.

To obtain the optimal convergence rate in Theorem~\ref{thm:main}, we use harmonic stepsizes with $\omega>2/(1-\gamma)$. Other stepsize choices offer different trade-offs. Constant stepsizes yield geometric decay of the bias but leave a nonvanishing stochastic error, while polynomial stepsizes guarantee convergence for any $\omega$, making the choice more robust at the cost of a slower rate.

The bounds have the same structure and rates for the IS and ETD updates, but ETD has milder dependence on the problem parameters $m,n,(1-\gamma)^{-1},\pi_{b,\min}$, and $\mu_{b,\min}$ (cf. Eq. \eqref{def:M_constant}). This is expected: by directly computing the conditional expectation, ETD avoids the additional stochasticity introduced by importance sampling, although at a higher computational cost.

The proof of Theorem~\ref{thm:main} is presented in Section \ref{sec:Proof}. As a direct consequence, we obtain the following sample complexity bound.

\begin{corollary}
\label{cor:complexity}
Given $\epsilon>0$, to achieve
$\mathbb{E}\|Q^* - Q^{\pi_T}\|_\infty < \epsilon$, the number of samples required by Algorithm~\ref{alg} is 
\begin{align*}
&\mathcal{O}\left(\frac{m^3 n^3} {(1-\gamma)^{15} \pi_{b,\min}^{7} \mu_{b,\min}^{5} \log(1/\sigma_b)} \frac{ \log(1/\epsilon)}{\epsilon^{2}}\right)\text{ with the IS-based critic,}\\
&\mathcal{O}\left(\frac{m^2 n^2} {(1-\gamma)^{14} \pi_{b,\min}^{4} \mu_{b,\min}^{4} \log(1/\sigma_b)} \frac{ \log(1/\epsilon)}{\epsilon^{2}}\right)\text{ with the ETD-based critic.}
\end{align*}
\end{corollary}

To the best of our knowledge, this is the first last-iterate convergence result for a single-loop actor--critic algorithm achieving $\tilde{\mathcal{O}}(\epsilon^{-2})$ sample complexity, which is minimax optimal with respect to $\epsilon$ up to logarithmic factors \cite{gheshlaghi2013minimax}. Notably, we do not impose strong, algorithm-dependent assumptions on the policies generated along the algorithm trajectory. While the dependence on $\epsilon$ is near optimal, the dependence on problem-specific constants such as $n$, $m$, and $(1-\gamma)^{-1}$ is likely suboptimal. Recall that even for value-based methods such as $Q$-learning \cite{watkins1992q}, attaining minimax sample complexity with respect to all constants requires more refined techniques, such as variance reduction \cite{wainwright2019variance}. Developing analogous algorithmic or technical improvements for single-loop actor--critic methods is an interesting direction for future work.

\subsection{Proof of Theorem \ref{thm:main}}
\label{sec:Proof}

The main challenge in analyzing Algorithm~\ref{alg} is that, due to its single-loop structure, the actor and critic iterates are coupled, and therefore the analysis cannot be decoupled as in nested-loop actor--critic schemes. Moreover, under off-policy learning with importance sampling, the critic iterates are not uniformly bounded. To address these challenges, we adopt a Lyapunov drift analysis, a standard approach in the study of dynamical systems \cite{khalil2002nonlinear,srikant2019finite,borkar2000ode}. Specifically, we construct one Lyapunov function for the actor and another for the critic, establish a drift inequality for each, and then solve the resulting recursions to derive convergence rates. Crucially, because the update equations are coupled, the drift inequalities contain additional error terms that capture this interaction. As a result, establishing an appropriate \textit{cross-domination structure} is essential for obtaining the desired guarantees. We next outline this approach in more detail.

For the actor, we denote the Lyapunov function by $V_t$, which depends on $Q^*-Q^{\pi_t}$ and captures the suboptimality gap of the policy $\pi_t$. For the critic, we denote the Lyapunov function by $W_t$, which depends on $Q_t-Q^{\pi_t}$ and captures the critic estimation error. The explicit forms of these Lyapunov functions will be presented in the next two sections, where we provide the detailed analysis. Our goal is to derive Lyapunov drift inequalities of the form
\begin{align}
  \mathbb{E}V_{t+1} \le (1 - \kappa_1 \omega_t)\mathbb{E}V_t  +D_{1,t} \mathbb{E}W_t + o(\omega_t), \quad 
  \mathbb{E}W_{t+1} \le (1 - \kappa_2 \alpha_t)\mathbb{E}W_t + D_{2,t} \mathbb{E}V_t + o(\alpha_t),\label{eq:desired_drift}
\end{align}
where $\kappa_1,\kappa_2$ are positive constants, and $D_{1,t},D_{2,t}$ depend on the stepsizes. These drift inequalities are coupled: $W_t$ appears as an additive error term in the drift inequality for $V_t$, and vice versa. To solve the recursions, it is essential to ensure, while deriving the separate drift inequalities, that $D_{2,t}<\kappa_1\omega_t$ and $D_{1,t}<\kappa_2\alpha_t$, a condition we call \textit{cross-domination}. Once this condition holds, we can sum the two Lyapunov drift inequalities to obtain a clean recursion for $V_t+W_t$, which, upon iteration, yields the convergence of Algorithm~\ref{alg}. This idea is inspired by the small-gain analysis of continuous-time dynamical systems \cite{jiang1994small}. Related approaches have been used in actor--critic analysis, but only to establish stationary-point convergence \cite{olshevsky2023small}.

We next follow the above plan and present the details of the analysis. Throughout the rest of this subsection, we assume that the assumptions in Theorem~\ref{thm:main} hold.

\subsubsection{Analysis of the Actor}\label{subsec:actor_analysis}

To present our analysis of the actor, we begin by introducing the Bellman operators. Given a policy $\pi$, let $\mathcal{H}_\pi:\mathbb{R}^{mn}\to \mathbb{R}^{mn}$ be the Bellman operator associated with $\pi$, defined as 
\begin{align*}
    [\mathcal{H}_\pi(Q)](s,a)=\mathcal{R}(s,a)+\gamma \sum_{s'}p(s'\mid s,a)\sum_{a'}\pi(a'|s')Q(s',a'),\quad \forall\,(s,a).
\end{align*}
Let $\mathcal{H}:\mathbb{R}^{mn}\to \mathbb{R}^{mn}$ be the Bellman optimality operator defined as 
\begin{align*}
    [\mathcal{H}(Q)](s,a)=\mathcal{R}(s,a)+\gamma \sum_{s'}p(s'\mid s,a)\max_{a'}Q(s',a'),\quad \forall\,(s,a).
\end{align*}
We use $\|Q^*-Q^{\pi_t}\|_\infty$ as our Lyapunov function for the actor.

The following proposition presents its Lyapunov drift inequality. See Appendix~\ref{app:actor} for its proof.

\begin{proposition}
\label{prop:actor}
The following inequality holds for all $t$:
\begin{align}\label{eq:actor_drift}
&\infnorm{Q^* - Q^{\pi_{t+1}}}
\le
\underbrace{\big(1 - \omega_t(1-\gamma)\big) \infnorm{Q^* - Q^{\pi_t}}}_{\text{actor drift}} +
\underbrace{\frac{2\omega_t}{1-\gamma} \infnorm{Q_t - Q^{\pi_t}}}_{\text{critic coupling error}}
+
\underbrace{\frac{\omega_t}{1-\gamma}\chi_t}_{\text{temperature error}},
\end{align}
where $\chi_t := \infnorm{\mathcal{H}Q_t - \mathcal{H}_{\tilde{\pi}_t} Q_t}$. 
\end{proposition}

In Proposition~\ref{prop:actor}, the first term on the right-hand side represents the actor drift and exhibits a contraction structure. To see this, consider the special case where (i) there is no critic evaluation error, i.e., $Q_t = Q^{\pi_t}$, and (ii) we perform a greedy update, i.e., $\chi_t = 0$ and $\omega_t \equiv 1$. In this case, Algorithm~\ref{alg} reduces to policy iteration, and Proposition~\ref{prop:actor} simplifies to $\infnorm{Q^* - Q^{\pi_{t+1}}}
\le
\gamma \infnorm{Q^* - Q^{\pi_t}}$,
which is consistent with the geometric convergence of policy iteration \cite{puterman2014markov}.

The critic evaluation error $\infnorm{Q_t - Q^{\pi_t}}$ will be analyzed in the next subsection. As for the last term $\chi_t$, it captures the deviation from a greedy update and vanishes as $\tau_t \to 0$ for all three update rules presented in Section~\ref{subsec:algo}. For this reason, we refer to the last term as the temperature error.

The complete proof of Proposition~\ref{prop:actor} is provided in Appendix~\ref{app:actor}. Here, we present a proof sketch highlighting the main ideas.

\begin{proof}[Proof Sketch of Proposition~\ref{prop:actor}]
We adopt an approximate policy iteration viewpoint of Algorithm \ref{alg}. Recall that the two key steps in establishing the geometric convergence of policy iteration are: (i) using the monotonicity of the Bellman operators to show monotonic improvement, namely $ Q^{\pi_t}\leq  Q^{\pi_{t+1}} $, where the inequalities are interpreted entry-wise, and (ii) using the contraction property of the Bellman operators to show that
$\infnorm{Q^* - Q^{\pi_{t+1}}} \le \gamma \infnorm{Q^* - Q^{\pi_t}}$ \cite{puterman2014markov}. In our setting, two deviations from exact policy iteration must be accounted for: (i) policy improvement is performed using an estimate $Q_t$ rather than the true $Q^{\pi_t}$, which introduces dependence on the evaluation error $\infnorm{Q_t - Q^{\pi_t}}$, and (ii) the policy update is conservative, and the target $\tilde{\pi}_t$ is only approximately greedy with respect to $Q_t$, which introduces dependence on the temperature error $\chi_t$. In this case, we show that monotonic improvement holds up to additive error terms corresponding to these deviations. In particular, let $\delta_t := \max_{s,a} \bigl(Q^{\pi_t}(s,a) - Q^{\pi_{t+1}}(s,a)\bigr)$. We show that
\begin{align*}
\delta_t
\le
\frac{2\omega_t}{1-\gamma} \infnorm{Q_t - Q^{\pi_t}}
+
\frac{\omega_t }{1-\gamma}\chi_t.
\end{align*}
Notably, when $Q_t = Q^{\pi_t}$ and $\chi_t = 0$, we recover monotonic improvement, i.e., $\delta_t \le 0$.

Using the above bound together with the contraction property of the Bellman operators, we obtain the inequality \eqref{eq:actor_drift}.
\end{proof}

\subsubsection{Analysis of the Critic}
\label{subsec:analysis_critic}

We present the analysis of the critic only for the IS-based critic, as the analysis for the ETD-based critic is similar. The critic update can be viewed as a stochastic approximation scheme that tracks a \emph{moving target} $Q^{\pi_t}$ \cite{robbins1951stochastic}. Stochastic approximation with fixed targets has been extensively studied using Lyapunov methods \cite{srikant2019finite,borkar2009stochastic}. In particular, off-policy TD-learning with a fixed target policy has been analyzed via weighted $\ell_2$-norm Lyapunov functions \cite{chen2021finite}. Our setting is substantially more challenging due to the coupling with the actor and the resulting time-varying policies.

We begin by reformulating the critic update as a stochastic approximation algorithm. For any $t\geq 0$, let $\{Y_t\}$ be a stochastic process defined as $Y_t=(S_t,A_t,S_{t+1},A_{t+1})$. It is clear that $\{Y_t\}$ is a Markov chain with a finite state space, denoted by $\mathcal{Y}$. Moreover, under Assumption~\ref{assum:behavior}, the Markov chain $\{Y_t\}$ admits a unique stationary distribution, denoted by $\mu_Y$, which satisfies $\mu_Y(s_1,a_1,s_2,a_2)=\mu_b(s_1)\pi_b(a_1|s_1)p(s_2|s_1,a_1)\pi_b(a_2|s_2)$. Let $F_\mathrm{IS}:\mathbb{R}^{mn}\times \mathcal{Y}\times {(\Delta^{m})}^n\to \mathbb{R}^{mn}$ be an operator such that given inputs $Q \in \mathbb{R}^{mn}$,
$y = (s_1,a_1,s_2,a_2) \in \mathcal{Y}$ and $\pi\in{(\Delta^{m})}^n$, the $(s,a)$-th component of the output is defined as
\begin{align*} 
[F_\mathrm{IS}(Q,y,\pi)](s,a) =\mathds{1}_{\{ (s,a)= (s_1,a_1)\}}\left(\mathcal{R}(s_1,a_1) + \gamma \frac{\pi(a_2|s_2)}{\pi_b(a_2| s_2)} Q(s_2,a_2)\right)+\mathds{1}_{\{ (s,a)\neq (s_1,a_1)\}}Q(s,a).
 \end{align*}
The critic update in Algorithm \ref{alg}, Line 10, can now be written compactly as
\begin{align*} 
Q_{t+1} = Q_t + \alpha_t \bigl(F_\mathrm{IS}(Q_t, Y_t, \pi_{t+1}) - Q_t\bigr).
 \end{align*}
We further define the operator $\bar{F}: \mathbb{R}^{mn}\times {(\Delta^{m})}^n\to\mathbb{R}^{mn}$ as $\bar{F}(Q,\pi)=\mathbb{E}_{Y\sim \mu_Y}[F_\mathrm{IS}(Q,Y,\pi)]$. It is now clear that the critic update is a Markovian stochastic approximation scheme for tracking the solution to the (time-varying) fixed-point equation $\bar{F}(Q,\pi_t)=Q$. While both $F$ and $\bar{F}$ depend on $\pi_b$, since $\pi_b$ is fixed throughout the learning process, we omit this dependence from our notation.

Next, we present several key properties of these operators that facilitate the convergence analysis. In particular, we show that the equation $\bar{F}(Q,\pi_t)=Q$ admits $Q^{\pi_t}$ as its unique solution, and that $\bar{F}(\cdot,\pi)$ is a contractive operator with respect to a weighted $\ell_2$ norm. For ease of presentation, for any nonnegative integers $i\le j$, we use the shorthand $\alpha_{i,j}:=\sum_{u=i}^j\alpha_u$ and $\omega_{i,j}:=\sum_{u=i}^j\omega_u$. The proof of the following result is provided in Appendix~\ref{app:F_IS}.

\begin{lemma}\label{lem:F}
There exists $\nu \in (\Delta^{m})^n$ with
$\nu_{\min}:=\min_{s,a}\nu(s,a) \ge (1-\gamma)\mu_{b,\min}\pi_{b,\min}/(nm)$
such that the following properties hold.
\begin{enumerate}[(1)]
\item $\bar{F}(\cdot, \pi)$ is a contraction mapping with respect to $\|\cdot\|_\nu$ for all $\pi\in{(\Delta^{m})}^n$, with contraction ratio $\gamma_c := (1-(1-\gamma)\mu_{b,\min}\pi_{b,\min})^{1/2}$. Moreover, the fixed-point equation $\bar{F}(Q,\pi)=Q$ admits a unique solution $Q^{\pi}$.
\item For all $y\in\mathcal{Y}$ and $\pi\in{(\Delta^{m})}^n$, the operator $F_\mathrm{IS}(\cdot, y,\pi)$ satisfies
\begin{align*}
    \|F_\mathrm{IS}(Q_1, y,\pi)-F_\mathrm{IS}(Q_2, y,\pi)\|_\nu\leq\,& \frac{1}{\pi_{b,\min}\sqrt{\nu_{\min}}}\|Q_1-Q_2\|_\nu,\\
    \|F_\mathrm{IS}(Q_1, y,\pi)-F_\mathrm{IS}(Q_2, y,\pi)\|_\infty\leq\,& \frac{1}{\pi_{b,\min}}\|Q_1-Q_2\|_\infty,
\end{align*}
for all $Q_1,Q_2\in\mathbb{R}^{mn}$,
where $\|\cdot\|_\nu$ denotes the weighted $\ell_2$ norm with weights $\{\nu(s,a)\}_{(s,a)\in\mathcal{S}\times \mathcal{A}}$. 
\item For all non-negative integers $t_1<t_2$, $\|\bar{F}(Q,\pi_{t_1}) - \bar{F}(Q,\pi_{t_2})\|_\infty \leq 2\omega_{t_1,t_2-1}\|Q\|_\infty$ for all $Q\in\mathbb{R}^{mn}$.
\item For all non-negative integers $t_1<t_2$, $\|F_\mathrm{IS}(Q,y,\pi_{t_1}) - F_\mathrm{IS}(Q,y,\pi_{t_2})\|_\infty \leq 2\omega_{t_1,t_2-1}\|Q\|_\infty/\pi_{b,\min}$ for all $y\in\mathcal{Y}$ and $Q\in\mathbb{R}^{mn}$.
\end{enumerate}
\end{lemma}
These properties play distinct roles in the analysis: the contraction of $\bar{F}$ ensures a negative drift; the Lipschitz continuity of $F_\mathrm{IS}$ is essential for controlling the stochastic noise; and properties~(3) and~(4) quantify the errors induced by the time-varying evaluation operators.

In light of Lemma~\ref{lem:F}, we define the Lyapunov function for the critic as $W_t := \|Q_t - Q^{\pi_t}\|_\nu^2/2$. The following result establishes a one-step drift inequality for the critic; the proof is provided in Appendix~\ref{app:critic}.

\begin{proposition}
\label{prop:critic}
For sufficiently small, non-increasing stepsizes $\{\alpha_t,\omega_t\}$ satisfying $\omega_t\leq\tilde{C}_r\alpha_t$ (where $\tilde{C}_r$ is introduced in Theorem~\ref{thm:main} and will be explicitly defined in Appendix~\ref{app:critic}), the following holds for all $t\geq K$:
\begin{align*}
\mathbb{E}W_{t+1}
\le\;&
\underbrace{\big(1-\alpha_t(1-\gamma_c)\big)\mathbb{E}W_t}_{\text{critic drift}}
+\underbrace{\frac{108\alpha_t\alpha_{t-z_t,t-1}}
{(1-\gamma)^2\pi_{b,\min}^2\nu_{\min}}}_{\text{stochastic error}}
+\underbrace{
\frac{\omega_t(1-\gamma)}{2}\left(\mathbb{E}\infnorm{Q^*-Q^{\pi_t}}^2
+\mathbb{E}\chi_t^2\right)
}_{\text{actor-associated errors}}.
\end{align*}
\end{proposition}

The first two terms on the right-hand side of the previous inequality are standard in the study of TD-learning for policy evaluation \cite{srikant2019finite,bhandari2018finite}; in particular, one provides a negative drift and the other captures the error due to stochasticity. The third term is unique, as it captures the convergence error of the actor, which arises from the coupled nature of our single-loop algorithm. Importantly, the actor-associated error is dominated by the negative drift in the actor Lyapunov inequality~\eqref{eq:actor_drift}, which is crucial for establishing the cross-domination property in the next step of the proof.

\begin{proof}[Proof Sketch of Proposition \ref{prop:critic}]
Let $\langle\cdot,\cdot\rangle_\nu$ denote the inner product defined as $\langle Q_1,Q_2\rangle_\nu=\sum_{s,a}Q_1(s,a)Q_2(s,a)\nu(s,a)$. Since the weighted $\ell_2$-norm $\|\cdot\|_\nu$ is induced by this inner product, by the binomial formula, we have 
\begin{align*}
W_{t+1} = W_t
&+ \underbrace{\alpha_t\langle  Q_t-Q^{\pi_t},
\bar{F}(Q_t, \pi_t)-Q_t\rangle_\nu}_{T_1:\text{ expected update term}}+\underbrace{\alpha_t\langle Q_t-Q^{\pi_t},
F_\mathrm{IS}(Q_t,Y_t,\pi_{t+1})-\bar{F}(Q_t, \pi_t)\rangle_\nu}_{T_2:\text{ Markovian noise term}} \\
&+ \underbrace{\langle Q_t-Q^{\pi_t},
Q^{\pi_t}-Q^{\pi_{t+1}}\rangle_\nu}_{T_3:\text{ time-varying target term}}+ \underbrace{\frac{1}{2}\|(Q_{t+1}-Q_t)-(Q^{\pi_{t+1}}-Q^{\pi_t})\|_\nu^2}_{T_4:\text{ residuals}}.
\end{align*}
To proceed, we summarize the key steps involved in bounding each term on the right-hand side.

\textbf{Terms $T_1$ and $T_4$.} 
Using the contraction property of $\bar{F}(\cdot, \pi_t)$ established in Lemma~\ref{lem:F}, we show in Appendix \ref{app:critic} that $T_1\leq-2\alpha_t(1-\gamma_c)W_t$, making this expected update term the main contributor to the negative drift inequality. The residual term $T_4$ is controlled using the Lipschitz properties of $F_{IS}$ together with standard stepsize arguments. In particular, we have the following lemma.
\begin{lemma}\label{mainbody:res}
    There exist constants $C_1,C_2,C_3,C_4$ such that for all $t\geq 0$, we have
    \[T_4\leq C_1\alpha_t^2W_t + C_2\omega_t^2\infnorm{Q^*-Q^{\pi_t}}^2 + C_3\omega_t^2\chi_t^2 + C_4\alpha_t^2.\]
\end{lemma}
The proof of Lemma~\ref{mainbody:res}, along with the explicit expressions of the constants, is presented in Appendix~\ref{app:res}.

\textbf{Term $T_2$ (Markovian noise).}\label{mainbody:noise}
This term captures stochastic fluctuations arising from sampling along a
Markovian trajectory. Its control relies on geometric mixing of the state
process $\{S_t\}$, which ensures that $F_\mathrm{IS}(Q, Y_t, \pi)$ concentrates around $\bar{F}(Q, \pi)$ after
$z_t$ steps. Unlike the fixed-policy setting, the arguments of $\bar{F}$ itself evolve over time due to the changing policy, making this term inherently time-dependent. The conservative actor update plays a key role here in controlling the effect of this time-sensitivity. Following this roadmap, we have the following lemma.
\begin{lemma}\label{le:Markovian_Noise}
    There exist constants $C_5,C_6$ such that for all $t\geq K$, we have
    \[\mathbb{E}T_2\leq C_5\alpha_t\alpha_{t-z_t,t-1}\mathbb{E}W_t + C_6\alpha_t\alpha_{t-z_t,t-1}.\]
\end{lemma}
The proof of Lemma~\ref{le:Markovian_Noise}, along with the explicit expressions of the constants, is presented in Appendix~\ref{app:noise}.

\textbf{Term $T_3$ (Time-varying target).}
This term arises because, in the single-loop implementation, the critic tracks $Q^{\pi_t}$ rather than a fixed target. We show in Appendix~\ref{app:terms} that
\begin{align*}
\infnorm{Q^{\pi_t}-Q^{\pi_{t+1}}}
\le\;& \frac{\omega_t}{1-\gamma}
\bigl(2\infnorm{Q_t-Q^{\pi_t}} + \chi_t\bigr) + \frac{\omega_t}{1-\gamma}\infnorm{Q^*-Q^{\pi_t}},
\end{align*}
which states that the impact of the time-varying target, i.e., $\infnorm{Q^{\pi_t}-Q^{\pi_{t+1}}}$, can be controlled by terms related to the actor Lyapunov function, the critic Lyapunov function, and the temperature error. In the end, we have the following lemma.
\begin{lemma}\label{mainbody:target}
    There exist constants $C_7, C_8$ such that for all $t\geq 0$, we have
    \[T_3\leq C_7\omega_tW_t+\frac{\omega_t(1-\gamma)}{4}\infnorm{Q^*-Q^{\pi_t}}^2 + C_8\omega_t\chi_t^2.\]
\end{lemma}
The proof of Lemma~\ref{mainbody:target}, along with the explicit expressions of the constants, is presented in Appendix~\ref{app:target}.
\end{proof}

\subsubsection{Cross Domination}
Propositions~\ref{prop:actor} and~\ref{prop:critic} establish Lyapunov drift inequalities for the actor and critic, respectively, both of the form in~\eqref{eq:desired_drift}. It therefore remains to combine these inequalities and solve the resulting recursion to obtain the convergence rate of Algorithm~\ref{alg}. In particular, combining Propositions~\ref{prop:actor} and~\ref{prop:critic} yields the following lemma, where we denote $V_t=\infnorm{Q^*-Q^{\pi_t}}^2$ for simplicity.

\begin{lemma}\label{lem:coupled_drift}
For sufficiently small, non-increasing stepsizes $\{\alpha_t,\omega_t\}$ satisfying $\omega_t\leq \tilde{C}_r\alpha_t$, we have for all $t\geq K$ that
\begin{align*}
\mathbb{E}\left[V_{t+1}+W_{t+1}\right]
\leq \left(1-\frac{\omega_t(1-\gamma)}{2}\right)\mathbb{E}\left[V_t+W_t\right]
+\frac{6\omega_t}{(1-\gamma)^3}\mathbb{E}\chi_t^2
+\frac{108\alpha_t\alpha_{t-z_t,t-1}}{(1-\gamma)^2\pi_{b,\min}^2\nu_{\min}}.
\end{align*}
\end{lemma}
The proof of Lemma~\ref{lem:coupled_drift} follows by squaring both sides of~\eqref{eq:actor_drift}, taking expectations, and adding the resulting bound to the critic drift inequality in Proposition~\ref{prop:critic}. See Appendix~\ref{app:cross} for more details.

\begin{proof}[Proof of Theorem~\ref{thm:main}]
Applying Lemma~\ref{lem:coupled_drift} recursively, and using the fact that $\chi_t \le \tau/(t+h)^{\eta/2}$ for all $t$ under Condition~\ref{cond:temp} \cite[Lemma 5.1]{chen2025approximate}, yields the desired convergence rates in Theorem~\ref{thm:main} for the IS-based critic after substituting the specified stepsizes. The details are provided in Appendix~\ref{app:main}. The ETD-based critic satisfies an analogous coupled drift inequality with a smaller stochastic error term; the full proof is provided in Appendix~\ref{app:ETD}. The unified statement in Theorem~\ref{thm:main} uses constants that upper bound those in both the IS and ETD cases.
\end{proof}

\section{Conclusion}\label{sec:conclusion}
In this work, we establish the first $\tilde{\mathcal{O}}(\epsilon^{-2})$ sample complexity bounds, measured by the optimality gap of the last iterate, for a single-loop actor--critic method under minimal assumptions, without relying on timescale separation between the actor and critic updates. We conclude by outlining several directions for future work.

While the dependence on problem parameters such as $n$, $m$, and $(1-\gamma)$ is not fully optimized, prior work on value-based methods suggests that these dependencies can often be improved via advanced variance-reduction techniques. It would be interesting to investigate whether similar ideas can be extended to policy-space methods. 

Another natural direction is to extend our analysis to the function approximation setting. However, the interplay of off-policy learning, bootstrapping methods such as TD learning, and function approximation, commonly referred to as the deadly triad, can lead to instability \cite{sutton2018reinforcement}. While this deadly triad challenge can be addressed using gradient TD or TD with gradient correction \cite{sutton2008convergent}, these TD-learning methods are already two-timescale. After combining them with the actor update, one obtains an iterative algorithm with three sets of coupled iterates. Understanding whether the Lyapunov-based approach developed in this work can address this challenge under minimal assumptions remains an interesting question.
\bibliography{references}
\bibliographystyle{apalike}

\appendix

\newpage

\section{Proof of Theorem \ref{thm:main} with the IS-Based Critic}\label{app:IS}
\subsection{Analysis of the Actor}\label{app:actor}

\subsubsection{Proof of Proposition \ref{prop:actor}}

The following lemma is needed, whose proof is presented in Appendix \ref{app:bell}.
\begin{lemma}\label{lem:Bellman_Affine}
For any $\pi_1,\pi_2$ and $\alpha\in\mathbb{R}$, let $\pi=(1-\alpha)\pi_1+\alpha\pi_2$. Then, for any $Q\in\mathbb{R}^{mn}$, we have $\mathcal{H}_\pi Q=(1-\alpha)\mathcal{H}_{\pi_1}Q+\alpha\mathcal{H}_{\pi_2}Q$.
\end{lemma}

As illustrated in Section \ref{subsec:actor_analysis}, our proof of Proposition \ref{prop:actor} consists of two main steps. We first show that the policies generated by Algorithm \ref{alg} are approximately monotonic. Then, we leverage the approximate monotonicity property together with the contraction property of the Bellman operator to establish the one-step recursive inequality of $\|Q^*-Q^{\pi_t}\|_\infty$. We recall the following notation: $\delta_t = \max_{s,a}\big(Q^{\pi_t}(s,a) - Q^{\pi_{t+1}}(s,a)\big)$ and $\chi_t= \infnorm{\mathcal{H}Q_t-\mathcal{H}_{\tilde{\pi}_t}Q_t}$,
and additionally define $\xi_t := \infnorm{Q_t-Q^{\pi_t}}$. Throughout this proof and all upcoming appendices, inequalities between vectors are interpreted entry-wise.

\paragraph{Step 1: Approximate monotonic improvement.}
Observe that by the definition of $\delta_t$, we have $Q^{\pi_{t+1}} \ge Q^{\pi_t} - \delta_t \mathbf{1}$.
Using the Bellman equation, together with the monotonicity and translation invariance of the Bellman operator, we have
\begin{align}
Q^{\pi_{t+1}}
=\,& \mathcal{H}_{\pi_{t+1}} Q^{\pi_{t+1}} \nonumber\\
\ge\,& \mathcal{H}_{\pi_{t+1}}(Q^{\pi_t} - \delta_t\mathbf{1}) \nonumber\\
=\,& \mathcal{H}_{\pi_{t+1}} Q^{\pi_t} - \gamma \delta_t \mathbf{1}\nonumber\\
=\,&(1-\omega_t)\mathcal{H}_{\pi_t}Q^{\pi_t}
+
\omega_t \mathcal{H}_{\tilde{\pi}_t}Q^{\pi_t}- \gamma \delta_t \mathbf{1}\tag{Lemma \ref{lem:Bellman_Affine}}\nonumber\\
=\,&(1-\omega_t)Q^{\pi_t}
+
\omega_t \mathcal{H}_{\tilde{\pi}_t}Q^{\pi_t}- \gamma \delta_t \mathbf{1}.\label{eq:mono}
\end{align}
Rearranging terms, we obtain
\begin{align*}
Q^{\pi_t} - Q^{\pi_{t+1}}
\le
\omega_t\big(Q^{\pi_t}-\mathcal{H}_{\tilde{\pi}_t}Q^{\pi_t}\big)
+
\gamma\delta_t\mathbf{1}
\le \omega_t(\mathcal{H}Q^{\pi_t} - \mathcal{H}_{\tilde{\pi}_t}Q^{\pi_t}) + \gamma\delta_t\mathbf{1},
\end{align*}
where the last inequality follows from $\mathcal{H}Q^{\pi_t}\geq \mathcal{H}_{\pi_t}Q^{\pi_t} = Q^{\pi_t}$.

To proceed, observe that
\begin{align*}
    \mathcal{H}Q^{\pi_t} - \mathcal{H}_{\tilde{\pi}_t}Q^{\pi_t}
    =\,& \mathcal{H}Q^{\pi_t} - \mathcal{H}Q_t + \mathcal{H}Q_t-\mathcal{H}_{\tilde{\pi}_t}Q_t + \mathcal{H}_{\tilde{\pi}_t}Q_t - \mathcal{H}_{\tilde{\pi}_t}Q^{\pi_t}\\
    \leq\,& \infnorm{\mathcal{H}Q^{\pi_t} - \mathcal{H}Q_t}\mathbf{1} + \infnorm{\mathcal{H}Q_t-\mathcal{H}_{\tilde{\pi}_t}Q_t}\mathbf{1} + \infnorm{\mathcal{H}_{\tilde{\pi}_t}Q_t - \mathcal{H}_{\tilde{\pi}_t}Q^{\pi_t}}\mathbf{1}\\
    \leq \,&2\gamma\xi_t\mathbf{1} + \chi_t\mathbf{1},
\end{align*}
where the last inequality follows from the contraction property of the operators $\mathcal{H}$ and $\mathcal{H}_{\tilde{\pi}_t}$ and the definitions of $\xi_t$ and $\chi_t$.

Combining the previous two inequalities, we have
\begin{align*}
    Q^{\pi_t} - Q^{\pi_{t+1}}\leq \omega_t(2\gamma\xi_t\mathbf{1} + \chi_t\mathbf{1})+\gamma\delta_t\mathbf{1}.
\end{align*}
Since the right-hand side is a constant vector, we must have
\begin{align*}
    \delta_t=\max_{s,a}(Q^{\pi_t}(s,a) - Q^{\pi_{t+1}}(s,a))\leq \omega_t(2\gamma\xi_t + \chi_t)+\gamma\delta_t.
\end{align*}
Rearranging terms, we obtain
\begin{align}\label{eq:delta_bound}
    \delta_t\leq \frac{\omega_t(2\gamma\xi_t + \chi_t)}{1-\gamma}.
\end{align}

\paragraph{Step 2: A contractive recursion.}
For any $t\geq 0$, we have by Inequality \eqref{eq:mono} that
\begin{align*}
Q^* - Q^{\pi_{t+1}}
&\le
(1-\omega_t)(Q^* - Q^{\pi_t})
+
\omega_t(Q^* - \mathcal{H}_{\tilde{\pi}_t}Q^{\pi_t})
+
\gamma\delta_t\mathbf{1}.
\end{align*}
To control the term $Q^*-\mathcal{H}_{\tilde{\pi}_t}Q^{\pi_t}$, note that
\begin{align*}
Q^* - \mathcal{H}_{\tilde{\pi}_t}Q^{\pi_t}
&=
\mathcal{H}Q^* - \mathcal{H}Q^{\pi_t}
+
\mathcal{H}Q^{\pi_t}
-
\mathcal{H}_{\tilde{\pi}_t}Q^{\pi_t}\\
&\leq\gamma\infnorm{Q^*-Q^{\pi_t}}\mathbf{1} + (2\gamma\xi_t+\chi_t)\mathbf{1},
\end{align*}
where the last inequality follows from the contraction property of $\mathcal{H}$ and the previous bound on $\mathcal{H}Q^{\pi_t}-\mathcal{H}_{\tilde{\pi}_t}Q^{\pi_t}$. Therefore, we have
\begin{align*}
Q^* - Q^{\pi_{t+1}}
\leq \,&
(1-\omega_t)
(Q^* - Q^{\pi_t}) +
\omega_t\big(\gamma\infnorm{Q^*-Q^{\pi_t}}+
2\gamma\xi_t + \chi_t\big)\mathbf{1}
+
\gamma\delta_t\mathbf{1}\\
\leq \,&(1-\omega_t)
\|Q^* - Q^{\pi_t}\|_\infty\mathbf{1} +
\omega_t\big(\gamma\infnorm{Q^*-Q^{\pi_t}}+
2\gamma\xi_t + \chi_t\big)\mathbf{1}
+
\gamma\delta_t\mathbf{1}\\
=\,&(1-(1-\gamma)\omega_t)
\|Q^* - Q^{\pi_t}\|_\infty\mathbf{1} +
\omega_t\big(
2\gamma\xi_t + \chi_t\big)\mathbf{1}
+
\gamma\delta_t\mathbf{1}.
\end{align*}
Since $Q^* - Q^{\pi_{t+1}}\geq 0$ and the right-hand side of the previous inequality is a constant vector, we have
\begin{align*}
    \infnorm{Q^*-Q^{\pi_{t+1}}} \leq \,&\big(1-\omega_t(1-\gamma)\big)\infnorm{Q^*-Q^{\pi_t}} + \omega_t(2\gamma\xi_t + \chi_t) + \gamma\delta_t\\
    \leq \,&\big(1-\omega_t(1-\gamma)\big)
\|Q^* - Q^{\pi_t}\|_\infty
+
\frac{2\omega_t}{1-\gamma}\xi_t
+
\frac{\omega_t}{1-\gamma}\chi_t,
\end{align*}
where the last inequality follows from the upper bound of $\delta_t$ in Inequality \eqref{eq:delta_bound}.

\subsubsection{Establishing the Drift Inequality for the Squared Norm Error}

The following corollary follows from Proposition \ref{prop:actor}, which is needed to combine the actor drift with the critic drift.

\begin{corollary}\label{cor:actor}
    The following inequality holds for any $t\geq 0$:
    \begin{align*} \infnorm{Q^*-Q^{\pi_{t+1}}}^2 \leq \big(1-\omega_t(1-\gamma)\big)\infnorm{Q^*-Q^{\pi_t}}^2 + \frac{6\omega_t}{(1-\gamma)^3}\infnorm{Q_t-Q^{\pi_t}}^2 + \frac{5\omega_t}{(1-\gamma)^3}\chi^2_t. \end{align*}
\end{corollary}
\begin{proof}[Proof of Corollary \ref{cor:actor}]
Recall the notation $V_t = \infnorm{Q^*-Q^{\pi_t}}^2$. Then, Proposition \ref{prop:actor} implies
\begin{align}
V_{t+1}\leq\,& \left[\big(1-\omega_t(1-\gamma)\big)\infnorm{Q^*-Q^{\pi_t}}
+
\frac{2\omega_t}{1-\gamma}\xi_t
+
\frac{\omega_t}{1-\gamma}\chi_t\right]^2\nonumber\\
=\,& 
\big(1-\omega_t(1-\gamma)\big)^2 V_t
+
\frac{4\omega_t^2}{(1-\gamma)^2}\xi_t^2
+
\frac{\omega_t^2}{(1-\gamma)^2}\chi_t^2 \nonumber\\
&
+
\frac{4\omega_t\big(1-\omega_t(1-\gamma)\big)}{1-\gamma}\infnorm{Q^*-Q^{\pi_t}}\xi_t
+
\frac{2\omega_t\big(1-\omega_t(1-\gamma)\big)}{1-\gamma}\infnorm{Q^*-Q^{\pi_t}}\chi_t
\nonumber\\
&+
\frac{4\omega_t^2}{(1-\gamma)^2}\xi_t\chi_t. \label{eq:square_decomposition}
\end{align}
We now control the cross terms using the AM--GM inequality:
\begin{align*}
\infnorm{Q^*-Q^{\pi_t}}\xi_t
\le\,&
\frac{(1-\gamma)^2}{6}V_t
+
\frac{3}{2(1-\gamma)^2}\xi_t^2,\\
\infnorm{Q^*-Q^{\pi_t}}\chi_t
\le\,&
\frac{(1-\gamma)^2}{6}V_t
+
\frac{3}{2(1-\gamma)^2}\chi_t^2,
\\
\xi_t\chi_t
\le\,&
\frac{1}{4}\xi_t^2
+
\chi_t^2.
\end{align*}
Substituting these bounds into \eqref{eq:square_decomposition} yields
\begin{align*}
V_{t+1}
&\le
\big(1-\omega_t(1-\gamma)\big)^2V_t
+
\frac{4\omega_t^2}{(1-\gamma)^2}\xi_t^2
+
\frac{\omega_t^2}{(1-\gamma)^2}\chi_t^2 \\
&\quad
+
\frac{4\omega_t\big(1-\omega_t(1-\gamma)\big)}{1-\gamma}
\left(
\frac{(1-\gamma)^2}{6}V_t
+
\frac{3}{2(1-\gamma)^2}\xi_t^2
\right) \\
&\quad
+
\frac{2\omega_t\big(1-\omega_t(1-\gamma)\big)}{1-\gamma}
\left(
\frac{(1-\gamma)^2}{6}V_t
+
\frac{3}{2(1-\gamma)^2}\chi_t^2
\right) 
+
\frac{4\omega_t^2}{(1-\gamma)^2}
\left(
\frac{1}{4}\xi_t^2
+
\chi_t^2
\right)\\
=\,&\left[\big(1-\omega_t(1-\gamma)\big)^2
+
\big(1-\omega_t(1-\gamma)\big)\omega_t(1-\gamma)\right]V_t\\
&+\left[\frac{4\omega_t^2}{(1-\gamma)^2}
+
\frac{6\omega_t\big(1-\omega_t(1-\gamma)\big)}{(1-\gamma)^3}
+
\frac{\omega_t^2}{(1-\gamma)^2}\right]\xi_t^2\\
&+\left[\frac{\omega_t^2}{(1-\gamma)^2}
+
\frac{3\omega_t\big(1-\omega_t(1-\gamma)\big)}{(1-\gamma)^3}
+
\frac{4\omega_t^2}{(1-\gamma)^2}\right]\chi_t^2.
\end{align*}
By straightforward algebraic manipulations, the coefficients of $V_t$, $\xi_t^2$, and $\chi_t^2$ can be simplified as
\begin{align*}
    &\big(1-\omega_t(1-\gamma)\big)^2
+
\big(1-\omega_t(1-\gamma)\big)\omega_t(1-\gamma)
=
1-\omega_t(1-\gamma),\\
&\frac{4\omega_t^2}{(1-\gamma)^2}
+
\frac{6\omega_t\big(1-\omega_t(1-\gamma)\big)}{(1-\gamma)^3}
+
\frac{\omega_t^2}{(1-\gamma)^2}
=
\frac{6\omega_t - \omega_t^2(1-\gamma)}{(1-\gamma)^3} \leq \frac{6\omega_t}{(1-\gamma)^3},\\
&\frac{\omega_t^2}{(1-\gamma)^2}
+
\frac{3\omega_t\big(1-\omega_t(1-\gamma)\big)}{(1-\gamma)^3}
+
\frac{4\omega_t^2}{(1-\gamma)^2}
=
\frac{3\omega_t+2\omega_t^2(1-\gamma)}{(1-\gamma)^3}\leq \frac{5\omega_t}{(1-\gamma)^3}.
\end{align*}
Therefore, we have
\begin{align*}
    V_{t+1} \leq\,& \big(1-\omega_t(1-\gamma)\big)V_t + \frac{6\omega_t}{(1-\gamma)^3}\xi_t^2 + \frac{5\omega_t}{(1-\gamma)^3}\chi_t^2.
\end{align*}
\end{proof}

\subsection{Analysis of the Critic}\label{app:critic}

We first state the complete version of Proposition \ref{prop:critic}, with explicit expressions for all constants. Recall that $z_t = \min\{k\geq 1\mid \max_{s\in\mathcal{S}}\mathrm{d}_{\mathrm{TV}}\big(P_{\pi_b}^{k-1}(s,\cdot),\mu_b(\cdot)\big) \leq \omega_t/2\}$ and $K = \min\{t\in\mathbb{N}\mid t\geq z_t\}$.

\begin{proposition}\label{big_critic}
    Under Assumption \ref{assum:behavior}, suppose that the stepsizes are non-increasing and satisfy 
    \begin{align*}\omega_t\leq \frac{(1-\gamma)^3(1-\gamma_c)\nu_{\min}}{20}\alpha_t, \quad\alpha_{t-z_t,t-1}\leq\frac{\pi_{b,\min}^2{\nu_{\min}}(1-\gamma_c)}{200}.
    \end{align*}
    Then, the following inequality holds for any $t\geq K$:
    \begin{align*}\mathbb{E}W_{t+1} \leq& \big(1-\alpha_t(1-\gamma_c)\big)\mathbb{E}W_t+ \frac{1-\gamma}{2}\omega_t\mathbb{E}\infnorm{Q^*-Q^{\pi_t}}^2 +\frac{1-\gamma}{2}\omega_t\mathbb{E}\chi_t^2+ \frac{108\alpha_t\alpha_{t-z_t,t-1}}{(1-\gamma)^2\pi_{b,\min}^2\nu_{\min}}.\end{align*}
\end{proposition}

\begin{proof}[Proof of Proposition \ref{big_critic}]
We have, by the binomial decomposition, that
    \begin{align}
        W_{t+1} 
        =& W_t + \langle Q_t-Q^{\pi_t}, Q_{t+1}-Q_t \rangle_\nu + \langle Q_t-Q^{\pi_t}, Q^{\pi_t}-Q^{\pi_{t+1}} \rangle_\nu \nonumber\\
        &+\frac{1}{2}\|(Q_{t+1}-Q_{t}) + (Q^{\pi_t}-Q^{\pi_{t+1}})\|_\nu^2 \nonumber\\
        =& W_t + \alpha_t\langle Q_t-Q^{\pi_t}, F_\mathrm{IS}(Q_t, Y_t, \pi_{t+1}) - Q_t \rangle_\nu + \langle Q_t-Q^{\pi_t}, Q^{\pi_t}-Q^{\pi_{t+1}} \rangle_\nu \nonumber\\
        &+\frac{1}{2}\|(Q_{t+1}-Q_{t}) + (Q^{\pi_t}-Q^{\pi_{t+1}})\|_\nu^2 \nonumber\\
        =& W_t + \underbrace{\alpha_t\langle Q_t-Q^{\pi_t}, \bar{F}(Q_t, \pi_t) - Q_t \rangle_\nu}_{T_1:\text{ expected update term}} + \underbrace{\alpha_t\langle Q_t-Q^{\pi_t}, F_\mathrm{IS}(Q_t, Y_t, \pi_{t+1}) - \bar{F}(Q_t, \pi_t) \rangle_\nu}_{T_2:\text{ Markovian noise term}} \nonumber\\
        & + \underbrace{\langle Q_t-Q^{\pi_t}, Q^{\pi_t}-Q^{\pi_{t+1}} \rangle_\nu}_{T_3:\text{ time-varying Target Term}}+ \underbrace{\frac{1}{2}\|(Q_{t+1}-Q_t) + (Q^{\pi_t}-Q^{\pi_{t+1}})\|_\nu^2}_{T_4:\text{ residuals}}. \label{critic_decomposition}
    \end{align}
    We now bound the terms $T_1, T_2, T_3$ and $T_4$. 
    
    For the term $T_1$, we have, by the fact that $Q^{\pi_t}$ is the fixed point of $\bar{F}(\cdot, \pi_t)$, that 
\begin{align}
    T_1
    =\,&\alpha_t\langle Q_t-Q^{\pi_t}, \bar{F}(Q_t, \pi_t) - Q_t \rangle_\nu\nonumber\\
    =\,&\alpha_t\langle Q_t - Q^{\pi_t}, \bar{F}(Q_t, \pi_t) - \bar{F}(Q^{\pi_t}, \pi_t) \rangle_\nu - \alpha_t\|Q^{\pi_t} - Q_t\|_\nu^2\nonumber\\
    \leq\,&\alpha_t\|Q_t - Q^{\pi_t}\|_\nu \|\bar{F}(Q_t, \pi_t) - \bar{F}(Q^{\pi_t}, \pi_t)\|_\nu - \alpha_t\|Q^{\pi_t} - Q_t\|_\nu^2 \nonumber\\
    \leq\,&-\alpha_t(1 - \gamma_c)\|Q^{\pi_t} - Q_t\|_\nu^2\nonumber\\
    =\,&-2\alpha_t(1 - \gamma_c)W_t.\label{eq:T1_bound}
\end{align}

    For the terms $T_2$, $T_3$, and $T_4$, they are bounded in the following sequence of lemmas, whose proofs are presented in Appendix \ref{app:noise}, Appendix \ref{app:target}, and Appendix \ref{app:res} respectively.

\begin{lemma}\label{lem:noise}
        The following inequality holds for all $t\geq K$:
        \begin{align*}
    \mathbb{E}T_{2}&\leq \frac{90\alpha_t\alpha_{t-z_t,t-1}}{\pi_{b,\min}^2\nu_{\min}}W_t+
    \frac{100\alpha_t\alpha_{t-z_t,t-1}}{(1-\gamma)^2\pi_{b,\min}^2\nu_{\min}}.
\end{align*}
\end{lemma}

\begin{lemma}\label{lem:target}
     The following inequality holds for all $t\geq 0$:
    \begin{align*}T_3 \leq \frac{8\omega_t}{(1-\gamma)^3\sqrt{\nu_{\min}}}W_t+\frac{1-\gamma}{4}\omega_t\|Q^*-Q^{\pi_t}\|_\infty^2+\frac{1-\gamma}{4}\omega_t\chi_t^2.
\end{align*}
\end{lemma}

\begin{lemma}\label{lem:residual}
        The following inequality holds for all $t\geq 0$:
\begin{align*}
    T_{4} \leq &\left(\frac{9\alpha_t^2}{\pi_{b,\min}^2\nu_{\min}}+\frac{24\omega_t^2}{(1-\gamma)^2\nu_{\min}}\right)W_t + \frac{3\omega_t^2}{(1-\gamma)^2}\infnorm{Q^*-Q^{\pi_t}}^2+\frac{3\omega_t^2\chi_t^2}{(1-\gamma)^2}+\frac{8\alpha_t^2}{(1-\gamma)^2\pi_{b,\min}^2}.
\end{align*}
\end{lemma}
Taking the expectation of both sides in \eqref{critic_decomposition} and substituting the bounds on
$T_1$, $\mathbb{E}T_{2}$, $T_3$, and $T_{4}$, we have 
\begin{align*}
\mathbb{E}W_{t+1} \leq& \mathbb{E}W_t- 2\alpha_t(1-\gamma_c)\mathbb{E}W_t+ \frac{90\alpha_t\alpha_{t-z_t, t-1}}{\pi_{b,\min}^2\nu_{\min}}\mathbb{E}W_t +\frac{100\alpha_t\alpha_{t-z_t, t-1}}{(1-\gamma)^2\pi_{b,\min}^2\nu_{\min}}\\
&+\frac{8\omega_t}{(1-\gamma)^3\sqrt{\nu_{\min}}}\mathbb{E}W_t + \frac{1-\gamma}{4}\omega_t\mathbb{E}\infnorm{Q^*-Q^{\pi_t}}^2 + \frac{1-\gamma}{4}\omega_t\mathbb{E}\chi_t^2\\
&+ \left(\frac{9\alpha_t^2}{\pi_{b,\min}^2\nu_{\min}}+\frac{24\omega_t^2}{(1-\gamma)^2\nu_{\min}}\right)\mathbb{E}W_t + \frac{3\omega_t^2}{(1-\gamma)^2}\mathbb{E}\infnorm{Q^*-Q^{\pi_t}}^2 + \frac{3\omega_t^2}{(1-\gamma)^2}\mathbb{E}\chi_t^2\\ &\quad+ \frac{8\alpha_t^2}{(1-\gamma)^2\pi_{b,\min}^2}\\
=& \left(1-2\alpha_t(1-\gamma_c) + \frac{90\alpha_t\alpha_{t-z_t, t-1}}{\pi_{b,\min}^2\nu_{\min}} + \frac{9\alpha_t^2}{\pi_{b,\min}^2\nu_{\min}}+\frac{24\omega_t^2}{(1-\gamma)^2\nu_{\min}}+\frac{8\omega_t}{(1-\gamma)^3\sqrt{\nu_{\min}}}\right)\mathbb{E}W_t\\&
+ \left(\frac{1-\gamma}{4}\omega_t+\frac{3\omega_t^2}{(1-\gamma)^2}\right)\mathbb{E}\infnorm{Q^*-Q^{\pi_t}}^2 + \left(\frac{1-\gamma}{4}\omega_t+\frac{3\omega_t^2}{(1-\gamma)^2}\right)\mathbb{E}\chi_t^2 \\
&+ \frac{100\alpha_t\alpha_{t-z_t, t-1}}{(1-\gamma)^2\pi_{b,\min}^2\nu_{\min}}+\frac{8\alpha_t^2}{(1-\gamma)^2\pi_{b,\min}^2}.
\end{align*}
The $\alpha_t^2$ terms can be absorbed into the
$\alpha_t\alpha_{t-z_t,t-1}$ contribution since $\alpha_t \leq \alpha_{t-z_t,t-1}$. Similarly the $\omega_t^2$ terms can be absorbed into the $\omega_t$ contribution as $\omega_t\leq (1-\gamma)^3(1-\gamma_c)\nu_{\min}\alpha_t/20 \leq \sqrt{\nu_{\min}}/12$. Thus, we have
\begin{align*}
    \mathbb{E}W_{t+1}
&\le
\Bigg(
1
-2\alpha_t(1-\gamma_c)
+
\frac{100\alpha_t\alpha_{t-z_t,t-1}}
{\pi_{b,\min}^2\nu_{\min}}
+
\frac{10\omega_t}
{(1-\gamma)^3\sqrt{\nu_{\min}}}
\Bigg)
\mathbb{E}W_t
\\
&\quad
+
\frac{\omega_t(1-\gamma)}{2}
\mathbb{E}\|Q^*-Q^{\pi_t}\|_\infty^2
+\frac{\omega_t(1-\gamma)}{2}
\mathbb{E}\chi_t^2
+
\frac{108\alpha_t\alpha_{t-z_t,t-1}}
{(1-\gamma)^2\pi_{b,\min}^2\nu_{\min}}.
\end{align*}
Bounding the perturbation terms in the contraction factor (coefficient of $\mathbb{E}W_t$) using the stepsize conditions $\alpha_{t-z_t, t-1}\leq \pi_{b,\min}^2\nu_{\min}(1-\gamma_c)/200$ and $\omega_t\leq (1-\gamma)^3(1-\gamma_c)\nu_{\min}\alpha_t/20$, we obtain the result.
\end{proof}

\subsection{Combining the Actor and Critic}\label{app:convergence}

\subsubsection{Proof of Lemma \ref{lem:coupled_drift}}\label{app:cross}
We begin by stating the complete result with the exact stepsizes and constants.
\begin{lemma}\label{big:coupled}
    Under the same conditions as Proposition \ref{big_critic}, the following inequality holds for all $t\geq K$:
    \begin{align*}
    \mathbb{E}[\infnorm{Q^*-Q^{\pi_{t+1}}}^2 + W_{t+1}] \leq& \left(1-\frac{\omega_t(1-\gamma)}{2}\right)\mathbb{E}[\infnorm{Q^*-Q^{\pi_{t}}}^2 + W_{t}] + \frac{6\omega_t}{(1-\gamma)^3}\mathbb{E}\chi_t^2\\
    &+\frac{108\alpha_t\alpha_{t-z_t,t-1}}{(1-\gamma)^2\pi_{b,\min}^2\nu_{\min}}.
    \end{align*}
\end{lemma}

\begin{proof}[Proof of Lemma \ref{big:coupled}]
We add the critic drift (Proposition \ref{big_critic}) to the squared actor
drift inequality from Corollary \ref{cor:actor} (after taking expectation) to obtain:
\begin{align*}
&\mathbb{E}[\|Q^*-Q^{\pi_{t+1}}\|_\infty^2 + W_{t+1}]\\
&\le
\big(1-\omega_t(1-\gamma)\big)
\mathbb{E}\|Q^*-Q^{\pi_t}\|_\infty^2
+
\frac{6\omega_t}{(1-\gamma)^3}
\mathbb{E}\|Q_t-Q^{\pi_t}\|_\infty^2
+
\frac{5\omega_t}{(1-\gamma)^3}
\mathbb{E}\chi_t^2
\\
&\quad
+
\big(1-\alpha_t(1-\gamma_c)\big)
\mathbb{E}W_t
+
\frac{\omega_t(1-\gamma)}{2}
\mathbb{E}\|Q^*-Q^{\pi_t}\|_\infty^2
+
\frac{\omega_t(1-\gamma)}{2}
\mathbb{E}\chi_t^2 + \frac{108\alpha_t\alpha_{t-z_t,t-1}}
{(1-\gamma)^2\pi_{b,\min}^2\nu_{\min}}\\
&\leq \left(1-\frac{\omega_t(1-\gamma)}{2}\right)\mathbb{E}\infnorm{Q^*-Q^{\pi_t}}^2 + \frac{6\omega_t}{(1-\gamma)^3}\mathbb{E}\chi_t^2\\
&\quad + \left(1-\alpha_t(1-\gamma_c)+\frac{12\omega_t}{(1-\gamma)^3\nu_{\min}}\right)\mathbb{E}W_t + \frac{108\alpha_t\alpha_{t-z_t,t-1}}
{(1-\gamma)^2\pi_{b,\min}^2\nu_{\min}}\\
&\leq \left(1-\frac{\omega_t(1-\gamma)}{2}\right)\mathbb{E}\infnorm{Q^*-Q^{\pi_t}}^2 + \left(1-\frac{2\alpha_t(1-\gamma_c)}{5}\right)\mathbb{E}W_t\tag{$\omega_t\leq (1-\gamma)^3(1-\gamma_c)\nu_{\min}\alpha_t/20$}\\
&\quad+\frac{6\omega_t}{(1-\gamma)^3}\mathbb{E}\chi_t^2+\frac{108\alpha_t\alpha_{t-z_t,t-1}}
{(1-\gamma)^2\pi_{b,\min}^2\nu_{\min}}.
\end{align*}
Due to the stepsize condition $\omega_t\leq (1-\gamma)^3(1-\gamma_c)\nu_{\min}\alpha_t/20$, the coefficient of $\mathbb{E}\infnorm{Q^*-Q^{\pi_t}}^2$ in the above is larger than that of $\mathbb{E}W_t$. Using the same ($1-\omega_t(1-\gamma)/2$) as the common coefficient for both $\mathbb{E}W_t$ and $\mathbb{E}\infnorm{Q^*-Q^{\pi_t}}^2$ gives the result.
\end{proof}

\subsubsection{Solving the Recursion}\label{app:main}
The following lemma will be useful. Recall that $K = \min\{k\in\mathbb{N}\mid k\geq z_k\}$.
\begin{lemma}\label{lem:W_T_bound}
    For any stepsize sequence satisfying $\alpha_{0,K-1}\leq \pi_{b,\min}\sqrt{\nu_{\min}}/4$, the following inequality holds almost surely:
    \begin{align*}
        \infnorm{Q^*-Q^{\pi_K}}^2+W_K \leq \frac{3}{(1-\gamma)^2}.
    \end{align*}
\end{lemma}
\begin{remark}
    The condition $\alpha_{0,K-1}\leq \pi_{b,\min}\sqrt{\nu_{\min}}/4$ holds when the stepsize condition $\alpha_{t-z_t, t-1}\leq \pi_{b,\min}^2\nu_{\min}(1-\gamma_c)/200$ appearing in Theorem \ref{thm:main} for the IS-based critic, holds; by substituting $t=K$, for large enough $h$.
\end{remark}
\begin{proof}[Proof of Lemma \ref{lem:W_T_bound}]
    Under the condition $\alpha_{0,K-1}\leq \pi_{b,\min}\sqrt{\nu_{\min}}/4$, it is a direct consequence of \cite[Lemma A.2]{chen2024lyapunov}, that the following holds almost surely:
    \begin{align*}
        \infnorm{Q_K-Q_0} \leq \frac{2\alpha_{0,K-1}}{\pi_{b,\min}}\infnorm{Q_0} + 2\alpha_{0,K-1}
        =2\alpha_{0,K-1} \tag{$Q_0=0$}
        \leq 1.
    \end{align*}
It hence follows that
\begin{align*}
    \infnorm{Q^*-Q^{\pi_K}}^2 + W_K &\leq \frac{1}{(1-\gamma)^2} + \frac{1}{2}\|Q_K-Q^{\pi_K}\|_\nu^2\\
    &\leq \frac{1}{(1-\gamma)^2} + \|Q_K\|_\nu^2 + \|Q^{\pi_K}\|_\nu^2\\
    &\leq \frac{2}{(1-\gamma)^2} + \|Q_K-Q_0\|_\nu^2 \tag{$Q_0=0$}\\
    &\leq \frac{3}{(1-\gamma)^2}.
\end{align*}
\end{proof}

To prove Theorem \ref{thm:main} for the IS-based critic, we first verify the conditions of Lemma~\ref{big:coupled}. We begin by noting that, due to geometric mixing, $z_t\leq \log(4/\omega_t)/\log(1/\sigma_b)+1$ is at most logarithmic in $t$ for any of the stepsize sequences considered. It follows that $\alpha_{t-z_t,t-1}\leq \alpha z_t/(t+h-z_t)^\eta$ is uniformly bounded over all $t\in \mathbb{N}$. Therefore, the condition $\alpha_{t-z_t,t-1}\leq \pi_{b,\min}^2\nu_{\min}(1-\gamma_c)/200$ can be ensured by choosing $\alpha$ small enough for constant stepsizes, and by choosing $h$ large enough for diminishing stepsizes with $\eta\in(0,1]$.

We now prove the existence of the constant threshold $\tilde{C}_r$. By choosing $\alpha$ and $\omega$ such that
\begin{align*}
C_r
&\leq \tilde{C}_{r}
:=
\frac{(1-\gamma)^3(1-\gamma_c)\nu_{\min}}{20},
\end{align*}
the stepsize condition $\omega_t \leq (1-\gamma)^3(1-\gamma_c)\nu_{\min}\alpha_t/20$ is satisfied for any of the stepsize sequences considered in Theorem~\ref{thm:main}. Thus, the stepsize conditions of Lemma~\ref{big:coupled}, the coupled drift lemma, hold for all $t$. Moreover, since $\nu_{\min}\geq (1-\gamma)\pi_{b,\min}\mu_{b,\min}/(mn)$ (cf. Lemma~\ref{lem:F}), the threshold $\tilde{C}_r$ is a constant depending only on $(n,m,p,\gamma,\pi_b)$, as claimed in Theorem~\ref{thm:main}.

We now simplify the coupled drift from Lemma \ref{big:coupled_ETD}. It follows directly from \cite[Lemma 5.1]{chen2025approximate} that under Condition \ref{cond:temp} with parameter $\tau\ge 0$, we have that $\chi_t\le \tau/(t+h)^{\eta/2}=\tau\sqrt{\omega_t/\omega}$ for all $t$. Now note that for constant stepsizes, $\alpha_{t-z_t,t-1}=\alpha z_t\leq2\alpha_tz_t$. For diminishing stepsizes, let $h$ be large enough so that $z_t\leq (t+h)/2$ for all $t\geq K$ (possible since $z_t$ is logarithmic in $t$). This implies that
\begin{align*}
    \alpha_{t-z_t,t-1} \leq \frac{\alpha z_t}{(t-z_t+h)^\eta}
    \leq\frac{\alpha z_t}{\big((t+h)/2\big)^\eta}
    \leq 2\alpha_tz_t.
\end{align*}
The coupled drift from Lemma \ref{big:coupled} hence reduces to
\begin{align*}
\mathbb{E}\left[\|Q^*-Q^{\pi_{t+1}}\|_\infty^2 + W_{t+1}\right]
&\leq
\left(1-\frac{\omega_t(1-\gamma)}{2}\right)
\mathbb{E}\left[\|Q^*-Q^{\pi_t}\|_\infty^2 + W_t\right] \\
&\quad+
\frac{6\tau^2\omega_t^2}{(1-\gamma)^3\omega}
+
\frac{216mn\alpha_t^2z_t}
{(1-\gamma)^3\pi_{b,\min}^3\mu_{b,\min}},
\end{align*}
for all $t\geq K$. Applying the above recursively yields the following for all $T\geq K$:
\begin{align*}
\mathbb{E}\left[\infnorm{Q^*-Q^{\pi_T}}^2+W_T\right]
\leq&\,
\mathbb{E}\left[\infnorm{Q^*-Q^{\pi_K}}^2+W_K\right]
\prod_{t=K}^{T-1}\left(1-\frac{\omega_t(1-\gamma)}{2}\right)\\
&+ \frac{6\tau^2}{(1-\gamma)^3\omega}
\sum_{t=K}^{T-1}\omega_t^2
\prod_{u=t+1}^{T-1}\left(1-\frac{\omega_u(1-\gamma)}{2}\right) \\
&+ \frac{216mn}{(1-\gamma)^3\pi_{b,\min}^3\mu_{b,\min}}
\sum_{t=K}^{T-1}\alpha_t^2z_t
\prod_{u=t+1}^{T-1}\left(1-\frac{\omega_u(1-\gamma)}{2}\right).
\end{align*}
Define $\Pi_{i,j}:=\prod_{u=i}^j\big(1-\omega_u(1-\gamma)/2\big)$. From Lemma~\ref{lem:W_T_bound} and the fact that $z_t\leq z_T$ for all $t\leq T$, it follows that
\begin{align*}
\mathbb{E}\left[\|Q^*-Q^{\pi_{T}}\|_\infty^2 + W_{T}\right]
\leq&\,
\frac{3\Pi_{K,T-1}}{(1-\gamma)^2}
+
\frac{216mnz_{T}}{(1-\gamma)^3\pi_{b,\min}^3\mu_{b,\min}}
\sum_{t=K}^{T-1}\alpha_t^2\Pi_{t+1,T-1}\\
&+
\frac{6\tau^2}{(1-\gamma)^3\omega}
\sum_{t=K}^{T-1}\omega_t^2\Pi_{t+1,T-1}.
\end{align*}
Note that $\omega_t=C_r\alpha_t$ and recall the notation $M_\mathrm{IS} = mn(1-\gamma)^{-3}\pi_{b,\min}^{-3}\mu_{b,\min}^{-1}$. Substituting this into the above gives
\begin{align}
\mathbb{E}\left[\infnorm{Q^*-Q^{\pi_T}}^2+W_T\right]
\leq E_T
:=
\frac{3\Pi_{K,T-1}}{(1-\gamma)^2}
+
\left(\frac{6\tau^2}{(1-\gamma)^3\omega}
+
\frac{216M_\mathrm{IS}z_T}{C_r^2}\right)
\sum_{t=K}^{T-1}\omega_t^2\Pi_{t+1,T-1}.
\label{eq:master}
\end{align}

The remainder of the proof follows by bounding the term $E_T$ for the specific
stepsize sequences.

\paragraph{(1) Constant stepsizes.}
When $\alpha_t = \alpha$ and $\omega_t = \omega$ for all $t$, we have
\begin{align*}
E_T
&\leq
\frac{3}{(1-\gamma)^2}
\left(1-\frac{\omega(1-\gamma)}{2}\right)^{T-K}
+
\left(\frac{6\tau^2}{(1-\gamma)^3\omega}
+
\frac{216M_\mathrm{IS}z_\omega}{C_r^2}\right)
\sum_{t=K}^{T-1}\omega^2
\left(1-\frac{\omega(1-\gamma)}{2}\right)^{T-1-t}\\
&\leq
\frac{3}{(1-\gamma)^2}
\left(1-\frac{\omega(1-\gamma)}{2}\right)^{T-K}
+
\frac{12\tau^2}{(1-\gamma)^4}
+
\frac{432M_\mathrm{IS}\omega z_\omega}{(1-\gamma)C_r^2},
\end{align*}
where the last step follows from the geometric sum
$\sum_{t=0}^\infty\big(1-\omega(1-\gamma)/2\big)^t
= {2}/{\big(\omega(1-\gamma)\big)}$. Substituting the above into
\eqref{eq:master} proves Theorem~\ref{thm:main} (1).

\paragraph{(2) Harmonic stepsizes.}
When $\alpha_t = \alpha/(t+h)$ and $\omega_t=\omega/(t+h)$, by
\cite[Lemma A.7]{chen2024lyapunov}, we have
\begin{align*}
E_T
\leq&\,
\frac{3}{(1-\gamma)^2}
\left(\frac{K+h}{T+h}\right)^{\omega(1-\gamma)/2}\\
&+
\left(\frac{6\tau^2}{(1-\gamma)^3\omega}
+
\frac{216M_\mathrm{IS}z_T}{C_r^2}\right)
\times
\begin{dcases}
\frac{8\omega^2}{\big(2-\omega(1-\gamma)\big)(T+h)^{\omega(1-\gamma)/2}},
& \omega < \frac{2}{1-\gamma},\\
\frac{\omega^2\log(T+h)}{T+h},
& \omega = \frac{2}{1-\gamma},\\
\frac{8e\omega^2}{\big(\omega(1-\gamma)-2\big)(T+h)},
& \omega > \frac{2}{1-\gamma}.
\end{dcases}
\end{align*}
Substituting the above into \eqref{eq:master} proves
Theorem~\ref{thm:main} (2).

\paragraph{(3) Polynomial stepsizes.}
When $\alpha_t = \alpha/(t+h)^\eta$ and
$\omega_t = \omega/(t+h)^\eta$ for $\eta\in (0,1)$, by
\cite[Lemma A.8]{chen2024lyapunov}, we have
\begin{align*}
E_T
\leq&\,
\frac{3}{(1-\gamma)^2}
\exp\left[
-\frac{\omega(1-\gamma)}{2(1-\eta)}
\Big((T+h)^{1-\eta}-(K+h)^{1-\eta}\Big)
\right]\\&+
\left(\frac{6\tau^2}{(1-\gamma)^3\omega}
+
\frac{216M_\mathrm{IS}z_T}{C_r^2}\right)
\frac{8\omega}{(T+h)^\eta}.
\end{align*}
Substituting the above into \eqref{eq:master} proves
Theorem~\ref{thm:main} (3).

\subsubsection{Proof of Corollary \ref{cor:complexity}}\label{app:complexity}
Let $\epsilon>0$ be sufficiently small. In order to ensure $\mathbb{E}\infnorm{Q^*-Q^{\pi_t}} < \epsilon$, it suffices to ensure $\mathbb{E}\infnorm{Q^*-Q^{\pi_t}}^2< \epsilon^2$ due to Jensen's inequality. To that end, we consider harmonic stepsizes with
$\omega = 3/(1-\gamma) > 2/(1-\gamma)$ and
\begin{align*}
\alpha
=
\frac{20\omega}{(1-\gamma)^3(1-\gamma_c)\nu_{\min}}
=
\frac{60}{(1-\gamma)^4(1-\gamma_c)\nu_{\min}},
\end{align*}
which ensures $C_r\leq\tilde{C}_r$. Further, for sufficiently large $h$, the
condition $\alpha_{t-z_t,t-1}\leq(1-\gamma_c)\pi_{b,\min}^2\nu_{\min}/200$ is also
satisfied. Therefore, by Theorem~\ref{thm:main}, for all $T \geq K$,
\begin{align*}
\mathbb{E}\infnorm{Q^*-Q^{\pi_T}}^2
\leq&\,
\frac{3}{(1-\gamma)^2}\left(\frac{K+h}{T+h}\right)^{3/2}
+
\frac{48e\tau^2\omega}{(1-\gamma)^3(T+h)} \\
&+
\underbrace{
\frac{9\times 60^2\times 216\times 8e \times mn}
{(1-\gamma)^{11}(1-\gamma_c)^2\pi_{b,\min}^3\mu_{b,\min}\nu_{\min}^2}
\frac{z_T}{T+h}
}_{D_T}.
\end{align*}
Observe that $D_T$, the third term above, is asymptotically larger than the
other two terms since $z_T$ is logarithmic in $T$. To ensure
$\mathbb{E}\infnorm{Q^*-Q^{\pi_T}}^2=\mathcal{O}(\epsilon^{2})$,
it therefore suffices to have $D_T = \mathcal{O}(\epsilon^{2})$.

We now simplify $D_T$. From the definition of $\gamma_c$ in Lemma~\ref{lem:F},
it follows that
$1-\gamma_c\geq {(1-\gamma)\pi_{b,\min}\mu_{b,\min}}/{2}$. Lemma~\ref{lem:F}
also gives $\nu_{\min}\geq (1-\gamma)\pi_{b,\min}\mu_{b,\min}/(mn)$. Further,
by geometric mixing, we have that
\begin{align*}
z_T
&\leq
\frac{\log(4/\omega_T)}{\log(1/\sigma_b)}+1  \\
&=
\frac{\log\big(4(1-\gamma)(T+h)/3\big)}{\log(1/\sigma_b)}+1  \\
&=
\mathcal{O}\left(\frac{\log T}{\log(1/\sigma_b)}\right).
\end{align*}
Substituting these bounds into the expression for $D_T$, we obtain
\begin{align*}
D_T
=
\mathcal{O}\left(
\frac{m^3n^3}
{(1-\gamma)^{15}\pi_{b,\min}^7\mu_{b,\min}^5\log(1/\sigma_b)}
\frac{\log T}{T}
\right).
\end{align*}
Denoting the coefficient of $(\log T)/T$ in the above by $\Lambda$, it suffices to
have $(\log T)/T = \mathcal{O}\left(\epsilon^2/\Lambda\right)$. This is ensured by
\begin{align*}
T
&=
\mathcal{O}\left(\frac{\Lambda}{\epsilon^2}\log\left(\frac{\Lambda}{\epsilon^2}\right)\right)  \\
&=
\mathcal{O}\left(\Lambda\frac{\log(1/\epsilon)}{\epsilon^2}\right)  \\
&=
\mathcal{O}\left(
\frac{m^3n^3}
{(1-\gamma)^{15}\pi_{b,\min}^7\mu_{b,\min}^5\log(1/\sigma_b)}
\frac{\log(1/\epsilon)}{\epsilon^2}
\right).
\end{align*}

\subsection{Proofs of All Technical Lemmas}\label{app:terms}

\subsubsection{Proof of Lemma \ref{lem:Bellman_Affine}}\label{app:bell}
    By definition of the Bellman operator, for any $(s,a)$, we have
\begin{align*}
    [\mathcal{H}_\pi Q](s,a)
    =\,& \mathcal{R}(s,a)
    +\gamma \sum_{s'\in\mathcal{S}} p(s'\mid s,a)
    \sum_{a'\in\mathcal{A}}
    \left((1-\alpha)\pi_1(a'\mid s')+\alpha\pi_2(a'\mid s')\right)Q(s',a')\\
    =\,& (1-\alpha)\left[\mathcal{R}(s,a)
    +\gamma \sum_{s'\in\mathcal{S}} p(s'\mid s,a)
    \sum_{a'\in\mathcal{A}}\pi_1(a'\mid s')Q(s',a')\right]\\
    &+\alpha\left[\mathcal{R}(s,a)
    +\gamma \sum_{s'\in\mathcal{S}} p(s'\mid s,a)
    \sum_{a'\in\mathcal{A}}\pi_2(a'\mid s')Q(s',a')\right]\\
    =\,& (1-\alpha)[\mathcal{H}_{\pi_1}Q](s,a)
    +\alpha[\mathcal{H}_{\pi_2}Q](s,a).
\end{align*}

\subsubsection{Proof of Lemma \ref{lem:F}}\label{app:F_IS}
The following lemma will be useful in proving Lemma~\ref{lem:F}; its proof is presented in Appendix~\ref{ap:pf:bounds}.
\begin{lemma}\label{bounds}
The following inequalities hold.
\begin{enumerate}[(1)]\label{lem:running_bounds}
    \item For any $Q\in\mathbb{R}^{mn}$, $y\in\mathcal{Y}, u\in\mathcal{U}$, and $\pi \in {(\Delta^{m})}^n$, we have
    \begin{align*}
        \infnorm{F_{\mathrm{IS}}(Q, y, \pi)} \leq 1 + \frac{\|Q\|_\infty}{\pi_{b,\min}},\quad
        \|F_{\mathrm{IS}}(Q, y, \pi)\|_\nu \leq 1 + \frac{\|Q\|_\nu}{\pi_{b,\min}\sqrt{\nu_{\min}}}.
    \end{align*}
    \item For any $Q\in\mathbb{R}^{mn}$ and $\pi \in {(\Delta^{m})}^n$, we have
    \begin{align*}
        \|\bar{F}(Q, \pi)\|_\nu \leq 1+\|Q\|_\nu.
    \end{align*}
    \item For any $t_1,t_2>0$ (assuming $t_1<t_2$) satisfying $\alpha_{t_1,t_2-1} \leq \pi_{b,\min}\sqrt{\nu_{\min}}/4$, we have
\begin{align*}   
    \|Q_{t_1}-Q_{t_2}\|_\nu \leq \,&4\alpha_{t_1,t_2-1} + \frac{4\alpha_{t_1,t_2-1}\|Q_{t_2}\|_\nu}{\pi_{b,\min}\sqrt{\nu_{\min}}},\\
    \|Q_{t_1}-Q_{t_2}\|_\infty \leq \,&4\alpha_{t_1,t_2-1} + \frac{4\alpha_{t_1,t_2-1}\|Q_{t_2}\|_\infty}{\pi_{b,\min}},\\
    \|Q_{t_1}-Q_{t_2}\|_\nu\leq \,&\frac{4\alpha_{t_1,t_2-1}}{\pi_{b,\min}\sqrt{\nu_{\min}}}\|Q_{t_2}-Q^{\pi_{t_2}}\|_\nu + \frac{6\alpha_{t_1,t_2-1}}{(1-\gamma)\pi_{b,\min}\sqrt{\nu_{\min}}},\\
    \|Q_{t_1}-Q_{t_2}\|_\infty\leq \,&\frac{4\alpha_{t_1,t_2-1}}{\pi_{b,\min}}\|Q_{t_2}-Q^{\pi_{t_2}}\|_\infty + \frac{6\alpha_{t_1,t_2-1}}{(1-\gamma)\pi_{b,\min}},\\
    \|Q_{t_1}\|_\nu \leq \,&\frac{3}{1-\gamma} + 2\|Q_{t_2}-Q^{\pi_{t_2}}\|_\nu,\\
    \|Q_{t_1}\|_\infty \leq \,&\frac{3}{1-\gamma} + 2\|Q_{t_2}-Q^{\pi_{t_2}}\|_\infty,
\end{align*}
    \item $\max_{s\in\mathcal{S}}\mathrm{d_{TV}}\big(\pi_{t_1}(s),\pi_{t_2}(s)\big) \leq \omega_{t_1,t_2-1}$.
    \item For any $t_1,t_2>0$ (assuming $t_1<t_2$), we have $\infnorm{Q^{\pi_{t_1}}-Q^{\pi_{t_2}}} \leq \frac{2\omega_{t_1,t_2-1}}{(1-\gamma)^2}$.
\end{enumerate}
\end{lemma}

We now return to the proof of Lemma~\ref{lem:F}.
\begin{enumerate}[(1)]
    \item By the tower property of conditional expectations and the Markov property, we have
    \begin{align}\label{eq:expression:barF}
        \bar{F}(Q,\pi)=(I-D)Q+D\mathcal{H}_\pi(Q),
    \end{align}
    where $D\in\mathbb{R}^{mn\times mn}$ is the diagonal matrix with diagonal entries $\{\mu_b(s)\pi_b(a\mid s)\}_{(s,a)\in\mathcal{S}\times \mathcal{A}}$. The existence of the weighted $\ell_2$ norm $\|\cdot\|_\nu$ and the contraction property then follow from \cite[Theorem 2.1]{chen2021finite}. In view of \eqref{eq:expression:barF}, $Q^\pi$ is clearly a fixed point of $\bar{F}(\cdot,\pi)$. The fact that it is the unique fixed point follows from $\bar{F}(\cdot,\pi)$ being a contraction mapping and the Banach fixed-point theorem \cite{banach1922operations}.

    \item For any $Q,Q'\in\mathbb{R}^{mn}$, $y=(s_1,a_1,s_2,a_2)\in\mathcal{Y}$, and $\pi$, we have for any $(s,a)$ that
    \begin{align*}
        |[F_{\mathrm{IS}}(Q,y,\pi)](s,a)-[F_{\mathrm{IS}}(Q',y,\pi)](s,a)|
        \leq \,&(1-\mathds{1}_{\{(s,a)=(s_1,a_1)\}})|Q(s,a)-Q'(s,a)|\\
        &+\frac{\mathds{1}_{\{(s,a)=(s_1,a_1)\}}}{\pi_{b,\min}}|Q(s_2,a_2)-Q'(s_2,a_2)|\\
        \leq \,&\frac{1}{\pi_{b,\min}}\|Q-Q'\|_\infty.
    \end{align*}
    Therefore, we have
    \begin{align*}
        \|F_{\mathrm{IS}}(Q,y,\pi)-F_{\mathrm{IS}}(Q',y,\pi)\|_\nu
        \leq\,& \|F_{\mathrm{IS}}(Q,y,\pi)-F_{\mathrm{IS}}(Q',y,\pi)\|_\infty\\
        \leq \,&\frac{\|Q-Q'\|_\infty}{\pi_{b,\min}}\\
        \leq\,& \frac{\|Q-Q'\|_\nu}{\pi_{b,\min}\sqrt{\nu_{\min}}}.
    \end{align*}

    \item For any two policies $\pi$ and $\pi'$, we have the following where $\widehat{P}_\pi$ and $\widehat{P}_{\pi'}$ denote their state-action transition probability matrices.
    \begin{align*}
        \|\bar{F}(Q,\pi)-\bar{F}(Q,\pi')\|_\infty
        =\,&\|D(\mathcal{H}_\pi(Q)-\mathcal{H}_{\pi'}(Q))\|_\infty\\
        =\,&\gamma \|D(\widehat{P}_\pi-\widehat{P}_{\pi'})Q\|_\infty\\
        \leq \,&\gamma \|D\|_\infty\|\widehat{P}_\pi-\widehat{P}_{\pi'}\|_\infty\|Q\|_\infty\\
        \leq \,&\|\widehat{P}_\pi-\widehat{P}_{\pi'}\|_\infty\|Q\|_\infty.
    \end{align*}
    Since
    \begin{align*}
        \|\widehat{P}_\pi-\widehat{P}_{\pi'}\|_\infty
        \leq \max_{s,a}\sum_{s'} p(s'\mid s,a)\|\pi(s')-\pi'(s')\|_1
        \leq 2\max_{s'}\mathrm{d_{TV}}(\pi(s'),\pi'(s')),
    \end{align*}
    we further obtain
    \begin{align*}
        \|\bar{F}(Q,\pi)-\bar{F}(Q,\pi')\|_\infty
        \leq 2\|Q\|_\infty\max_{s'}\mathrm{d_{TV}}(\pi(s'),\pi'(s')).
    \end{align*}
    The result then follows by combining the above inequality with Lemma \ref{bounds} (4). 

    \item For any $Q\in\mathbb{R}^{mn}$ and $y=(s_1, a_1, s_2, a_2)\in\mathcal{Y}$, we have 
    \begin{align*}
    &[F_{\mathrm{IS}}(Q, y, \pi_{t_1})-F_{\mathrm{IS}}(Q, y, \pi_{t_2})](s,a) \\&= \frac{\mathds{1}_{\{(s,a)=(s_1, a_1)\}}\gamma}{\pi_{b}(a_2\mid s_2)}\big(\pi_{t_1}(a_2\mid s_2)-\pi_{t_2}(a_2\mid s_2)\big)Q(s_2, a_2).
    \end{align*}
    Therefore, we have
    \begin{align*}
        \|F_{\mathrm{IS}}(Q, y, \pi_{t_1})-F_{\mathrm{IS}}(Q, y, \pi_{t_2})\|_\nu &\leq \infnorm{F_{\mathrm{IS}}(Q, y, \pi_{t_1})-F_{\mathrm{IS}}(Q, y, \pi_{t_2})} \\
        &\leq \frac{\gamma}{\pi_{b,\min}}\infnorm{\pi_{t_1}-\pi_{t_2}}\infnorm{Q}\\
        &\leq \frac{2\max_{s}\mathrm{d_{TV}}(\pi_{t_1}(s), \pi_{t_2}(s))}{\pi_{b,\min}}\|Q\|_\infty\\
        &\leq\frac{2\max_{s}\mathrm{d_{TV}}(\pi_{t_1}(s), \pi_{t_2}(s))}{\pi_{b,\min}\sqrt{\nu_{\min}}}\|Q\|_\nu.
    \end{align*}
    The result then follows by combining the above with Lemma \ref{bounds} (4).
\end{enumerate}

\subsubsection{Proof of Lemma \ref{lem:noise}}\label{app:noise}
For any $t\geq K$, we have
\begin{align}
    &\left\langle Q_t - Q^{\pi_t}, F_\mathrm{IS}(Q_t, Y_t, \pi_{t+1}) - \bar{F}(Q_t, \pi_t) \right\rangle_\nu \nonumber\\
    =& \underbrace{\left\langle Q_{t-z_t} - Q^{\pi_{t-z_t}}, F_\mathrm{IS}(Q_{t-z_t}, Y_t, \pi_{t-z_t}) - \bar{F}(Q_{t-z_t}, \pi_{t-z_t}) \right\rangle_\nu}_{T_{21}}\nonumber\\
    &+ \underbrace{\left\langle Q_{t} - Q^{\pi_{t}}, F_\mathrm{IS}(Q_t, Y_t, \pi_{t+1}) - F_\mathrm{IS}(Q_{t-z_t}, Y_t, \pi_{t+1}) \right\rangle_\nu}_{T_{22}}\nonumber\\
    &+\underbrace{\left\langle Q_{t} - Q^{\pi_{t}}, \bar{F}(Q_{t-z_t}, \pi_{t-z_t}) - \bar{F}(Q_{t}, \pi_{t-z_t}) \right\rangle_\nu}_{T_{23}}\nonumber\\
    &+ \underbrace{\left\langle Q_{t} - Q^{\pi_{t}}, F_\mathrm{IS}(Q_{t-z_t}, Y_t ,\pi_{t+1}) - F_\mathrm{IS}(Q_{t-z_t}, Y_t, \pi_{t-z_t}) \right\rangle_\nu}_{T_{24}}\nonumber\\
    &+\underbrace{\left\langle Q_{t} - Q^{\pi_{t}}, \bar{F}(Q_{t}, \pi_{t-z_t}) - \bar{F}(Q_{t}, \pi_t) \right\rangle_\nu}_{T_{25}}\nonumber\\
    &+\underbrace{\left\langle (Q_{t} - Q^\pi_{t})-(Q_{t-z_t} - Q^{\pi_{t-z_t}}), F_\mathrm{IS}(Q_{t-z_t}, Y_t, \pi_{t+1}) - \bar{F}(Q_{t-z_t}, \pi_{t-z_t})
    \right\rangle_\nu}_{T_{26}}.\label{eq:T2_decomposition}
\end{align}

We next bound each term on the right-hand side of the previous inequality.

\paragraph{The term $T_{21}$.}
    Let $\mathcal{F}_t$ be the sigma-algebra generated by $\{(S_i,A_i)\}_{0\leq i\leq t}$. 
Since $Q_{t-z_t}$ and $\pi_{t-z_t}$ are $\mathcal{F}_{t-z_t}$-measurable, we apply the tower property to obtain
\begin{align}
\mathbb{E}T_{21}
=\,&
\mathbb{E}\left[\langle
Q_{t-z_t}-Q^{\pi_{t-z_t}},
F_{\mathrm{IS}}(Q_{t-z_t}, Y_t, \pi_{t-z_t})-\bar F_{\mathrm{IS}}(Q_{t-z_t}, \pi_{t-z_t})
\rangle_\nu\right] \nonumber\\
=\,&
\mathbb{E}\left[\langle
Q_{t-z_t}-Q^{\pi_{t-z_t}},
\mathbb{E}[F_{\mathrm{IS}}(Q_{t-z_t}, Y_t, \pi_{t-z_t})|\mathcal{F}_{t-z_t}]
-\bar F_{\mathrm{IS}}(Q_{t-z_t}, \pi_{t-z_t})
\rangle_\nu\right]\nonumber\\
\leq \,&\mathbb{E}\left[\|Q_{t-z_t}-Q^{\pi_{t-z_t}}\|_\nu
\|\mathbb{E}[F_{\mathrm{IS}}(Q_{t-z_t}, Y_t, \pi_{t-z_t})|\mathcal{F}_{t-z_t}]
-\bar F_{\mathrm{IS}}(Q_{t-z_t}, \pi_{t-z_t})\|_\nu\right]\label{eq:T21_decomposition}
\end{align}
To bound $\|Q_{t-z_t}-Q^{\pi_{t-z_t}}\|_\nu$, using Lemma \ref{bounds} Items (3) and (5), we have
\begin{align} 
\|Q_{t-z_t}-Q^{\pi_{t-z_t}}\|_\nu
\le\,&
\|Q_t-Q^{\pi_t}\|_\nu
+
\|Q_t-Q_{t-z_t}\|_\nu
+
\|Q^{\pi_t}-Q^{\pi_{t-z_t}}\|_\nu\nonumber\\
\leq \,&\|Q_t-Q^{\pi_t}\|_\nu+\frac{4\alpha_{t-z_t,t-1}}{\pi_{b,\min}\sqrt{\nu_{\min}}}\|Q_{t}-Q^{\pi_{t}}\|_\nu + \frac{6\alpha_{t-z_t,t-1}}{(1-\gamma)\pi_{b,\min}\sqrt{\nu_{\min}}}+\frac{2\omega_{t-z_t,t-1}}{(1-\gamma)^2}\nonumber\\
\leq \,&2\|Q_t-Q^{\pi_t}\|_\nu
+
\frac{2}{1-\gamma},\label{eq:T21:term1}
\end{align}
where the last inequality follows from $\alpha_{t-z_t,t-1}\leq \pi_{b,\min}^2\nu_{\min}(1-\gamma_c)/200\leq \pi_{b,\min}\sqrt{\nu_{\min}}/4$, and $\omega_{t-z_t,t-1}\leq\alpha_{t-z_t,t-1}\leq(1-\gamma)/2$. 

As for the term $\|\mathbb{E}[F_{\mathrm{IS}}(Q_{t-z_t}, Y_t, \pi_{t-z_t})|\mathcal{F}_{t-z_t}]
-\bar F_{\mathrm{IS}}(Q_{t-z_t}, \pi_{t-z_t})\|_\nu$ on the right-hand side of \eqref{eq:T21_decomposition}, we have
\begin{align*}
    &\|\mathbb{E}[F_{\mathrm{IS}}(Q_{t-z_t}, Y_t, \pi_{t-z_t})|\mathcal{F}_{t-z_t}]
-\bar F_{\mathrm{IS}}(Q_{t-z_t}, \pi_{t-z_t})\|_\nu\\
\leq \,&\max_{s\in\mathcal{S}}\sum_{y=(s_1,a_1,s_2,a_2)\in\mathcal{Y}}\left|P^{z_t-1}(s,s_1)-\mu_b(s_1)\right|\pi_b(a_1|s_1)p(s_2\mid s_1,a_1)\pi_b(a_2|s_2)\|F_{\mathrm{IS}}(Q_{t-z_t},y,\pi_{t-z_t})\|_\nu\\
\leq \,&\left(1+\frac{\|Q_{t-z_t}\|_\nu}{\pi_{b,\min}\sqrt{\nu_{\min}}}\right)\max_{s\in\mathcal{S}}\sum_{s_1\in\mathcal{S}}\left|P^{z_t-1}(s,s_1)-\mu_b(s_1)\right|\tag{Lemma \ref{bounds} (1)}\\
\leq \,&\left(1+\frac{\|Q_{t-z_t}\|_\nu}{\pi_{b,\min}\sqrt{\nu_{\min}}}\right)\omega_t.
\end{align*}
Here, the last inequality follows from the definition of $z_t$, where we recall that $z_t = \min\{k\geq 1\mid \max_{s\in\mathcal{S}}\mathrm{d_{TV}}\big(P_{\pi_b}^{k-1}(s,\cdot), \mu_b(\cdot)\big)\leq\omega_t/2\}$, which exists and is finite due to geometric mixing.

To proceed, observe that
\begin{align*}
    \|Q_{t-z_t}\|_\nu\leq \,&\|Q_{t-z_t}-Q_t\|_\nu+\|Q_t-Q^{\pi_t}\|_\nu+\|Q^{\pi_t}\|_\nu\\
    \leq \,&\frac{4\alpha_{t-z_t,t-1}}{\pi_{b,\min}\sqrt{\nu_{\min}}}\|Q_{t}-Q^{\pi_{t}}\|_\nu + \frac{6\alpha_{t-z_t,t-1}}{(1-\gamma)\pi_{b,\min}\sqrt{\nu_{\min}}}+\|Q_t-Q^{\pi_t}\|_\nu+\frac{1}{1-\gamma}\tag{Lemma \ref{bounds} (3)}\\
    \leq \,&2\|Q_t-Q^{\pi_t}\|_\nu+\frac{5}{2(1-\gamma)},
\end{align*}
where the last inequality follows from $\alpha_{t-z_t,t-1}\leq \pi_{b,\min}\sqrt{\nu_{\min}}/4$. Therefore, we have
\begin{align}
    \|\mathbb{E}[F_{\mathrm{IS}}(Q_{t-z_t}, Y_t, \pi_{t-z_t})|\mathcal{F}_{t-z_t}]
-\bar F_{\mathrm{IS}}(Q_{t-z_t}, \pi_{t-z_t})\|_\nu\leq \frac{\omega_t}{\pi_{b,\min}\sqrt{\nu_{\min}}}\left(2\|Q_t-Q^{\pi_t}\|_\nu+\frac{3}{1-\gamma}\right).\label{eq:T21:term2}
\end{align}
Finally, substituting \eqref{eq:T21:term1} and \eqref{eq:T21:term2} in \eqref{eq:T21_decomposition}, we obtain
\begin{align}
    \mathbb{E}T_{21}\leq \,&\frac{\omega_t}{\pi_{b,\min}\sqrt{\nu_{\min}}}\mathbb{E}\left[\left(2\|Q_t-Q^{\pi_t}\|_\nu
+
\frac{3}{1-\gamma}\right)^2\right]\nonumber\\
\leq \,&\frac{2\omega_t}{\pi_{b,\min}\sqrt{\nu_{\min}}}\mathbb{E}\left[4\|Q_t-Q^{\pi_t}\|_\nu^2
+
\frac{9}{(1-\gamma)^2}\right]\tag{$(a+b)^2\leq 2(a^2+b^2)$ for any $a,b\in\mathbb{R}$}\nonumber\\
= \,&\frac{16\omega_t}{\pi_{b,\min}\sqrt{\nu_{\min}}}\mathbb{E}W_t
+\frac{18\omega_t}{\pi_{b,\min}\sqrt{\nu_{\min}}(1-\gamma)^2}.\label{eq:T21_bound}
\end{align}

\paragraph{The term $T_{22}+T_{23}$.}
By the Cauchy-Schwarz inequality, we have
\begin{align}
T_{22}+T_{23}
\leq \,&
\|Q_t-Q^{\pi_t}\|_\nu
\|\bar{F}(Q_{t-z_t}, \pi_{t-z_t})-\bar{F}(Q_t, \pi_{t-z_t})\|_\nu
\nonumber\\
&+\|Q_t-Q^{\pi_t}\|_\nu
\|F_{\mathrm{IS}}(Q_t, Y_t, \pi_{t+1})-F_{\mathrm{IS}}(Q_{t-z_t}, Y_t, \pi_{t+1})\|_\nu\label{eq:T22+T23}
\end{align}
Since $\bar{F}(\cdot,\pi_{t-z_t})$ is a contraction mapping with respect to $\|\cdot\|_\nu$ (cf. Lemma \ref{lem:F} (1)), we have
\begin{align*}
    \|\bar{F}(Q_{t-z_t}, \pi_{t-z_t})-\bar{F}(Q_t, \pi_{t-z_t})\|_\nu\leq \|Q_t-Q_{t-z_t}\|_\nu.
\end{align*}
Since $F_{\mathrm{IS}}(\cdot,y,\pi)$ is Lipschitz continuous with respect to $\|\cdot\|_\nu$ (cf. Lemma \ref{lem:F} (2)), we have 
\begin{align*}
    \|F_{\mathrm{IS}}(Q_t, Y_t, \pi_{t+1})-F_{\mathrm{IS}}(Q_{t-z_t}, Y_t, \pi_{t+1})\|_\nu\leq \,&\frac{1}{\pi_{b,\min}\sqrt{\nu_{\min}}}
\|Q_t-Q_{t-z_t}\|_\nu.
\end{align*}
Substituting the previous two inequalities in \eqref{eq:T22+T23}, we obtain
\begin{align}
    T_{22}+T_{23}\leq \,&\frac{2}{\pi_{b,\min}\sqrt{\nu_{\min}}}\|Q_t-Q^{\pi_t}\|_\nu
\|Q_t-Q_{t-z_t}\|_\nu\nonumber\\
\leq \,&\frac{4\alpha_{t-z_t,t-1}}{\pi_{b,\min}^2\nu_{\min}}\|Q_t-Q^{\pi_t}\|_\nu\left(2\|Q_{t}-Q^{\pi_{t}}\|_\nu + \frac{3}{1-\gamma}\right)\tag{Lemma \ref{bounds} (3)}\nonumber\\
= \,&\frac{4\alpha_{t-z_t,t-1}}{\pi_{b,\min}^2\nu_{\min}}\left(2\|Q_{t}-Q^{\pi_{t}}\|_\nu^2 + \frac{3}{1-\gamma}\|Q_t-Q^{\pi_t}\|_\nu\right)\nonumber\\
\leq \,&\frac{4\alpha_{t-z_t,t-1}}{\pi_{b,\min}^2\nu_{\min}}\left(2\|Q_{t}-Q^{\pi_{t}}\|_\nu^2 +\|Q_t-Q^{\pi_t}\|_\nu^2+\frac{9}{4(1-\gamma)^2}\right)\tag{The AM-GM inequality}\nonumber\\
=\,&\frac{24\alpha_{t-z_t,t-1}}{\pi_{b,\min}^2\nu_{\min}}W_t+\frac{9\alpha_{t-z_t,t-1}}{\pi_{b,\min}^2\nu_{\min}(1-\gamma)^2}.\label{eq:T2223_bound}
\end{align}

\paragraph{The term $T_{24}+T_{25}$.}
By the Cauchy-Schwarz inequality, we have
\begin{align}
    T_{24}+T_{25}\leq \,&\|Q_t-Q^{\pi_t}\|_\nu
\|\bar{F}(Q_t, \pi_{t-z_t}) - \bar{F}(Q_t, \pi_t)\|_\nu\nonumber\\
&+\|Q_t-Q^{\pi_t}\|_\nu\|{F}(Q_{t-z_t}, Y_t, \pi_{t+1}) - {F}(Q_{t-z_t}, Y_t, \pi_{t-z_t}) \|_\nu\label{eq:T2425:decomposition}
\end{align}
For the term $\|\bar{F}(Q_t, \pi_{t-z_t}) - \bar{F}(Q_t, \pi_t)\|_\nu$ on the right-hand side of \eqref{eq:T2425:decomposition}, we have by Lemma \ref{lem:F} (3) that
\begin{align*}
    \|\bar{F}(Q_t, \pi_{t-z_t}) - \bar{F}(Q_t, \pi_t)\|_\nu\leq\,&\|\bar{F}(Q_t, \pi_{t-z_t}) - \bar{F}(Q_t, \pi_t)\|_\infty\\
    \leq \,&2\omega_{t-z_t,t-1}\|Q_t\|_\infty\\
    \leq \,&2\omega_{t-z_t,t-1}\left(\|Q_t-Q^{\pi_t}\|_\infty+\|Q^{\pi_t}\|_\infty\right)\\
    \leq \,&2\omega_{t-z_t,t-1}\left(\frac{1}{\sqrt{\nu_{\min}}}\|Q_t-Q^{\pi_t}\|_\nu+\frac{1}{1-\gamma}\right)
\end{align*}
For the term $\|{F}(Q_{t-z_t}, Y_t, \pi_{t+1}) - {F}(Q_{t-z_t}, Y_t, \pi_{t-z_t}) \|_\nu$ on the right-hand side of \eqref{eq:T2425:decomposition}, we have by Lemma \ref{lem:F} (4) and Lemma \ref{bounds} (3) that
\begin{align*}
    \|{F}(Q_{t-z_t}, Y_t, \pi_{t+1}) - {F}(Q_{t-z_t}, Y_t, \pi_{t-z_t}) \|_\nu\leq \,& \|{F}(Q_{t-z_t}, Y_t, \pi_{t+1}) - {F}(Q_{t-z_t}, Y_t, \pi_{t-z_t}) \|_\infty\\
    \leq&\, \frac{2\omega_{t-z_t, t}}{\pi_{b,\min}}\|Q_{t-z_t}\|_\infty\\
    \leq \,&\frac{2\omega_{t-z_t, t}}{\pi_{b,\min}}\left(2\|Q_t-Q^{\pi_t}\|_\infty + \frac{3}{1-\gamma}\right)\\
    \leq \,&\frac{2\omega_{t-z_t, t}}{\pi_{b,\min}}\left(\frac{2\|Q_t-Q^{\pi_t}\|_\nu}{\sqrt{\nu_{\min}}} + \frac{3}{1-\gamma}\right).
\end{align*}
Substituting the previous two inequalities in \eqref{eq:T2425:decomposition}, we obtain
\begin{align}
    T_{24}+T_{25}\leq \,&\frac{2\omega_{t-z_t, t}}{\pi_{b,\min}}\|Q_t-Q^{\pi_t}\|_\nu\left[\frac{3}{\sqrt{\nu_{\min}}}\|Q_t-Q^{\pi_t}\|_\nu+\frac{4}{1-\gamma}\right]\nonumber\\
    =\,&\frac{2\omega_{t-z_t, t}}{\pi_{b,\min}}\left[\frac{3}{\sqrt{\nu_{\min}}}\|Q_t-Q^{\pi_t}\|_\nu^2+\frac{4}{1-\gamma}\|Q_t-Q^{\pi_t}\|_\nu\right]\nonumber\\
    \leq \,&\frac{2\omega_{t-z_t, t}}{\pi_{b,\min}}\left[\frac{3}{\sqrt{\nu_{\min}}}\|Q_t-Q^{\pi_t}\|_\nu^2+\frac{1}{\sqrt{\nu_{\min}}}\|Q_t-Q^{\pi_t}\|_\nu^2+\frac{4\sqrt{\nu_{\min}}}{(1-\gamma)^2}\right]\tag{The AM-GM inequality}\nonumber\\
    \leq \,&\frac{16\omega_{t-z_t, t}}{\pi_{b,\min}\sqrt{\nu_{\min}}}W_t+\frac{8\omega_{t-z_t, t}}{\pi_{b,\min}(1-\gamma)^2}.\label{eq:T2425_bound}
\end{align}

\paragraph{The term $T_{26}$.}
By the Cauchy–Schwarz inequality, we have
\begin{align}\label{eq:T25_breakup}
    T_{26}
\le\,&
\left\|
(Q_t - Q^{\pi_t}) - (Q_{t-z_t}-Q^{\pi_{t-z_t}})
\right\|_\nu
\;
\left\|
F_{\mathrm{IS}}(Q_{t-z_t}, Y_t, \pi_{t+1}) - \bar F_{\mathrm{IS}}(Q_{t-z_t}, \pi_{t-z_t})
\right\|_\nu\nonumber\\
\leq \,&\big(\|
Q_t - Q_{t-z_t}\|_\nu+\|Q^{\pi_t}-Q^{\pi_{t-z_t}})
\|_\nu\big)
\;
\left\|F_{\mathrm{IS}}(Q_{t-z_t}, Y_t, \pi_{t+1}) - \bar F_{\mathrm{IS}}(Q_{t-z_t}, \pi_{t-z_t})\right\|_\nu
\end{align}
For the first term on the right-hand side of \eqref{eq:T25_breakup}, we have by Lemma \ref{bounds} (3) and (5) that
\begin{align}
\|Q_t-Q_{t-z_t}\|_\nu + \|Q^{\pi_t}-Q^{\pi_{t-z_t}}\|_\nu 
\le\,&
\frac{4\alpha_{t-z_t,t-1}}{\pi_{b,\min}\sqrt{\nu_{\min}}}
\|Q_t-Q^{\pi_t}\|_\nu
+
\frac{6\alpha_{t-z_t,t-1}}{(1-\gamma)\pi_{b,\min}\sqrt{\nu_{\min}}}
+
\frac{2\omega_{t-z_t,t-1}}{(1-\gamma)^2}.\label{eq:T25_first_factor}
\end{align}
For the second term on the right-hand side of \eqref{eq:T25_breakup}, we have by Lemma \ref{bounds} (1) and (2) that
\begin{align}
\|F_{\mathrm{IS}}(Q_{t-z_t}, Y_t, \pi_{t+1}) - \bar F_{\mathrm{IS}}(Q_{t-z_t}, \pi_{t-z_t})\|_\nu &\leq \|F_{\mathrm{IS}}(Q_{t-z_t}, Y_t, \pi_{t+1})\|_\nu+\|\bar{F}(Q_{t-z_t}, \pi_{t-z_t})\|_\nu \nonumber\\
&\le 2 + \frac{\|Q_{t-z_t}\|_\nu}{\pi_{b,\min}\sqrt{\nu_{\min}}} + \|Q_{t-z_t}\|_\nu \nonumber\\
&\le 2 + \frac{3\|Q_{t-z_t}\|_\nu}{2\pi_{b,\min}\sqrt{\nu_{\min}}}\tag{$\pi_{b,\min}\leq 1/2$}\nonumber\\
&\le 2 + \frac{3}{2\pi_{b,\min}\sqrt{\nu_{\min}}}\left(\frac{3}{1-\gamma}+2\|Q_t-Q^{\pi_t}\|_\nu\right)\nonumber\\
&\leq \frac{3\|Q_t-Q^{\pi_t}\|_\nu}{\pi_{b,\min}\sqrt{\nu_{\min}}} + \frac{5}{(1-\gamma)\pi_{b,\min}\sqrt{\nu_{\min}}}. \label{T25_second_factor}
\end{align}
Substituting the bounds \eqref{eq:T25_first_factor} and \eqref{T25_second_factor} in \eqref{eq:T25_breakup},  since $\omega_t\leq \alpha_t(1-\gamma)$ and $\pi_{b,\min}\leq 1/2$, we have
\begin{align}
    T_{26}\leq \,&\left(\frac{4\alpha_{t-z_t,t-1}}{\pi_{b,\min}\sqrt{\nu_{\min}}}
\|Q_t-Q^{\pi_t}\|_\nu
+
\frac{6\alpha_{t-z_t,t-1}}{(1-\gamma)\pi_{b,\min}\sqrt{\nu_{\min}}}
+
\frac{2\omega_{t-z_t,t-1}}{(1-\gamma)^2}\right)\nonumber\\
&\times \left(\frac{3\|Q_t-Q^{\pi_t}\|_\nu}{\pi_{b,\min}\sqrt{\nu_{\min}}} + \frac{5}{(1-\gamma)\pi_{b,\min}\sqrt{\nu_{\min}}}\right)\nonumber\\
\leq \,&\alpha_{t-z_t,t-1}\left(\frac{4}{\pi_{b,\min}\sqrt{\nu_{\min}}}
\|Q_t-Q^{\pi_t}\|_\nu
+
\frac{7}{(1-\gamma)\pi_{b,\min}\sqrt{\nu_{\min}}}\right)^2\nonumber\\
\leq \,&\frac{64\alpha_{t-z_t,t-1}}{\pi_{b,\min}^2\nu_{\min}}
W_t
+
\frac{98\alpha_{t-z_t,t-1}}{(1-\gamma)^2\pi_{b,\min}^2\nu_{\min}},\label{eq:T26_bound}
\end{align}
where the last line follows from $(a+b)^2\leq 2(a^2+b^2)$.

\paragraph{Combining everything together.} Substituting \eqref{eq:T21_bound}, \eqref{eq:T2223_bound}, \eqref{eq:T2425_bound} and \eqref{eq:T26_bound} in \eqref{eq:T2_decomposition}, we obtain
\begin{align*}
    \mathbb{E}\left\langle Q_t - Q^{\pi_t}, F_{\mathrm{IS}}(Q_t, Y_t, \pi_{t+1}) - \bar{F}(Q_t, \pi_t) \right\rangle\leq \,&\frac{16\omega_t}{\pi_{b,\min}\sqrt{\nu_{\min}}}\mathbb{E}W_t
+\frac{18\omega_t}{\pi_{b,\min}\sqrt{\nu_{\min}}(1-\gamma)^2}\\
&+\frac{24\alpha_{t-z_t,t-1}}{\pi_{b,\min}^2\nu_{\min}}\mathbb{E}W_t+\frac{9\alpha_{t-z_t,t-1}}{\pi_{b,\min}^2\nu_{\min}(1-\gamma)^2}\\
&+\frac{16\omega_{t-z_t, t}}{\pi_{b,\min}\sqrt{\nu_{\min}}}\mathbb{E}W_t+\frac{8\omega_{t-z_t, t}}{\pi_{b,\min}(1-\gamma)^2}\\
&+\frac{64\alpha_{t-z_t,t-1}}{\pi_{b,\min}^2\nu_{\min}}
\mathbb{E}W_t
+
\frac{98\alpha_{t-z_t,t-1}}{(1-\gamma)^2\pi_{b,\min}^2\nu_{\min}}\\
\leq \,&\frac{90\alpha_{t-z_t,t-1}}{\pi_{b,\min}^2\nu_{\min}}
\mathbb{E}W_t
+
\frac{100\alpha_{t-z_t,t-1}}{(1-\gamma)^2\pi_{b,\min}^2\nu_{\min}},
\end{align*}
where the last inequality follows from
\begin{align*}
    \omega_t\leq \frac{\alpha_{t-1,t-z_t}}{16\pi_{b,\min}\sqrt{\nu_{\min}}},
\end{align*}
which itself is a consequence of the stepsize condition $\omega_t\leq (1-\gamma)^3(1-\gamma_c)\nu_{\min}\alpha_t/20$. It now follows that
\begin{align*}
    \mathbb{E}T_2\leq \frac{90\alpha_t\alpha_{t-z_t,t-1}}{\pi_{b,\min}^2\nu_{\min}}
\mathbb{E}W_t
+
\frac{100\alpha_t\alpha_{t-z_t,t-1}}{(1-\gamma)^2\pi_{b,\min}^2\nu_{\min}}.
\end{align*}

\subsubsection{Proof of Lemma \ref{lem:target}}\label{app:target}

The following lemma will be useful in proving Lemma~\ref{lem:target}; its proof is presented in Appendix~\ref{ap:pf:lem:target_shift}.
\begin{lemma}\label{lem:target_shift}
For the policy iterates $\{\pi_t\}$ generated by Algorithm \ref{alg}, we have
\begin{align*} 
\|Q^{\pi_t}-Q^{\pi_{t+1}}\|_\infty
\le
\frac{\omega_t}{1-\gamma}
\big(
\|Q^*-Q^{\pi_t}\|_\infty
+
2\|Q_t-Q^{\pi_t}\|_\infty
+
\chi_t
\big),\quad\forall\,t\geq 0.
\end{align*}
\end{lemma}

Returning to the proof of Lemma \ref{lem:target}, observe that
\begin{align} 
T_3
=\,&\langle Q_t - Q^{\pi_t}, Q^{\pi_t}-Q^{\pi_{t+1}}\rangle_\nu\nonumber\\
\leq \,&\|Q_t - Q^{\pi_t}\|_\nu
\|Q^{\pi_t}-Q^{\pi_{t+1}}\|_\nu\nonumber\\
\leq \,&
\|Q_t - Q^{\pi_t}\|_\nu
\|Q^{\pi_t}-Q^{\pi_{t+1}}\|_\infty\nonumber\\
\leq \,&\frac{\omega_t}{1-\gamma}
\|Q_t-Q^{\pi_t}\|_\nu
\left(2\|Q_t-Q^{\pi_t}\|_\infty+
\|Q^*-Q^{\pi_t}\|_\infty+
\chi_t
\right)\tag{Lemma \ref{lem:target_shift}}\nonumber\\
\leq \,&\frac{\omega_t}{1-\gamma}
\|Q_t-Q^{\pi_t}\|_\nu
\left(\frac{2}{\sqrt{\nu_{\min}}}\|Q_t-Q^{\pi_t}\|_\nu+
\|Q^*-Q^{\pi_t}\|_\infty+
\chi_t
\right)\nonumber\\
=\,&\frac{4\omega_t}{(1-\gamma)\sqrt{\nu_{\min}}}W_t+\frac{\omega_t}{1-\gamma}
\|Q_t-Q^{\pi_t}\|_\nu
\|Q^*-Q^{\pi_t}\|_\infty+\frac{\omega_t}{1-\gamma}
\|Q_t-Q^{\pi_t}\|_\nu
\chi_t.\label{T25}
 \end{align}

To proceed, note that by the AM-GM inequality, we have for any $C_1,C_2>0$ that
\begin{align*}
\|Q_t-Q^{\pi_t}\|_\nu
\|Q^*-Q^{\pi_t}\|_\infty
&\le
\frac{C_1}{2}
\|Q_t-Q^{\pi_t}\|_\nu^2
+
\frac{1}{2C_1}
\|Q^*-Q^{\pi_t}\|_\infty^2,
\\
\|Q_t-Q^{\pi_t}\|_\nu
\chi_t
&\le
\frac{C_2}{2}
\|Q_t-Q^{\pi_t}\|_\nu^2
+
\frac{1}{2C_2}
\chi_t^2.
\end{align*}
Substituting the above in \eqref{T25} yields
\begin{align*}
T_3
\le\,&
\frac{\omega_t}{1-\gamma}
\Bigg(
\left(
\frac{2}{\sqrt{\nu_{\min}}}
+
\frac{C_1}{2}
+
\frac{C_2}{2}
\right)
\|Q_t-Q^{\pi_t}\|_\nu^2
+
\frac{1}{2C_1}
\|Q^*-Q^{\pi_t}\|_\infty^2
+
\frac{1}{2C_2}
\chi_t^2
\Bigg)\\
=\,&
\frac{\omega_t}{1-\gamma}
\Bigg(
\left(
\frac{4}{\sqrt{\nu_{\min}}}
+
C_1
+
C_2
\right)
W_t
+
\frac{1}{2C_1}
\|Q^*-Q^{\pi_t}\|_\infty^2
+
\frac{1}{2C_2}
\chi_t^2
\Bigg)\\
=\,&\frac{8\omega_t}{(1-\gamma)^3\sqrt{\nu_{\min}}}W_t+\frac{1-\gamma}{4}\omega_t\|Q^*-Q^{\pi_t}\|_\infty^2+\frac{1-\gamma}{4}\omega_t\chi_t^2,
\end{align*}
where the last inequality follows by choosing  $C_1=C_2=2/(1-\gamma)^2$.

\subsubsection{Proof of Lemma \ref{lem:residual}}\label{app:res}

Observe that
\begin{equation}
T_4
=
\frac{1}{2}
\|(Q_{t+1}-Q_t)+(Q^{\pi_t}-Q^{\pi_{t+1}})\|_\nu^2
\le
\|Q_{t+1}-Q_t\|_\nu^2
+
\|Q^{\pi_t}-Q^{\pi_{t+1}}\|_\nu^2.\label{26}
\end{equation}
For the first term on the right-hand side of \eqref{26}, the IS-based critic update, that
\begin{align}
\|Q_{t+1}-Q_t\|_\nu^2
=\,&
\alpha_t^2\nu_{(S_t,A_t)}
\big(\mathcal{R}(S_t,A_t)
+
\gamma\rho_{t+1}Q_t(S_{t+1},A_{t+1})
-
Q_t(S_t,A_t)\big)^2\nonumber\\
\leq \,&\alpha_t^2
\left(1
+
\frac{1}{\pi_{b,\min}}\|Q_t\|_\infty
+
\|Q_t\|_\infty\right)^2\nonumber\\
\leq \,&\alpha_t^2
\left(1
+
\frac{3}{2\pi_{b,\min}}\|Q_t\|_\infty
\right)^2\tag{$\pi_{b,\min}\leq 1/2$}\nonumber\\
\leq \,&\alpha_t^2
\left(1
+
\frac{3}{2\pi_{b,\min}}\|Q_t-Q^{\pi_t}\|_\infty
+\frac{3}{2\pi_{b,\min}}\|Q^{\pi_t}\|_\infty\right)^2\nonumber\\
\leq \,&\alpha_t^2
\left(1
+
\frac{3}{2\pi_{b,\min}\sqrt{\nu_{\min}}}\|Q_t-Q^{\pi_t}\|_\nu
+\frac{3}{2\pi_{b,\min}(1-\gamma)}\right)^2\nonumber\\
\leq \,&\alpha_t^2
\left(
\frac{3}{2\pi_{b,\min}\sqrt{\nu_{\min}}}\|Q_t-Q^{\pi_t}\|_\nu
+\frac{2}{\pi_{b,\min}(1-\gamma)}\right)^2\tag{$\pi_{b,\min}\leq 1/2$}\nonumber\\
\leq \,&
\frac{9\alpha_t^2}{\pi_{b,\min}^2\nu_{\min}}W_t
+\frac{8\alpha_t^2}{\pi_{b,\min}^2(1-\gamma)^2}.\label{28}
\end{align}

For the second term on the right-hand side of \eqref{26}, we have by Lemma \ref{lem:target_shift} that
\begin{align}
\|Q^{\pi_t}-Q^{\pi_{t+1}}\|_\nu^2
\le\,&
\|Q^{\pi_t}-Q^{\pi_{t+1}}\|_\infty^2\nonumber\\
\le\,&
\frac{\omega_t^2}{(1-\gamma)^2}\left(\|Q^*-Q^{\pi_t}\|_\infty
+
2\|Q_t-Q^{\pi_t}\|_\infty
+
\chi_t\right)^2\nonumber\\
\le\,&
\frac{\omega_t^2}{(1-\gamma)^2}\left(\|Q^*-Q^{\pi_t}\|_\infty
+
\frac{2}{\sqrt{\nu_{\min}}}\|Q_t-Q^{\pi_t}\|_\nu
+
\chi_t\right)^2\nonumber\\
=\,&\frac{24\omega_t^2}{\nu_{\min}(1-\gamma)^2}W_t+\frac{3\omega_t^2}{(1-\gamma)^2}\|Q^*-Q^{\pi_t}\|_\infty^2
+
\frac{3\omega_t^2\chi_t^2}{(1-\gamma)^2},\label{29}
\end{align}
where the last inequality follows from $(a+b+c)^2\leq 3(a^2+b^2+c^2)$ for any $a,b,c\in\mathbb{R}$.

Substituting the two bounds \eqref{28} and \eqref{29} in \eqref{26} yields the desired result.

\subsubsection{Proof of Lemma \ref{bounds}}\label{ap:pf:bounds}
\begin{enumerate}[(1)]
    \item By definition of the operator $F_\mathrm{IS}$, for any $Q \in \mathbb{R}^{mn}$, $y = (s_1,a_1,s_2,a_2) \in \mathcal{Y}$, and policy $\pi\in{(\Delta^m)}^n$, we have for any $(s,a)$ that
\begin{align*}
    \left|[F_{\mathrm{IS}}(Q, y, \pi)](s,a)\right|\leq \,&(1-\mathds{1}_{\{(s,a)= (s_1,a_1)\}})|Q(s,a)|\\
    &+\mathds{1}_{\{(s,a)=(s_1,a_1)\}}\left|\mathcal{R}(s_1,a_1) + \gamma\frac{\pi(a_2\mid s_2)}{\pi_b(a_2\mid s_2)}Q(s_2,a_2)\right|\\
    \leq \,&(1-\mathds{1}_{\{(s,a)= (s_1,a_1)\}})\|Q\|_\infty+\mathds{1}_{\{(s,a)=(s_1,a_1)\}}\left(1+ \frac{1}{\pi_{b,\min}}\|Q\|_\infty\right)\\
    \leq \,&1+ \frac{1}{\pi_{b,\min}}\|Q\|_\infty.
\end{align*}
It follows that
\begin{align*}
    \|F_{\mathrm{IS}}(Q, y, \pi)\|_\infty \leq 1+ \frac{1}{\pi_{b,\min}}\|Q\|_\infty.
\end{align*}
Since $\|\cdot\|_\nu \leq \|\cdot\|_\infty \leq \|\cdot\|_\nu/\sqrt{\nu_{\min}}$, we have
\begin{align*}
    \|F_{\mathrm{IS}}(Q, y, \pi)\|_\nu \leq 1+ \frac{1}{\pi_{b,\min}\sqrt{\nu_{\min}}}\|Q\|_\nu.
\end{align*}

\item 
For any policy $\pi$, since $\bar{F}(\cdot,\pi)$ is a contraction mapping with respect to $\|\cdot\|_\nu$ (cf. Lemma \ref{lem:F}), we have
\begin{align*} 
\|\bar{F}(Q, \pi)\|_\nu
&\le
\|\bar{F}(0, \pi)\|_\nu
+
\|\bar{F}(Q, \pi_{t+1})-\bar{F}(0, \pi)\|_\nu\\
&\le \|\bar{F}(0, \pi)\|_\nu+
\|Q\|_\nu\\
&\leq 1+\|Q\|_\nu,
\end{align*}
where the last inequality follows from $\|\bar{F}(0, \pi)\|_\nu\leq   1$.
\item 
The first two inequalities are direct consequences of \cite[Lemma 4]{chen2024lyapunov}. Since
\begin{align*} 
\|Q_{t_1} - Q_{t_2}\|_\nu
\le\,&
4\alpha_{t_1,t_2-1}
+
\frac{4\alpha_{t_1,t_2-1}}{\pi_{b,\min}\sqrt{\nu_{\min}}}
\|Q_{t_2}\|_\nu\\
\le\,&
4\alpha_{t_1,t_2-1}
+
\frac{4\alpha_{t_1,t_2-1}}{\pi_{b,\min}\sqrt{\nu_{\min}}}
\|Q_{t_2} - Q^{\pi_{t_2}}\|_\nu
+
\frac{4\alpha_{t_1,t_2-1}}{\pi_{b,\min}\sqrt{\nu_{\min}}}
\|Q^{\pi_{t_2}}\|_\nu\\
\leq \,&
\frac{4\alpha_{t_1,t_2-1}}{\pi_{b,\min}\sqrt{\nu_{\min}}}
\|Q_{t_2} - Q^{\pi_{t_2}}\|_\nu
+
\frac{6\alpha_{t_1,t_2-1}}{\pi_{b,\min}\sqrt{\nu_{\min}}(1-\gamma)}, \tag{$\pi_{b,\min} \leq 1/2$}
\end{align*}
we obtain the third inequality. The fourth inequality follows by an identical argument.

Using the triangle inequality, we have 
\begin{align*} 
\|Q_{t_1}\|_\nu
\le\,&
\|Q_{t_1}-Q_{t_2}\|_\nu
+
\|Q_{t_2}-Q^{\pi_{t_2}}\|_\nu
+
\|Q^{\pi_{t_2}}\|_\nu\\
\leq \,&\frac{4\alpha_{t_1,t_2-1}}{\pi_{b,\min}\sqrt{\nu_{\min}}}
\|Q_{t_2} - Q^{\pi_{t_2}}\|_\nu
+
\frac{6\alpha_{t_1,t_2-1}}{\pi_{b,\min}\sqrt{\nu_{\min}}(1-\gamma)}
+\|Q_{t_2} - Q^{\pi_{t_2}}\|_\nu
+\frac{1}{1-\gamma}\\
\leq \,&2\|Q_{t_2} - Q^{\pi_{t_2}}\|_\nu
+\frac{3}{1-\gamma},
\end{align*}
where the last inequality follows from $\alpha_{t_1,t_2-1}\leq \pi_{b,\min}\sqrt{\nu_{\min}}/4$. This proves the fifth inequality. The last inequality follows by an identical argument.

 \item By the definition of the total variation distance, we have for any $s\in\mathcal{S}$ that
\begin{align*}
    \mathrm{d_{TV}}\big(\pi_{t_1}(s),\pi_{t_2}(s)\big) = \frac{1}{2}\max_s\|\pi_{t_1}(s)-\pi_{t_2}(s)\|_1 \leq \frac{1}{2}\max_s\sum_{t=t_1}^{t_2-1}\|\pi_{t+1}(s)-\pi_t(s)\|_1.
\end{align*}
Since $\|\pi_{t+1}(s)-\pi_t(s)\|_1=\omega_t\|\tilde{\pi}_t(s)-\pi_t(s)\|_1\leq 2\omega_t$,
we have 
\begin{align*}
    \max_{s\in\mathcal{S}}\mathrm{d_{TV}}(\pi_{t_1}(s),\pi_{t_2}(s)) \leq \omega_{t_1,t_2-1}.
\end{align*}
\item Using the Bellman equations for $Q^{\pi_{t_1}}$ and $Q^{\pi_{t_2}}$, we have
\begin{align*}
\|Q^{\pi_{t_1}} - Q^{\pi_{t_2}}\|_\infty
&=
\|\mathcal H_{\pi_{t_1}}Q^{\pi_{t_1}}
-
\mathcal H_{\pi_{t_2}}Q^{\pi_{t_2}}\|_\infty \\
&\le \|\mathcal H_{\pi_{t_1}}Q^{\pi_{t_1}}
-
\mathcal H_{\pi_{t_1}}Q^{\pi_{t_2}}\|_\infty + 
\|\mathcal H_{\pi_{t_1}}Q^{\pi_{t_2}}
-
\mathcal H_{\pi_{t_2}}Q^{\pi_{t_2}}\|_\infty
\\
&\le 
\gamma
\|Q^{\pi_{t_1}} - Q^{\pi_{t_2}}\|_\infty
+
\|\mathcal H_{\pi_{t_1}}Q^{\pi_{t_2}}
-
\mathcal H_{\pi_{t_2}}Q^{\pi_{t_2}}\|_\infty.
\end{align*}
Rearranging terms, we obtain the following where $\widehat{P}_\pi$ denotes the state-action transition probability matrix under policy $\pi$:
\begin{align*}
    \|Q^{\pi_{t_1}} - Q^{\pi_{t_2}}\|_\infty
\le\,&
\frac{1}{1-\gamma}
\|\mathcal H_{\pi_{t_1}}Q^{\pi_{t_2}}
-
\mathcal H_{\pi_{t_2}}Q^{\pi_{t_2}}\|_\infty\\
=\,&\frac{\gamma}{1-\gamma}
\|(\widehat{P}_{\pi_{t_1}}-\widehat{P}_{\pi_{t_2}})Q^{\pi_{t_2}}\|_\infty\\
\leq \,&\frac{1}{(1-\gamma)^2}
\|\widehat{P}_{\pi_{t_1}}-\widehat{P}_{\pi_{t_2}}\|_\infty
\end{align*}

Since
\begin{align*}
    \|\widehat{P}_{\pi_{t_1}}-\widehat{P}_{\pi_{t_2}}\|_\infty=\,&\max_{s,a}\sum_{s'}p(s'\mid s,a)\sum_{a'}\left|\pi_{t_1}(a'\mid s')-\pi_{t_2}(a'\mid s')\right|\\
    = \,&2\max_{s,a}\sum_{s'}p(s'\mid s,a)\mathrm{d_{TV}}(\pi_{t_1}(s'),\pi_{t_2}(s'))\\
    \leq \,&2\omega_{t_1-1,t_2},
\end{align*}
where the last inequality follows from Item (6) of this lemma, we have
\begin{align*}
    \|Q^{\pi_{t_1}} - Q^{\pi_{t_2}}\|_\infty\leq \frac{2\omega_{t_1,t_2-1}}{(1-\gamma)^2}.
\end{align*}

\end{enumerate}

\subsubsection{Proof of Lemma \ref{lem:target_shift}}\label{ap:pf:lem:target_shift}
Recall that $\delta_t
=\max_{s,a}\big(Q^{\pi_{t}}(s,a)-Q^{\pi_{t+1}}(s,a)\big)$. Let $\delta_t'=
\max_{s,a}
\big(Q^{\pi_{t+1}}(s,a)
-
Q^{\pi_t}(s,a)\big)$. Then, we have $\|Q^{\pi_t}-Q^{\pi_{t+1}}\|_\infty= \max(\delta_t,\delta_t')$.

For $\delta_t$, we have shown in \eqref{eq:delta_bound} that
\begin{align}\label{eq1:lem:target_shift}
\delta_t
\le
\frac{2\omega_t}{1-\gamma}
\|Q_t-Q^{\pi_t}\|_\infty
+
\omega_t\frac{\chi_t}{1-\gamma}.
\end{align}

To bound $\delta_t'$, by the monotonicity and translation invariance of the Bellman operator $\mathcal{H}_{\pi_{t+1}}$, we have
\begin{align}
Q^{\pi_{t+1}}
=
\mathcal{H}_{\pi_{t+1}} Q^{\pi_{t+1}}
\le \mathcal{H}_{\pi_{t+1}}(Q^{\pi_t} + \delta_t' \mathbf{1})
=
\mathcal{H}_{\pi_{t+1}} Q^{\pi_t}
+
\gamma \delta_t' \mathbf{1}. \label{eq:initial_improvement_bound}
\end{align}
To further bound $\mathcal{H}_{\pi_{t+1}} Q^{\pi_t}$, observe that
\begin{align*}
   \mathcal{H}_{\pi_{t+1}} Q^{\pi_t}
   =\,& (1 - \omega_t) Q^{\pi_t} + \omega_t \mathcal{H}_{\tilde{\pi}_t} Q^{\pi_t} \tag{Lemma \ref{lem:Bellman_Affine}}\\
   \leq\,& (1 - \omega_t) Q^{\pi_t} + \omega_t \mathcal{H} Q^{\pi_t} \tag{$\mathcal{H}_{\tilde{\pi}_t} Q^{\pi_t} \leq \mathcal{H} Q^{\pi_t}$}\\
   \leq\,& (1 - \omega_t) Q^{\pi_t} + \omega_t \mathcal{H} Q^* \tag{$\mathcal{H} Q^{\pi_t} \leq \mathcal{H} Q^*$}\\
   =\,& Q^{\pi_t} + \omega_t (Q^* - Q^{\pi_t}) \tag{$\mathcal{H} Q^* = Q^*$}.
\end{align*}
Combining the previous inequality with \eqref{eq:initial_improvement_bound}, yields $\delta_t' \leq \omega_t \|Q^* - Q^{\pi_t}\|_\infty + \gamma \delta_t'$.
Rearranging terms, we obtain
\begin{align}\label{eq2:lem:target_shift}
    \delta_t' \leq \frac{\omega_t \|Q^* - Q^{\pi_t}\|_\infty}{1 - \gamma}.
\end{align}

In view of \eqref{eq1:lem:target_shift} and \eqref{eq2:lem:target_shift}, we have
\begin{align*}
    \|Q^{\pi_t}-Q^{\pi_{t+1}}\|_\infty
    = \max(\delta_t,\delta_t')
    \leq \delta_t + \delta_t'
    = \frac{\omega_t}{1-\gamma}\left(\|Q^*-Q^{\pi_t}\|_\infty
    +
    2\|Q_t-Q^{\pi_t}\|_\infty
    +
    \chi_t\right).
\end{align*}

\section{Proof of Theorem \ref{thm:main} with the ETD-Based Critic}\label{app:ETD}

\subsection{Analysis of the Actor}\label{app:actor_ETD}
We begin by stating the actor drift inequalities.
\begin{proposition}
\label{prop:actor_ETD}
The following inequality holds for all $t$:
\begin{align*}
&\infnorm{Q^* - Q^{\pi_{t+1}}}
\le
\underbrace{\big(1 - \omega_t(1-\gamma)\big) \infnorm{Q^* - Q^{\pi_t}}}_{\text{actor drift}} +
\underbrace{\frac{2\omega_t}{1-\gamma} \infnorm{Q_t - Q^{\pi_t}}}_{\text{critic coupling error}}
+
\underbrace{\frac{\omega_t}{1-\gamma}\chi_t}_{\text{temperature error}}.
\end{align*} 
\end{proposition}

The following corollary follows, and is needed to combine the actor drift with a drift inequality for the critic.

\begin{corollary}\label{cor:actor_ETD}
    The following inequality holds for any $t\geq 0$:
    \begin{align*} \infnorm{Q^*-Q^{\pi_{t+1}}}^2 \leq \big(1-\omega_t(1-\gamma)\big)\infnorm{Q^*-Q^{\pi_t}}^2 + \frac{6\omega_t}{(1-\gamma)^3}\infnorm{Q_t-Q^{\pi_t}}^2 + \frac{5\omega_t}{(1-\gamma)^3}\chi^2_t. \end{align*}
\end{corollary}

Since our actor analysis methodology is agnostic of the critic choice, the statements for Proposition \ref{prop:actor_ETD} and Corollary \ref{cor:actor_ETD}, along with their proofs, are identical to those of Proposition \ref{prop:actor} and Corollary \ref{cor:actor}, which are proved in Appendix \ref{app:actor}. We hence omit the proofs here.

\subsection{Analysis of the Critic}\label{app:critic_ETD}

We begin by reformulating the ETD-based critic update as a stochastic approximation algorithm. For any $t\geq 0$, let $\{U_t\}$ be a stochastic process defined as $U_t=(S_t,A_t,S_{t+1})$. It is clear that $\{U_t\}$ is a Markov chain with a finite state space, denoted by $\mathcal{U}$. Moreover, under Assumption~\ref{assum:behavior}, the Markov chain $\{U_t\}$ admits a unique stationary distribution, denoted by $\mu_U$, which satisfies $\mu_U(s_1,a_1,s_2)=\mu_b(s_1)\pi_b(a_1|s_1)p(s_2|s_1,a_1)$. Let $F_\mathrm{ETD}:\mathbb{R}^{mn}\times \mathcal{U}\times {(\Delta^{m})}^n\to \mathbb{R}^{mn}$ be an operator such that given inputs $Q \in \mathbb{R}^{mn}$,
$u = (s_1,a_1,s_2) \in \mathcal{U}$ and $\pi\in{(\Delta^{m})}^n$, the $(s,a)$-th component of the output is defined as
\begin{align*} 
[F_\mathrm{ETD}(Q,u,\pi)](s,a) =&\mathds{1}_{\{ (s,a)= (s_1,a_1)\}}\left(\mathcal{R}(s_1,a_1) + \gamma \sum_{a\in\mathcal{A}}\pi(a|s_2) Q(s_2,a)\right)+\mathds{1}_{\{ (s,a)\neq (s_1,a_1)\}}Q(s,a).
 \end{align*}
The critic update in Algorithm \ref{alg}, Line 8, under the ETD-based critic choice, 
can now be written compactly as
\begin{align*} 
Q_{t+1} = Q_t + \alpha_t \bigl(F_\mathrm{ETD}(Q_t, U_t, \pi_{t+1}) - Q_t\bigr).
 \end{align*}

Recall the operator $\bar{F}:\mathbb{R}^{mn}\times{(\Delta^{m})}^n\to\mathbb{R}^{mn}$ defined in Section \ref{subsec:analysis_critic} as $\bar{F}(Q,\pi) = \mathbb{E}_{Y\sim\mu_Y}F_\mathrm{IS}(Q,Y,\pi)$, where we recall that $\mu_Y$ is the stationary distribution of the Markov chain $\{Y_t\}$ defined as $Y_t =(U_t,A_{t+1})$, which satisfies $\mu_Y(s_1,a_1,s_2,a_2) = \mu_U(s_1,a_1,s_2)\pi_{b}(a_2\mid s_2)$. We claim that $\mathbb{E}_{U\sim\mu_U}F_\mathrm{ETD}(Q,U,\pi) = \bar{F}(Q,\pi)$, implying that the expected operators for both the updates coincide. To see this, note that for $U = (s_1,a_1,s_2), Y = (u,a_2)$, we have
\begin{align*}
    \bar{F}(Q,\pi) =\,& \mathbb{E}_{Y\sim\mu_Y}\left[\mathds{1}_{\{ (s,a)= (s_1,a_1)\}}\left(\mathcal{R}(s_1,a_1) + \gamma \frac{\pi(a_2|s_2)}{\pi_b(a_2|s_2)} Q(s_2,a_2)\right)+\mathds{1}_{\{ (s,a)\neq (s_1,a_1)\}}Q(s,a)\right]\\
    =\,& \mathbb{E}_{U\sim\mu_U}\Bigg[\mathds{1}_{\{ (s,a)= (s_1,a_1)\}}\left(\mathcal{R}(s_1,a_1) + \gamma \mathbb{E}_{a_2\sim\pi_{b}(\cdot\mid s_2)}\left[ \frac{\pi(a_2|s_2)}{\pi_b(a_2|s_2)} Q(s_2,a_2)\right]\right)\Bigg]\\
    &+\mathbb{E}_{U\sim\mu_{U}}\big[\mathds{1}_{\{ (s,a)\neq (s_1,a_1)\}}Q(s,a)\big]\\
    =\,& \mathbb{E}_{U\sim\mu_U}\Big[\mathds{1}_{\{ (s,a)= (s_1,a_1)\}}\big(\mathcal{R}(s_1,a_1) + \gamma \mathbb{E}_{a_2\sim\pi(\cdot\mid s_2)}Q(s_2,a_2)\big)\Big]\\
    &+\mathbb{E}_{U\sim\mu_{U}}\big[\mathds{1}_{\{ (s,a)\neq (s_1,a_1)\}}Q(s,a)\big]\\
    =\,&\mathbb{E}_{U\sim\mu_{U}}F_\mathrm{ETD}(Q,U,\pi).
\end{align*}

It is now clear that the ETD-based critic update is a Markovian stochastic approximation scheme for tracking the solution to the (time-varying) fixed-point equation $\bar{F}(Q,\pi_t)=Q$. 

Next, we present several key properties of $F_\mathrm{ETD}$ and $\bar{F}$ that facilitate the convergence analysis. In particular, we show that the equation $\bar{F}(Q,\pi_t)=Q$ admits $Q^{\pi_t}$ as its unique solution, and that $\bar{F}(\cdot,\pi)$ is a contractive operator with respect to a weighted $\ell_2$ norm. The proof of the following result is provided in Appendix~\ref{app:F_ETD}.

\begin{lemma}\label{lem:F_ETD}
There exists $\nu \in \Delta^{nm}$ with
$\nu_{\min} \ge (1-\gamma)\mu_{b,\min}\pi_{b,\min}/(nm)$
such that:
\begin{enumerate}[(1)]
\item $\bar{F}(\cdot, \pi)$ is $\gamma_c$-contractive in $\|\cdot\|_\nu$ for all $\pi\in{(\Delta^{m})}^n$, where $\gamma_c = \sqrt{1-(1-\gamma)\mu_{b,\min}\pi_{b,\min}}$. Moreover, the fixed-point equation $\bar{F}(Q,\pi)=Q$ admits a unique solution $Q^{\pi}$.
\item $F_\mathrm{ETD}(\cdot, u,\pi)$ is $1/\sqrt{\nu_{\min}}$-Lipschitz in the weighted $\ell_2$ norm $\|\cdot\|_\nu$ and $1$-Lipschitz in $\|\cdot\|_\infty$, uniformly for all $u\in\mathcal{U}$ and $\pi\in{(\Delta^{m})}^n$.
\item For all non-negative integers $t_1<t_2$, $\|\bar{F}(Q,\pi_{t_1}) - \bar{F}(Q,\pi_{t_2})\|_\infty \leq 2\omega_{t_1,t_2-1}\|Q\|_\infty$ for all $Q\in\mathbb{R}^{mn}$.
\item For all non-negative integers $t_1<t_2$, $\|F_\mathrm{ETD}(Q,u,\pi_{t_1}) - F_\mathrm{ETD}(Q,u,\pi_{t_2})\|_\infty \leq 2\omega_{t_1,t_2-1}\|Q\|_\infty$ for all $u\in\mathcal{U}$ and $Q\in\mathbb{R}^{mn}$.
\end{enumerate}
\end{lemma}

We now state and prove the drift inequality for the ETD-based critic. Recall that $z_t = \min\{k\geq 1\mid \max_{s\in\mathcal{S}}\mathrm{d}_{\mathrm{TV}}\big(P_{\pi_b}^{k-1}(s,\cdot),\mu_b(\cdot)\big) \leq \omega_t/2\}$ and $K = \min\{t\in\mathbb{N}\mid t\geq z_t\}$.

\begin{proposition}\label{prop:critic_ETD}
    Under Assumption \ref{assum:behavior}, suppose that the stepsizes are non-increasing and satisfy
    \begin{align*}
        \omega_t\leq \frac{(1-\gamma)^3(1-\gamma_c)\nu_{\min}}{16}\alpha_t.
    \end{align*}
    Then, the following inequality holds for any $t\geq K$:
    \begin{align*}\mathbb{E}W_{t+1} &\leq \big(1-\alpha_t(1-\gamma_c)\big)\mathbb{E}W_t+ \frac{1-\gamma}{2}\omega_t\mathbb{E}\infnorm{Q^*-Q^{\pi_t}}^2 +\frac{1-\gamma}{2}\omega_t\mathbb{E}\chi_t^2+ 
    \frac{7\alpha_t\alpha_{t-z_t,t-1}}{(1-\gamma)^2}.\end{align*}
\end{proposition}

\begin{proof}[Proof of Proposition \ref{prop:critic_ETD}]
We have, by the binomial decomposition, that 
    \begin{align}
        W_{t+1} 
        =& W_t + \underbrace{\alpha_t\langle Q_t-Q^{\pi_t}, \bar{F}(Q_t, \pi_t) - Q_t \rangle_\nu}_{T_1:\text{ expected update term}}+ \underbrace{\alpha_t\langle Q_t-Q^{\pi_t}, F_\mathrm{ETD}(Q_t, U_t, \pi_{t+1}) - \bar{F}(Q_t, \pi_t) \rangle_\nu}_{T_2:\text{ Markovian noise term}} \nonumber\\
        & + \underbrace{\langle Q_t-Q^{\pi_t}, Q^{\pi_t}-Q^{\pi_{t+1}} \rangle_\nu}_{T_3:\text{ time-varying target term}}+ \underbrace{\frac{1}{2}\|(Q_{t+1}-Q_t) + (Q^{\pi_t}-Q^{\pi_{t+1}})\|_\nu^2}_{T_4:\text{ residuals}}. \label{critic_decomposition_ETD}
    \end{align}
    We now bound the terms $T_1, T_2, T_3$ and $T_4$. 
    
    For the term $T_1$, we have, by the fact that $Q^{\pi_t}$ is the fixed point of $\bar{F}(\cdot, \pi_t)$, that
\begin{align}
    T_1
    =\,&\alpha_t\langle Q_t-Q^{\pi_t}, \bar{F}(Q_t, \pi_t) - Q_t \rangle_\nu\nonumber\\
    =\,&\alpha_t\langle Q_t - Q^{\pi_t}, \bar{F}(Q_t, \pi_t) - \bar{F}(Q^{\pi_t}, \pi_t) \rangle_\nu - \alpha_t\|Q^{\pi_t} - Q_t\|_\nu^2\nonumber\\
    \leq\,&\alpha_t\|Q_t - Q^{\pi_t}\|_\nu \|\bar{F}(Q_t, \pi_t) - \bar{F}(Q^{\pi_t}, \pi_t)\|_\nu - \alpha_t\|Q^{\pi_t} - Q_t\|_\nu^2 \nonumber\\
    \leq\,&-\alpha_t(1 - \gamma_c)\|Q^{\pi_t} - Q_t\|_\nu^2\nonumber\\
    =\,&-2\alpha_t(1 - \gamma_c)W_t.\label{eq:T1_bound_ETD}
\end{align}

    For the terms $T_2$, $T_3$, and $T_4$, they are bounded in the following sequence of lemmas.

\begin{lemma}\label{lem:noise_ETD}
        The following inequality holds for all $t\geq K$:
        \begin{align*}
    \mathbb{E}T_2\leq \frac{5\alpha_t\alpha_{t-z_t,t-1}}{(1-\gamma)^2}.
\end{align*}
\end{lemma}
The proof of the above is presented in Appendix \ref{app:terms_ETD}.

\begin{lemma}\label{lem:target_ETD}
     The following inequality holds for all $t\geq 0$:
    \begin{align*}T_3 \leq \frac{8\omega_t}{(1-\gamma)^3\sqrt{\nu_{\min}}}W_t+\frac{1-\gamma}{4}\omega_t\|Q^*-Q^{\pi_t}\|_\infty^2+\frac{1-\gamma}{4}\omega_t\chi_t^2.
\end{align*}
\end{lemma}

The above lemma and its proof are identical to Lemma~\ref{lem:target}, which is proved in Appendix~\ref{app:target}. Hence, we omit the proof here.

\begin{lemma}\label{lem:residual_ETD}
        The following inequality holds for all $t\geq 0$:
\begin{align*}
    T_4\leq& \frac{2\alpha_t^2}{(1-\gamma)^2}.
\end{align*}
\end{lemma}
The proof of the above is presented in Appendix \ref{app:terms_ETD}.

We take expectation in \eqref{critic_decomposition_ETD} and substitute the bounds on $T_1,\mathbb{E}T_2,T_3,T_4$ to obtain
\begin{align*}
    \mathbb{E}W_{t+1} \leq& \mathbb{E}W_t - 2\alpha_t(1-\gamma_c)\mathbb{E}W_t + \frac{5\alpha_t\alpha_{t-z_t,t-1}}{(1-\gamma)^2}\\
    &+\frac{8\omega_t}{(1-\gamma)^3\sqrt{\nu_{\min}}}\mathbb{E}W_t+\frac{1-\gamma}{4}\omega_t\mathbb{E}\infnorm{Q^*-Q^{\pi_t}}^2+\frac{1-\gamma}{4}\omega_t\mathbb{E}\chi_t^2\\
    &+ \frac{2\alpha_t\alpha_{t-z_t,t-1}}{(1-\gamma)^2}\\
    =& \left(1-2\alpha_t(1-\gamma_c)+\frac{8\omega_t}{(1-\gamma)^3\sqrt{\nu_{\min}}}\right)\mathbb{E}W_t \\&+\frac{1-\gamma}{4}\omega_t\mathbb{E}\infnorm{Q^*-Q^{\pi_t}}^2+\frac{1-\gamma}{4}\omega_t\mathbb{E}\chi_t^2+ \frac{7\alpha_t\alpha_{t-z_t,t-1}}{(1-\gamma)^2}.
\end{align*}
The result follows by bounding the perturbation term (in the coefficient of $\mathbb{E}W_t$) in the above using the stepsize condition $\omega_t\leq (1-\gamma)^3(1-\gamma_c)\nu_{\min}\alpha_t/16$.
\end{proof}

\subsection{Combining the Actor and the Critic}\label{app:convergence_ETD}
We begin with a coupled drift inequality for the combined actor and critic error.

\begin{lemma}\label{big:coupled_ETD}
    Under the same conditions as Proposition \ref{prop:critic_ETD}, the following inequality holds for all $t\geq K$:
    \begin{align*}
    \mathbb{E}\left[\infnorm{Q^*-Q^{\pi_{t+1}}}^2 + W_{t+1}\right] \leq& \left(1-\frac{\omega_t(1-\gamma)}{2}\right)\mathbb{E}\left[\infnorm{Q^*-Q^{\pi_{t}}}^2 + W_{t}\right] + \frac{6\omega_t}{(1-\gamma)^3}\mathbb{E}\chi_t^2\\
    &+ \frac{7\alpha_t\alpha_{t-z_t,t-1}}{(1-\gamma)^2}.
    \end{align*}
\end{lemma}
\begin{proof}[Proof of Lemma \ref{big:coupled_ETD}]
We add the critic drift (Proposition \ref{prop:critic_ETD}) to the squared actor
drift inequality from Corollary \ref{cor:actor_ETD} (after taking expectation) to obtain:

\begin{align*}
&\mathbb{E}[\|Q^*-Q^{\pi_{t+1}}\|_\infty^2 + W_{t+1}]\\
&\le
\big(1-\omega_t(1-\gamma)\big)
\mathbb{E}\|Q^*-Q^{\pi_t}\|_\infty^2
+
\frac{6\omega_t}{(1-\gamma)^3}
\mathbb{E}\|Q_t-Q^{\pi_t}\|_\infty^2
+
\frac{5\omega_t}{(1-\gamma)^3}
\mathbb{E}\chi_t^2
\\
&\quad
+
\big(1-\alpha_t(1-\gamma_c)\big)
\mathbb{E}W_t
+
\frac{\omega_t(1-\gamma)}{2}
\mathbb{E}\|Q^*-Q^{\pi_t}\|_\infty^2
+
\frac{\omega_t(1-\gamma)}{2}
\mathbb{E}\chi_t^2+\frac{7\alpha_t\alpha_{t-z_t,t-1}}
{(1-\gamma)^2}\\
&\leq \left(1-\frac{\omega_t(1-\gamma)}{2}\right)\mathbb{E}\infnorm{Q^*-Q^{\pi_t}}^2 + \frac{6\omega_t}{(1-\gamma)^3}\mathbb{E}\chi_t^2\\
&\quad + \left(1-\alpha_t(1-\gamma_c)+\frac{12\omega_t}{(1-\gamma)^3\nu_{\min}}\right)\mathbb{E}W_t + \frac{7\alpha_t\alpha_{t-z_t,t-1}}
{(1-\gamma)^2}\\
&\leq \left(1-\frac{\omega_t(1-\gamma)}{2}\right)\mathbb{E}\infnorm{Q^*-Q^{\pi_t}}^2 + \left(1-\frac{\alpha_t(1-\gamma_c)}{4}\right)\mathbb{E}W_t\tag{$\omega_t\leq (1-\gamma)^3(1-\gamma_c)\nu_{\min}\alpha_t/16$}+\frac{6\omega_t}{(1-\gamma)^3}\mathbb{E}\chi_t^2+\frac{7\alpha_t\alpha_{t-z_t,t-1}}
{(1-\gamma)^2}.
\end{align*}

Due to the stepsize condition $\omega_t\leq (1-\gamma)^3(1-\gamma_c)\nu_{\min}\alpha_t/16$, the coefficient of $\mathbb{E}\infnorm{Q^*-Q^{\pi_t}}^2$ in the above is larger than that of $\mathbb{E}W_t$. Using ($1-\omega_t(1-\gamma)/2$) as the common coefficient for both $\mathbb{E}W_t$ and $\mathbb{E}\infnorm{Q^*-Q^{\pi_t}}^2$ gives the result.
\end{proof}

\subsubsection{Solving the Recursion}\label{app:main_ETD}

To prove Theorem~\ref{thm:main} for the ETD-based critic, we first verify the conditions for Lemma~\ref{big:coupled_ETD} by showing the existence of the constant threshold $\tilde C_r$. By choosing $\alpha,\omega$ such that 
\begin{align*}
C_r\leq \tilde{C}_r := \frac{(1-\gamma)^3(1-\gamma_c)\nu_{\min}}{16},
\end{align*}
the condition $\omega_t\leq(1-\gamma)^3(1-\gamma_c)\nu_{\min}\alpha_t/16$ is satisfied for any of the stepsize sequences considered in Theorem \ref{thm:main}. Thus, the stepsize conditions of Lemma \ref{big:coupled_ETD}, the coupled drift lemma, hold for all $t$. Moreover, note that since $\nu_{\min}\geq(1-\gamma)\pi_{b,\min}\mu_{b,\min}/(mn)$ (cf. Lemma \ref{lem:F_ETD}), the threshold $\tilde{C}_r$ is a constant depending only on $(n,m,p,\gamma,\pi_b)$, as claimed in Theorem \ref{thm:main}.

We now simplify the coupled drift from Lemma \ref{big:coupled_ETD}. It follows directly from \cite[Lemma 5.1]{chen2025approximate} that under Condition \ref{cond:temp} with parameter $\tau\ge 0$, we have that $\chi_t\le \tau/(t+h)^{\eta/2}=\tau\sqrt{\omega_t/\omega}$ for all $t$. Now note that for constant stepsizes, $\alpha_{t-z_t,t-1}=\alpha z_t\leq2\alpha_tz_t$. For diminishing stepsizes, let $h$ be large enough so that $z_t\leq (t+h)/2$ for all $t\geq K$ (possible since $z_t$ is logarithmic in $t$). This implies that
\begin{align*}
    \alpha_{t-z_t,t-1} \leq \frac{\alpha z_t}{(t-z_t+h)^\eta}
    \leq\frac{\alpha z_t}{\big((t+h)/2\big)^\eta}
    \leq 2\alpha_tz_t.
\end{align*}
The coupled drift from Lemma \ref{big:coupled_ETD} hence reduces to
\begin{align*}
\mathbb{E}\left[\|Q^*-Q^{\pi_{t+1}}\|_\infty^2 + W_{t+1}\right]
&\le
\left(1-\frac{\omega_t(1-\gamma)}{2}\right)
\mathbb{E}\left[\|Q^*-Q^{\pi_t}\|_\infty^2 + W_t\right]
+
\frac{6\tau^2\omega_t^2}{(1-\gamma)^3\omega}
+
\frac{14\alpha_t^2z_t}
{(1-\gamma)^2},
\end{align*}

for all $t\geq K$. Applying the above repeatedly, we have for all $T\geq K$, that
\begin{align*}
    \mathbb{E}\left[\infnorm{Q^*-Q^{\pi_T}}^2+W_T\right] \leq& \mathbb{E}\left[\infnorm{Q^*-Q^{\pi_K}}^2+W_K\right]\prod_{t=K}^{T-1}\left(1-\frac{\omega_t(1-\gamma)}{2}\right)\\
    &+ \frac{6\tau^2}{(1-\gamma)^3\omega}\sum_{t=K}^{T-1}\omega_t^2\prod_{u=t+1}^{T-1}\left(1-\frac{\omega_u(1-\gamma)}{2}\right) \\
    &+ \frac{14}{(1-\gamma)^2}\sum_{t=K}^{T-1}\alpha_t^2z_t\prod_{u=t+1}^{T-1}\left(1-\frac{\omega_u(1-\gamma)}{2}\right).
\end{align*}
We define notation $\Pi_{i,j}:=\prod_{u=i}^j\big(1-\omega_u(1-\gamma)/2\big)$. Since $z_t\leq z_T$ for all $t\leq T$, it follows that
\begin{align}
\mathbb{E}\left[\|Q^*-Q^{\pi_{T}}\|_\infty^2 + W_{T}\right]
\le&
\mathbb{E}\left[\infnorm{Q^*-Q^{\pi_K}}^2+W_K\right]\Pi_{K,T-1}  + \frac{7z_{T}}{(1-\gamma)^2}\sum_{t=K}^{T-1}\alpha_t^2\Pi_{t+1,T-1}\nonumber\\
&+ {\frac{6\tau^2}{(1-\gamma)^3\omega}\sum_{t=K}^{T-1}\omega_t^2\Pi_{t+1,T-1}}\nonumber\\
\leq&\frac{3\Pi_{K,T-1}}{2(1-\gamma)^2}  + \frac{14z_{T}}{(1-\gamma)^2}\sum_{t=K}^{T-1}\alpha_t^2\Pi_{t+1,T-1}+{\frac{6\tau^2}{(1-\gamma)^3\omega}\sum_{t=K}^{T-1}\omega_t^2\Pi_{t+1,T-1}}\label{eq:little_master_ETD},
\end{align}
where the last line follows since
\begin{align*}
    \infnorm{Q^*-Q^{\pi_K}}^2+W_K &= \infnorm{Q^*-Q^{\pi_K}}^2+\frac{1}{2}\|Q_K-Q^{\pi_K}\|_\nu^2\\
    &\leq \frac{1}{(1-\gamma)^2}+\frac{1}{2(1-\gamma)^2}\\
    &=\frac{3}{2(1-\gamma)^2},
\end{align*}
which in turn, holds since $Q_t(s,a)\in[0,1/(1-\gamma)]$ for all $(s,a)$ (cf. Lemma \ref{bounds_ETD} (2)).

Now note that $\omega_t=C_r\alpha_t$ and recall the notation $M_\mathrm{ETD} = (1-\gamma)^{-2}$. Substituting in \eqref{eq:little_master_ETD}, we have
\begin{equation}
    \mathbb{E}\left[\infnorm{Q^*-Q^{\pi_T}}^2+W_T\right] \leq E_T:= \frac{3\Pi_{K,T-1}}{2(1-\gamma)^2} + \left(\frac{6\tau^2}{(1-\gamma)^3\omega}+\frac{14M_\mathrm{ETD}z_T}{C_r^2}\right)\sum_{t=K}^{T-1}\omega_t^2\Pi_{t+1,T-1}.\label{eq:master_ETD}
\end{equation}

The remainder of the proof follows by bounding the term $E_T$ for the specific stepsizes.

\paragraph{(1) Constant stepsizes.}
When $\alpha_t = \alpha$ and $\omega_t = \omega$ for all $t$, we have,
\begin{align*}
    E_T &\leq \frac{3}{2(1-\gamma)^2}\left(1-\frac{\omega(1-\gamma)}{2}\right)^{T-K}+\left(\frac{6\tau^2}{(1-\gamma)^3\omega}+\frac{7M_\mathrm{ETD}z_\omega}{C_r^2}\right)\sum_{t=K}^{T-1}\omega^2\left(1-\frac{\omega(1-\gamma)}{2}\right)^{T-1-t}\\
    &\leq \frac{3}{2(1-\gamma)^2}\left(1-\frac{\omega(1-\gamma)}{2}\right)^{T-K}+\frac{12\tau^2}{(1-\gamma)^4}+\frac{28M_\mathrm{ETD}\omega z_\omega}{(1-\gamma)C_r^2},
\end{align*}
where the last step follows from the geometric sum $\sum_{t=0}^\infty\big(1-\omega(1-\gamma)/2\big)^t = {2}/{\big(\omega(1-\gamma)\big)}$. Substituting the above in \eqref{eq:master_ETD} proves Theorem \ref{thm:main} (1).

\paragraph{(2) Harmonic stepsizes.} When $\alpha_t = \alpha/(t+h)$ and $\omega_t=\omega/(t+h)$, we have by \cite[Lemma A.7]{chen2024lyapunov} that
\begin{align*}
    E_T \leq& \frac{3}{2(1-\gamma)^2}\left(\frac{K+h}{T+h}\right)^{\omega(1-\gamma)/2}\\
    &+ \left(\frac{6\tau^2}{(1-\gamma)^3\omega}+\frac{14M_\mathrm{ETD}z_T}{C_r^2}\right)\times\begin{dcases}
        \frac{8\omega^2}{\big(2-\omega(1-\gamma)\big)(T+h)^{\omega(1-\gamma)/2}} & \omega < \frac{2}{1-\gamma},\\
        \frac{\omega^2\log(T+h)}{T+h} & \omega = \frac{2}{1-\gamma},\\
        \frac{8e\omega^2}{\big(\omega(1-\gamma)-2\big)(T+h)} & \omega > \frac{2}{1-\gamma}.
    \end{dcases}
\end{align*} 
Substituting the above in \eqref{eq:master_ETD} proves Theorem \ref{thm:main} (2).

\paragraph{(3) Polynomial stepsizes.} When $\alpha_t = \alpha/(t+h)^\eta$ and $\omega_t = \omega/(t+h)^\eta$ for $\eta\in (0,1)$, we have by \cite[Lemma A.8]{chen2024lyapunov} that
\begin{align*}
    E_T\leq& \frac{3}{2(1-\gamma)^2}\exp\left[-\frac{\omega(1-\gamma)}{2(1-\eta)}\Big((T+h)^{1-\eta}-(K+h)^{1-\eta}\Big)\right] \\
    &+ \left(\frac{6\tau^2}{(1-\gamma)^3\omega}+\frac{14M_\mathrm{ETD}z_T}{C_r^2}\right)\frac{8\omega}{(T+h)^\eta}.
\end{align*}
Substituting the above in \eqref{eq:master_ETD} proves Theorem \ref{thm:main} (3).

\subsubsection{Proof of Corollary \ref{cor:complexity}}\label{app:complexity_ETD}
Let $\epsilon>0$ be sufficiently small. In order to ensure $\mathbb{E}\infnorm{Q^*-Q^{\pi_t}} < \epsilon$, it suffices to ensure $\mathbb{E}\infnorm{Q^*-Q^{\pi_t}}^2< \epsilon^2$ due to Jensen's inequality. To that end, we consider harmonic stepsizes with $\omega = 3/(1-\gamma) > 2/(1-\gamma)$ and
\begin{align*}
\alpha = \frac{16\omega}{(1-\gamma)^3(1-\gamma_c)\nu_{\min}} = \frac{48}{(1-\gamma)^4(1-\gamma_c)\nu_{\min}},
\end{align*}
which ensures $C_r\leq\tilde{C}_r$. We have by Theorem \ref{thm:main} that for all $T \geq K$,
\begin{align*}
    \mathbb{E}\infnorm{Q^*-Q^{\pi_T}}^2 \leq \frac{3}{(1-\gamma)^2}\left(\frac{K+h}{T+h}\right)^{3/2} + \frac{48e\tau^2\omega}{(1-\gamma)^3(T+h)}+ \underbrace{\frac{9\times 48^2 \times 216 \times 8e}{(1-\gamma)^{10}(1-\gamma_c)^2\nu_{\min}^2}\frac{z_T}{T+h}}_{D_T}.
\end{align*}
Observe also that $D_T$, the third term in the above, is asymptotically larger than the other two terms as $z_T$ is logarithmic in $T$. For having $\mathbb{E}\infnorm{Q^*-Q^{\pi_T}}^2=\mathcal{O}(\epsilon^{2})$, it hence suffices to have $D_T = \mathcal{O}(\epsilon^{2})$. 

We shall now simplify $D_T$. From the definition of $\gamma_c$ in Lemma \ref{lem:F_ETD}, it follows that $1-\gamma_c\geq {(1-\gamma)\pi_{b,\min}\mu_{b,\min}}/{2}$. Lemma \ref{lem:F_ETD} also gives us that $\nu_{\min}\geq (1-\gamma)\pi_{b,\min}\mu_{b,\min}/(mn)$. Further, we have by geometric mixing, that
\begin{align*}
z_T
&\leq
\frac{\log(4/\omega_T)}{\log(1/\sigma_b)}+1  \\
&=
\frac{\log\big(4(1-\gamma)(T+h)/3\big)}{\log(1/\sigma_b)}+1  \\
&=
\mathcal{O}\left(\frac{\log T}{\log(1/\sigma_b)}\right).
\end{align*}
Substituting these bounds in the expression for $D_T$, we have that
\begin{align*}
    D_T = \mathcal{O}\left(\frac{m^2n^2}{(1-\gamma)^{14}\pi_{b,\min}^4\mu_{b,\min}^4\log(1/\sigma_b)}\frac{\log T}{T}\right).
\end{align*}
Denoting the coefficient of $(\log T)/T$ in the above by $\Lambda$, it suffices to have $(\log T)/T = \mathcal{O}\left(\epsilon^2/\Lambda\right)$. This is ensured by
\begin{align*}
    T &= \mathcal{O}\left(\frac{\Lambda}{\epsilon^2}\log\left(\frac{\Lambda}{\epsilon^2}\right)\right) \\
    &= \mathcal{O}\left(\Lambda\frac{\log(1/\epsilon)}{\epsilon^2}\right) \\
    &= \mathcal{O}\left(\frac{m^2n^2}{(1-\gamma)^{14}\pi_{b,\min}^4\mu_{b,\min}^4\log(1/\sigma_b)}\frac{\log(1/\epsilon)}{\epsilon^2}\right).
\end{align*}

\subsection{Proofs of All Technical Lemmas}\label{app:terms_ETD}
The following two lemmas will be used extensively. 

\begin{lemma}\label{lem:target_shift_ETD}

For the policy iterates $\{\pi_t\}$ generated by Algorithm \ref{alg}, we have
\begin{align*} 
\|Q^{\pi_t}-Q^{\pi_{t+1}}\|_\infty
\le
\frac{\omega_t}{1-\gamma}
\big(
\|Q^*-Q^{\pi_t}\|_\infty
+
2\|Q_t-Q^{\pi_t}\|_\infty
+
\chi_t
\big),\quad\forall\,t\geq 0.
\end{align*}
\end{lemma}

The statement and the proof of the above are identical to Lemma \ref{lem:target_shift}, which is proved in Appendix \ref{ap:pf:lem:target_shift}. We hence omit the proof here.

\begin{lemma}\label{bounds_ETD}
The following hold.
\begin{enumerate}[(1)]\label{lem:running_bounds_ETD}
    \item For any $Q\in\mathbb{R}^{mn}$, $y\in\mathcal{Y}, u\in\mathcal{U}$, and $\pi \in {(\Delta^{m})}^n$, we have
    \begin{align*}
        \infnorm{F_{\mathrm{ETD}}(Q, u, \pi)} \leq 1 + \gamma\|Q\|_\infty,\quad
        \|\bar{F}(Q, \pi)\|_\infty \leq 1 + \gamma\|Q\|_\infty.
    \end{align*}

    \item For all $t\geq 0$, $Q_t(s,a)\in[0,1/(1-\gamma)]$ for all $(s,a)$.

    \item For all $t_1,t_2\geq 0$ (assuming $t_1<t_2$), we have $\infnorm{Q_{t_1}-Q_{t_2}}\leq \frac{\alpha_{t_1,t_2-1}}{1-\gamma}$.
    
    \item $\max_{s\in\mathcal{S}}\mathrm{d_{TV}}\big(\pi_{t_1}(s),\pi_{t_2}(s)\big) \leq \omega_{t_1,t_2-1}$.
    
    \item For any $t_1,t_2\geq0$ (assuming $t_1<t_2$), we have $\infnorm{Q^{\pi_{t_1}}-Q^{\pi_{t_2}}} \leq \frac{2\omega_{t_1,t_2-1}}{(1-\gamma)^2}$.

\end{enumerate}
\end{lemma}

The proof of Lemma \ref{bounds_ETD} is presented in Appendix \ref{app:iterates_ETD}.

\subsubsection{Proof of Lemma \ref{lem:noise_ETD}}\label{app:noise_ETD}

For any $t\geq K$, we have
\begin{align}
    &\left\langle Q_t - Q^{\pi_t}, F_\mathrm{ETD}(Q_t, U_t, \pi_{t+1}) - \bar{F}(Q_t, \pi_t) \right\rangle_\nu \nonumber\\
    =& \underbrace{\left\langle Q_{t-z_t} - Q^{\pi_{t-z_t}}, F_\mathrm{ETD}(Q_{t-z_t}, U_t, \pi_{t-z_t}) - \bar{F}(Q_{t-z_t}, \pi_{t-z_t}) \right\rangle_\nu}_{T_{21}}\nonumber\\
    &+ \underbrace{\left\langle Q_{t} - Q^{\pi_{t}}, F_\mathrm{ETD}(Q_t, U_t, \pi_{t+1}) - F_\mathrm{ETD}(Q_{t-z_t}, U_t, \pi_{t+1}) \right\rangle_\nu}_{T_{22}}\nonumber\\
    &+\underbrace{\left\langle Q_{t} - Q^{\pi_{t}}, \bar{F}(Q_{t-z_t}, \pi_{t-z_t}) - \bar{F}(Q_{t}, \pi_{t-z_t}) \right\rangle_\nu}_{T_{23}}\nonumber\\
    &+ \underbrace{\left\langle Q_{t} - Q^{\pi_{t}}, F_\mathrm{ETD}(Q_{t-z_t}, U_t ,\pi_{t+1}) - F_\mathrm{ETD}(Q_{t-z_t}, U_t, \pi_{t-z_t}) \right\rangle_\nu}_{T_{24}}\nonumber\\
    &+\underbrace{\left\langle Q_{t} - Q^{\pi_{t}}, \bar{F}(Q_{t}, \pi_{t-z_t}) - \bar{F}(Q_{t}, \pi_t) \right\rangle_\nu}_{T_{25}}\nonumber\\
    &+\underbrace{\left\langle (Q_{t} - Q^\pi_{t})-(Q_{t-z_t} - Q^{\pi_{t-z_t}}), F_\mathrm{ETD}(Q_{t-z_t}, U_t, \pi_{t+1}) - \bar{F}(Q_{t-z_t}, \pi_{t-z_t})
    \right\rangle_\nu}_{T_{26}}.\label{eq:T2_decomposition_ETD}
\end{align}

We next bound each term on the right-hand side of the previous inequality.

\paragraph{The term $T_{21}$.}
    Let $\mathcal{F}_t$ be the sigma-algebra generated by $\{(S_i,A_i)\}_{0\leq i\leq t}$. 
Since $Q_{t-z_t}$ and $\pi_{t-z_t}$ are $\mathcal{F}_{t-z_t}$-measurable, we apply the tower property to obtain
\begin{align}
\mathbb{E}T_{21}
=\,&
\mathbb{E}\left[\langle
Q_{t-z_t}-Q^{\pi_{t-z_t}},
F_{\mathrm{ETD}}(Q_{t-z_t}, U_t, \pi_{t-z_t})-\bar{F}(Q_{t-z_t}, \pi_{t-z_t})
\rangle_\nu\right] \nonumber\\
=\,&
\mathbb{E}\left[\langle
Q_{t-z_t}-Q^{\pi_{t-z_t}},
\mathbb{E}[F_\mathrm{ETD}(Q_{t-z_t}, U_t, \pi_{t-z_t})|\mathcal{F}_{t-z_t}]
-\bar{F}(Q_{t-z_t}, \pi_{t-z_t})
\rangle_\nu\right]\nonumber\\
\leq \,&\mathbb{E}\left[\|Q_{t-z_t}-Q^{\pi_{t-z_t}}\|_\nu
\|\mathbb{E}[F_\mathrm{ETD}(Q_{t-z_t}, U_t, \pi_{t-z_t})|\mathcal{F}_{t-z_t}]
-\bar{F}(Q_{t-z_t}, \pi_{t-z_t})\|_\nu\right]\label{eq:T21_decomposition_ETD}
\end{align}
Since $Q_{t-z_t}(s,a)\in[0,1/(1-\gamma)]$ for all $(s,a)$ (cf. Lemma \ref{bounds_ETD}, (2)), we have that $\|Q_{t-z_t}-Q^{\pi_{t-z_t}}\|_\nu \leq 1/(1-\gamma)$. 

As for the term $\|\mathbb{E}[F_{\mathrm{IS}}(Q_{t-z_t}, Y_t, \pi_{t-z_t})|\mathcal{F}_{t-z_t}]
-\bar F_{\mathrm{IS}}(Q_{t-z_t}, \pi_{t-z_t})\|_\nu$ on the right-hand side of \eqref{eq:T21_decomposition_ETD}, we have
\begin{align*}
    &\|\mathbb{E}[F_{\mathrm{ETD}}(Q_{t-z_t}, Y_t, \pi_{t-z_t})|\mathcal{F}_{t-z_t}]
-\bar{F}(Q_{t-z_t}, \pi_{t-z_t})\|_\nu\\
\leq \,&\max_{s\in\mathcal{S}}\sum_{y=(s_1,a_1,s_2,a_2)}\left|P^{z_t-1}(s,s_1)-\mu_b(s_1)\right|\pi_b(a_1|s_1)p(s_2\mid s_1,a_1)\pi_b(a_2|s_2)\|F_{\mathrm{ETD}}(Q_{t-z_t},y,\pi_{t-z_t})\|_\infty\\
\leq \,&\big(1+\gamma\|Q_{t-z_t}\|_\infty\big)\max_{s\in\mathcal{S}}\sum_{s_1\in\mathcal{S}}\left|P^{z_t-1}(s,s_1)-\mu_b(s_1)\right|\tag{Lemma \ref{bounds_ETD} (1)}\\
\leq& \frac{1}{1-\gamma} \max_{s\in\mathcal{S}}\sum_{s_1\in\mathcal{S}}\left|P^{z_t-1}(s,s_1)-\mu_b(s_1)\right|\tag{Lemma \ref{bounds_ETD} (2)}
\\
\leq \,&\frac{\omega_t}{1-\gamma},
\end{align*}
where the last inequality follows from the definition of $z_t$, where we recall that $z_t = \min\{k\geq 1\mid \max_{s\in\mathcal{S}}\mathrm{d_{TV}}\big(P_{\pi_b}^{k-1}(s,\cdot), \mu_b(\cdot)\big)\leq\omega_t/2\}$, which exists and is finite due to geometric mixing.

Substituting in \eqref{eq:T21_decomposition_ETD}, we have
\begin{align}
    \mathbb{E}T_{21}\leq \,&\frac{\omega_t}{(1-\gamma)^2}.\label{eq:T21_bound_ETD}
\end{align}

\paragraph{The term $T_{22}+T_{23}$.}
By the Cauchy-Schwarz inequality, we have
\begin{align}
T_{22}+T_{23}
\leq \,&
\|Q_t-Q^{\pi_t}\|_\nu
\|\bar{F}(Q_{t-z_t}, \pi_{t-z_t})-\bar{F}(Q_t, \pi_{t-z_t})\|_\nu
\nonumber\\
&+\|Q_t-Q^{\pi_t}\|_\nu
\|F_{\mathrm{ETD}}(Q_t, U_t, \pi_{t+1})-F_{\mathrm{ETD}}(Q_{t-z_t}, U_t, \pi_{t+1})\|_\nu\label{eq:T22+T23_ETD}
\end{align}
Since $\bar{F}(\cdot,\pi_{t-z_t})$ is a contraction mapping with respect to $\|\cdot\|_\infty$ (cf. Lemma \ref{lem:F_ETD} (1)), we have
\begin{align*}
    \|\bar{F}(Q_{t-z_t}, \pi_{t-z_t})-\bar{F}(Q_t, \pi_{t-z_t})\|_\nu &\leq \|\bar{F}(Q_{t-z_t}, \pi_{t-z_t})-\bar{F}(Q_t, \pi_{t-z_t})\|_\infty\\
    &\leq \|Q_t-Q_{t-z_t}\|_\infty\\
    &\leq \frac{\alpha_{t-z_t,t-1}}{1-\gamma} \tag{Lemma \ref{bounds_ETD}, (3)}.
\end{align*}
Since $F_{\mathrm{ETD}}(\cdot,y,\pi)$ is $1$-Lipschitz continuous with respect to $\|\cdot\|_\infty$ (cf. Lemma \ref{lem:F_ETD}), we similarly have 
\begin{align*}
    \|F_{\mathrm{ETD}}(Q_t, U_t, \pi_{t+1})-F_{\mathrm{ETD}}(Q_{t-z_t}, U_t, \pi_{t+1})\|_\nu\leq \frac{\alpha_{t-z_t,t-1}}{1-\gamma}.
\end{align*}
Substituting the previous two inequalities in \eqref{eq:T22+T23_ETD}, we obtain
\begin{align}
    T_{22}+T_{23}\leq \,&\frac{2\alpha_{t-z_t,t-1}}{1-\gamma}\|Q_t-Q^{\pi_t}\|_\nu \nonumber\\
\leq \,&\frac{2\alpha_{t-z_t,t-1}}{(1-\gamma)^2},\label{eq:T2223_bound_ETD}
\end{align}
where the last inequality follows since $Q^{\pi_t}(s,a)\in[0,1/(1-\gamma)]$ for all $(s,a)$ (cf. Lemma \ref{bounds_ETD}, (2)).

\paragraph{The term $T_{24}+T_{25}$.}
By the Cauchy-Schwarz inequality, we have
\begin{align}
    T_{24}+T_{25}\leq \,&\|Q_t-Q^{\pi_t}\|_\nu
\|\bar{F}(Q_t, \pi_{t-z_t}) - \bar{F}(Q_t, \pi_t)\|_\nu\nonumber\\
&+\|Q_t-Q^{\pi_t}\|_\nu\|F_\mathrm{ETD}(Q_{t-z_t}, U_t, \pi_{t+1}) - F_\mathrm{ETD}(Q_{t-z_t}, U_t, \pi_{t-z_t}) \|_\nu\label{eq:T2425:decomposition_ETD}
\end{align}
For the term $\|\bar{F}(Q_t, \pi_{t-z_t}) - \bar{F}(Q_t, \pi_t)\|_\nu$ on the right-hand side of \eqref{eq:T2425:decomposition_ETD}, we have by Lemma \ref{lem:F_ETD} (3), that
\begin{align*}
    \|\bar{F}(Q_t, \pi_{t-z_t}) - \bar{F}(Q_t, \pi_t)\|_\nu\leq\,&\|\bar{F}(Q_t, \pi_{t-z_t}) - \bar{F}(Q_t, \pi_t)\|_\infty\\
    \leq \,&2\omega_{t-z_t,t-1}\|Q_t\|_\infty\\
    \leq \,&\frac{2\omega_{t-z_t,t-1}}{1-\gamma},
\end{align*}
where the last inequality follows since $\infnorm{Q_t}\leq1/(1-\gamma)$ (cf. Lemma \ref{bounds_ETD}, (2)).

For the term $\|F_\mathrm{ETD}(Q_{t-z_t}, U_t, \pi_{t+1}) - F_\mathrm{ETD}(Q_{t-z_t}, U_t, \pi_{t-z_t}) \|_\nu$ on the right-hand side of \eqref{eq:T2425:decomposition_ETD}, we similarly have by Lemma \ref{lem:F_ETD} (4), that
\begin{align*}
    \|F_\mathrm{ETD}(Q_{t-z_t}, U_t, \pi_{t+1}) - F_\mathrm{ETD}(Q_{t-z_t}, U_t, \pi_{t-z_t}) \|_\nu\leq \frac{2\omega_{t-z_t,t}}{1-\gamma}.
\end{align*}

Substituting the previous two inequalities in \eqref{eq:T2425:decomposition_ETD}, we obtain
\begin{align}
    T_{24}+T_{25}\leq \,&\frac{2\omega_{t-z_t, t}}{1-\gamma}\|Q_t-Q^{\pi_t}\|_\nu\nonumber\\
    \leq & \frac{2\omega_{t-z_t,t}}{(1-\gamma)^2},\label{eq:T2425_bound_ETD}
\end{align}
where the last line follows since $Q_t(s,a)\in[0,1/(1-\gamma)]$ for all $(s,a)$ (cf. Lemma \ref{bounds_ETD} (2)).

\paragraph{The term $T_{26}$.}
By the Cauchy–Schwarz inequality, we have
\begin{align}\label{eq:T25_breakup_ETD}
    T_{26}
\le\,&
\left\|
(Q_t - Q^{\pi_t}) - (Q_{t-z_t}-Q^{\pi_{t-z_t}})
\right\|_\nu
\;
\left\|
F_\mathrm{ETD}(Q_{t-z_t}, Y_t, \pi_{t+1}) - \bar{F}(Q_{t-z_t}, \pi_{t-z_t})
\right\|_\nu\nonumber\\
\leq \,&\big(\|
Q_t - Q_{t-z_t}\|_\nu+\|Q^{\pi_t}-Q^{\pi_{t-z_t}})
\|_\nu\big)
\;
\left\|F_\mathrm{ETD}(Q_{t-z_t}, Y_t, \pi_{t+1}) - \bar{F}(Q_{t-z_t}, \pi_{t-z_t})\right\|_\nu
\end{align}
For the first term on the right-hand side of \eqref{eq:T25_breakup_ETD}, we have by Lemma \ref{bounds_ETD} (3) and (5) that
\begin{align}
\|Q_t-Q_{t-z_t}\|_\nu + \|Q^{\pi_t}-Q^{\pi_{t-z_t}}\|_\nu 
\le\,&\|Q_t-Q_{t-z_t}\|_\infty + \|Q^{\pi_t}-Q^{\pi_{t-z_t}}\|_\infty\nonumber\\
\leq &\frac{\alpha_{t-z_t,t-1}}{1-\gamma}
+
\frac{2\omega_{t-z_t,t-1}}{(1-\gamma)^2}.\label{eq:T25_first_factor_ETD}
\end{align}
For the second term on the right-hand side of \eqref{eq:T25_breakup_ETD}, we have by Lemma \ref{bounds_ETD} (1), that
\begin{align}
&\|F_\mathrm{ETD}(Q_{t-z_t}, U_t, \pi_{t+1}) - \bar{F}(Q_{t-z_t}, \pi_{t-z_t})\|_\nu \\&\leq \|F_\mathrm{ETD}(Q_{t-z_t}, U_t, \pi_{t+1})\|_\infty+\|\bar{F}(Q_{t-z_t}, \pi_{t-z_t})\|_\infty \nonumber\\
&\le 2 + 2\gamma\infnorm{Q_{t-z_t}}\nonumber\\
&\leq \frac{2}{1-\gamma}, \label{T25_second_factor_ETD}
\end{align}
where the last line follows since $\infnorm{Q_{t-z_t}}\leq 1/(1-\gamma)$ (cf. Lemma \ref{bounds_ETD} (2)).

Substituting the bounds \eqref{eq:T25_first_factor_ETD} and \eqref{T25_second_factor_ETD} in \eqref{eq:T25_breakup_ETD}, we have
\begin{align}
    T_{26}\leq \,&\frac{2\alpha_{t-z_t,t-1}}{(1-\gamma)^2}+\frac{4\omega_{t-z_t,t-1}}{(1-\gamma)^3}.\label{eq:T26_bound_ETD}
\end{align}

\paragraph{Combining everything together.} Substituting \eqref{eq:T21_bound_ETD}, \eqref{eq:T2223_bound_ETD}, \eqref{eq:T2425_bound_ETD} and \eqref{eq:T26_bound_ETD} in \eqref{eq:T2_decomposition_ETD}, we obtain
\begin{align*}
    \mathbb{E}\left\langle Q_t - Q^{\pi_t}, F_\mathrm{ETD}(Q_t, Y_t, \pi_{t+1}) - \bar{F}(Q_t, \pi_t) \right\rangle\leq \,&\frac{\omega_t}{(1-\gamma)^2} + \frac{2\alpha_{t-z_t,t-1}}{(1-\gamma)^2} + \frac{2\omega_{t-z_t,t}}{(1-\gamma)^2}\\
&+\frac{2\alpha_{t-z_t,t-1}}{(1-\gamma)^2}+\frac{2\omega_{t-z_t,t-1}}{(1-\gamma)^3}\\
&\leq \frac{5\alpha_{t-z_t,t-1}}{(1-\gamma)^2},
\end{align*}
where the last inequality follows from the stepsize condition $\omega_t\leq (1-\gamma)^3(1-\gamma_c)\nu_{\min}\alpha_t/16$. It now follows that
\begin{align*}
    \mathbb{E}T_2\leq \frac{5\alpha_t\alpha_{t-z_t,t-1}}{(1-\gamma)^2}.
\end{align*}

\subsubsection{Proof of Lemma \ref{lem:residual_ETD}}\label{app:res_ETD}

Observe that
\begin{align*}
T_4
&=
\frac{1}{2}
\|(Q_{t+1}-Q_t)+(Q^{\pi_t}-Q^{\pi_{t+1}})\|_\nu^2\\
&\le
\|Q_{t+1}-Q_t\|_\infty^2
+
\|Q^{\pi_t}-Q^{\pi_{t+1}}\|_\infty^2\\
&\leq \frac{\alpha_t^2}{(1-\gamma)^2} + \frac{4\omega_t^2}{(1-\gamma)^4} \tag{Lemma \ref{bounds_ETD} (3) and (5)}\\
&\leq \frac{2\alpha_t^2}{(1-\gamma)^2} \tag{$\omega_t\leq (1-\gamma)^3(1-\gamma_c)\nu_{\min}\alpha_t/16$}.
\end{align*}

\subsubsection{Proof of Lemma \ref{lem:running_bounds_ETD}}
\label{app:iterates_ETD}

\begin{enumerate}[(1)]
    \item 
By definition of the operator $F_\mathrm{ETD}$, for any $Q \in \mathbb{R}^{mn}$, $u = (s_1,a_1,s_2) \in \mathcal{U}$, and $\pi$, we have for any $(s,a)$ that
\begin{align*}
    \left|[F_{\mathrm{ETD}}(Q, y, \pi)](s,a)\right|\leq \,&(1-\mathds{1}_{\{(s,a)= (s_1,a_1)\}})|Q(s,a)|\\
    &+\mathds{1}_{\{(s,a)=(s_1,a_1)\}}\left|\mathcal{R}(s_1,a_1) + \gamma\sum_{a\in\mathcal{A}}\pi(a_\mid s_2)Q(s_2,a)\right|\\
    \leq \,&(1-\mathds{1}_{\{(s,a)= (s_1,a_1)\}})\|Q\|_\infty+\mathds{1}_{\{(s,a)=(s_1,a_1)\}}\left(1+ \gamma\|Q\|_\infty\right)\\
    \leq \,&1+ \gamma\|Q\|_\infty.
\end{align*}
It follows that
\begin{align*}
    \|F_{\mathrm{ETD}}(Q, u, \pi)\|_\infty \leq 1+ \|Q\|_\infty.
\end{align*}

Since $\bar{F}(Q,\pi)=\mathbb{E}_{U\sim\mu_U}F_\mathrm{ETD}(Q,U,\pi)$, it follows immediately from the above that 
\begin{align*}\|\bar{F}(Q,\pi)\|_\infty \leq 1+ \|Q\|_\infty.\end{align*}

\item We proceed via induction. Clearly $Q_{0}(s,a)=0\in[0,1/(1-\gamma)]$ for all $(s,a)$. If $Q_{t}(s,a)\in[0,1/(1-\gamma)]$ for all $(s,a)$ for some $t\geq 0$, then by the ETD-based critic update in Algorithm \ref{alg}, Lines 8 \& 10, we have that $Q_{t+1}(s,a) = Q_t(s,a)\in[0,1/(1-\gamma)]$ for all $(s,a)\neq (S_t, A_t)$ and 
\begin{align*}
    Q_{t+1}(S_t,A_t) &= (1-\alpha_t)Q_{t+1}(S_t,A_t) + \alpha_t\left(\mathcal{R}(S_t,A_t) + \gamma\sum_{a\in\mathcal{A}}\pi_{t+1}(a\mid S_{t+1})Q_t(S_{t+1},a)\right)\\
    &\leq \frac{1-\alpha_t}{1-\gamma}+\alpha_t\left(1+\frac{\gamma}{1-\gamma}\right)\\
    &=\frac{1}{1-\gamma}.
\end{align*}
Moreover, since $\mathcal{R}(S_t,A_t), \pi_{t+1}(a\mid S_{t+1}), Q(s, a), \alpha_t \geq 0$ for all $(s,a)$, it follows that $Q_{t+1}(S_t,A_t) \geq 0$. This concludes the induction. 

\item Note that for any $t\geq 0$, we have 
\begin{align*}
    \infnorm{Q_{t+1}-Q_t} = \alpha_t\left|\mathcal{R}(S_t,A_t)+\gamma\sum_{a\in\mathcal{A}}\pi(a\mid S_{t+1})Q_t(S_{t+1},a)-Q_t(S_t,A_t)\right|\leq \frac{\alpha_t}{1-\gamma},
\end{align*}
because $\mathcal{R}(S_t,A_t)\in[0,1]$ and $Q_t(s,a)\in[0,1/(1-\gamma)]$ for all $(s,a)$. Repeatedly applying from $t_1$ to $t_2-1$ yields the result.

\end{enumerate}

Items (4), (5), and their proofs are identical to Lemma \ref{lem:running_bounds} (4), (5), which are proved in Appendix \ref{ap:pf:bounds}. We hence omit the proofs here.

\subsubsection{Proof of Lemma \ref{lem:F_ETD}}\label{app:F_ETD}

\begin{enumerate}[(1)]

    \item The result and its proof are identical to Lemma \ref{lem:F} (1), which is proved in Appendix \ref{app:F_IS}. The proof is hence omitted. 

    \item 
    For any $Q,Q'\in\mathbb{R}^{mn}, u = (s_1,a_1,s_2)\in\mathcal{U}$ and policy $\pi\in{(\Delta^m)}^n$, we have for any $(s,a)$ that
    \begin{align*}
        &|[F_{\mathrm{ETD}}(Q,u,\pi)](s,a)-[F_{\mathrm{ETD}}(Q',u,\pi)](s,a)|\\
        &\leq \,(1-\mathds{1}_{\{(s,a)=(s_1,a_1)\}})|Q(s,a)-Q'(s,a)|\\
        &\quad+\mathds{1}_{\{(s,a)=(s_1,a_1)\}}\sum_{a\in\mathcal{A}}\pi_{t+1}(a\mid s_2)|Q(s_2,a)-Q'(s_2,a)|\\
        &\leq \,\|Q-Q'\|_\infty.
    \end{align*}
    Therefore, we have
    \begin{align*}
        \|F_{\mathrm{ETD}}(Q,u,\pi)-F_{\mathrm{ETD}}(Q',u,\pi)\|_\infty
        \leq \|Q-Q'\|_\infty.
    \end{align*}

    \item The result and its proof are identical to Lemma \ref{lem:F} (3), which is proved in Appendix \ref{app:F_IS}. The proof is hence omitted. 

    \item For any $Q\in\mathbb{R}^{mn}$ and $u = (s_1,a_1,s_2)\in\mathcal{U}$, we have 
    \begin{align*}
    &[F_{\mathrm{ETD}}(Q, u, \pi_{t_1})-F_{\mathrm{ETD}}(Q, u, \pi_{t_2})](s,a)\\ 
    &= \mathds{1}_{\{(s,a)=(s_1, a_1)\}}\gamma\left(\sum_{a\in\mathcal{A}}\pi_{t_1}(a\mid s_2)Q(s_2, a)-\sum_{a\in\mathcal{A}}\pi_{t_2}(a\mid s_2)Q(s_2, a)\right).
    \end{align*}
    Therefore, we have
    \begin{align*}
        \infnorm{F_{\mathrm{ETD}}(Q, u, \pi_{t_1})-F_{\mathrm{ETD}}(Q, u, \pi_{t_2})} 
        &\leq \gamma\|\pi_{t_1}-\pi_{t_2}\|_1\infnorm{Q}\\
        &\leq 2\max_{s}\mathrm{d_{TV}}\big(\pi_{t_1}(s), \pi_{t_2}(s)\big)\infnorm{Q}.
    \end{align*}
    The result then follows by combining the above with Lemma \ref{bounds_ETD} (4).

\end{enumerate}

\section{Discussion of Our Assumption}\label{app:assum}

In this section, we first justify the minimality of Assumption~\ref{assum:behavior} and discuss related assumptions in the literature.

\begin{proposition}\label{prop:minimal}
Assumption~\ref{assum:behavior} is minimal for state-space exploration in the following sense: if no policy induces an irreducible state trajectory on the full state space $\mathcal S$, then no algorithm based on a single trajectory of Markovian samples can visit every state infinitely often.
\end{proposition}

\begin{proof}
Consider the directed graph $\mathcal{G}$ on $\mathcal{S}$, where there is an edge from $s$ to $s'$ if there exists an action $a\in\mathcal{A}$ such that $p(s'\mid s,a)>0$. If $\mathcal{G}$ is strongly connected, then any positive policy\footnote{We call a policy $\pi$ positive if $\pi(a\mid s)>0$ for all $s\in\mathcal{S}$ and $a\in\mathcal{A}$} induces an irreducible Markov chain on $\mathcal{S}$. Therefore, if no policy induces an irreducible state trajectory, then $\mathcal{G}$ is not strongly connected.

Now consider any algorithm that generates a single trajectory $\{S_t,A_t\}_{t\geq0}$. If this trajectory visits every state infinitely often, then for any two states $s,s'\in\mathcal{S}$, there exist times $t_1<t_2$ such that $S_{t_1}=s$ and $S_{t_2}=s'$. The realized transitions from time $t_1$ to time $t_2$ form a directed path from $s$ to $s'$ in $\mathcal{G}$. Since every state is visited infinitely often, the same argument gives a directed path from $s'$ to $s$. Hence $\mathcal{G}$ must be strongly connected, which contradicts the fact that $\mathcal{G}$ is not strongly connected. Therefore, no single-trajectory algorithm can generate a sample path that visits every state infinitely often.
\end{proof}

We next show that if the behavior policy is positive, then Assumption~\ref{assum:behavior} implies that the Markov chain induced by the behavior policy is irreducible.

\begin{lemma}\label{lem:positive_equivalence}
The following are equivalent:
\begin{enumerate}[(1)]
\item There exists a policy $\pi$ whose induced state trajectory is irreducible.
\item Every positive policy $\pi_b$ induces an irreducible state trajectory.
\end{enumerate}
\end{lemma}

\begin{proof}
The implication (2)$\Rightarrow$(1) is immediate. For (1)$\Rightarrow$(2), suppose that $\pi$ induces an irreducible Markov chain on $\mathcal{S}$, and let $\pi_b$ be any positive policy. Since $\pi_b(a\mid s)>0$ for all $(s,a)$ and the state-action space is finite, there exists $\alpha\in(0,1)$ such that
$\pi_b(a\mid s)\geq \alpha \pi(a\mid s)$ for all $(s,a)$. Therefore, for any $s,s'\in\mathcal{S}$,
\begin{align*}
P_{\pi_b}(s,s')
=
\sum_{a\in\mathcal{A}}\pi_b(a\mid s)p(s'\mid s,a)
\geq
\alpha \sum_{a\in\mathcal{A}}\pi(a\mid s)p(s'\mid s,a)
=
\alpha P_{\pi}(s,s').
\end{align*}
It follows that, for any $T\geq1$,
\begin{align*}
P_{\pi_b}^T(s,s')\geq \alpha^T P_{\pi}^T(s,s').
\end{align*}
Since the Markov chain induced by $\pi$ is irreducible, for every pair $s,s'\in\mathcal{S}$ there exists $T$ such that $P_{\pi}^T(s,s')>0$. Hence $P_{\pi_b}^T(s,s')>0$, which implies that the Markov chain induced by $\pi_b$ is irreducible.
\end{proof}

\paragraph{Lazy-chain construction \cite{schweitzer1971iterative}.}

The preceding lemma shows that Assumption~\ref{assum:behavior}, together with the positivity of $\pi_b$, implies irreducibility of the Markov chain induced by $\pi_b$. We now explain why aperiodicity can be imposed without loss of generality through the standard lazy-chain construction.

Given a state-action trajectory $\{S_t,A_t\}$ generated by a policy $\pi$, let $P_\pi$ denote the corresponding state-action transition kernel. Fix $\lambda\in(0,1)$. The lazy version of this chain is defined by using the transition kernel
\begin{align*}
P'_\pi := (1-\lambda)P_\pi+\lambda I.
\end{align*}
Equivalently, at each step the chain stays at the current state-action pair with probability $\lambda$, and follows the original transition kernel $P_\pi$ with probability $1-\lambda$. This construction can be implemented online by repeating each observed state-action-reward tuple for a geometrically distributed number of times before advancing to the next transition.

It is clear that $P'_\pi$ has the same stationary distribution as $P_\pi$. Moreover, if $P_\pi$ is irreducible, then $P'_\pi$ is also irreducible. Since $P'_\pi$ has a self-loop probability at least $\lambda$ at every state-action pair, it is aperiodic. Therefore, on a finite state-action space, irreducibility and aperiodicity imply geometric mixing \cite[Theorem 4.9]{levin2017markov}.

\subsection{Comparison with Assumptions from Existing Literature}\label{app:comparison}

We now compare Assumption~\ref{assum:behavior} with commonly used assumptions in the actor--critic literature. We focus on three aspects: (i) mixing requirements, (ii) exploration conditions, and (iii) policy smoothness assumptions.

\paragraph{Uniform ergodicity across policies.}

A frequently adopted assumption concerns uniform geometric mixing of the Markov chains induced by a class of policies. Let $\Theta$ be a parameter space for policies. The following assumption appears frequently in the literature \cite[Assumption 4.2]{wu2020finite}, \cite[Assumption 2]{olshevsky2023small}, \cite[Assumption 2]{xu2020twotimescaleNAC}, \cite[Assumption 3.2]{chen2023actorcritic}.

\begin{assumption}\label{assum:uniform_ergodicity}
For any policy parameter $\theta\in\Theta$, the policy $\pi_\theta$ induces a unique stationary distribution $\mu_\theta$ on the state trajectory. Moreover, there exist constants $c>0$ and $\rho\in(0,1)$ such that
\begin{align*}
\sup_{\theta\in\Theta}\max_{s\in\mathcal{S}}
\mathrm{d_{TV}}\big(P_{\pi_\theta}^t(s,\cdot),\mu_\theta(\cdot)\big)
\leq c\rho^t,
\quad \forall t\geq0,
\end{align*}
where $P_{\pi_\theta}$ denotes the state transition matrix induced by $\pi_\theta$.
\end{assumption}

Assumption~\ref{assum:uniform_ergodicity} is substantially stronger than Assumption~\ref{assum:behavior}. It not only requires every policy in the parameterized class to induce an ergodic Markov chain, but also imposes a common geometric mixing rate over the entire class. Even if a similar condition is imposed only on the policies generated along the algorithm trajectory, it remains algorithm-dependent, since verifying it requires controlling the very policy sequence whose convergence is being analyzed.

In contrast, Assumption~\ref{assum:behavior} is purely structural: it only requires the existence of one policy that induces an irreducible Markov chain on the state space. As shown in Proposition~\ref{prop:minimal}, this condition is minimal for state-space exploration. In our analysis, the existence of a stationary distribution and geometric mixing follow from Assumption~\ref{assum:behavior}, the positivity of the fixed behavior policy $\pi_b$, and the lazy-chain construction above.

\paragraph{Exploration conditions.}

In analyses involving linear function approximation, a common exploration-type condition is imposed through the feature covariance structure. Let $\phi:\mathcal{S}\times\mathcal{A}\to\mathbb{R}^d$ be the feature map, and let $\mu_\theta$ be the stationary distribution of the state trajectory under $\pi_\theta$. The following condition appears frequently in equivalent forms \cite[Assumption 3.1]{kumar2026single}, \cite[Assumption 3.1]{kumarconvergence}, \cite[Assumption 6]{olshevsky2023small}, \cite[Assumption 4.1]{wu2020finite}.

\begin{assumption}\label{assum:exploration}
Let
\begin{align*}
A_\theta
=
\mathbb{E}_{s\sim\mu_\theta,\,a\sim\pi_\theta(\cdot\mid s),\,s'\sim p(\cdot\mid s,a),\,a'\sim\pi_\theta(\cdot\mid s')}
\big[\phi(s,a)\big(\gamma\phi(s',a')-\phi(s,a)\big)^\top\big].
\end{align*}
There exists $\lambda>0$ such that $A_\theta+\lambda I$ is negative semidefinite for all $\theta\in\Theta$.
\end{assumption}

The following lemma shows that, in the tabular case, this condition fails near optimal policies whenever a suboptimal action exists.

\begin{lemma}\label{lem:exploration_fail}
Consider the tabular case. Suppose that there exist a state $s\in\mathcal{S}$ and actions $a_1,a_2\in\mathcal{A}$ such that $Q^*(s,a_1)<Q^*(s,a_2)$. Let $\pi^*$ be an optimal policy satisfying $\pi^*(a_1\mid s)=0$. Then, for any $\lambda>0$, there exists $\delta>0$ such that $A_\pi+\lambda I$ is not negative semidefinite for every policy $\pi$ satisfying $\mathrm{d_{TV}}(\pi,\pi^*)<\delta$. Consequently, global convergence to an optimal policy and the validity of Assumption~\ref{assum:exploration} along the algorithm trajectory cannot hold simultaneously for general tabular MDPs.
\end{lemma}

\begin{proof}
In the tabular setting, the feature map is the identity over state-action pairs, and hence
\begin{align*}
A_\pi = D_\pi(\gamma \widehat{P}_\pi-I),
\end{align*}
where $D_\pi=\mathrm{diag}\big(\mu_\pi(s)\pi(a\mid s)\mid (s,a)\in\mathcal{S}\times\mathcal{A}\big)$ and $\widehat{P}_\pi$ is the state-action transition kernel induced by $\pi$. Let $e(s,a_1)$ denote the standard basis vector corresponding to the state-action pair $(s,a_1)$. Then
\begin{align*}
e(s,a_1)^\top A_\pi e(s,a_1)
=
\mu_\pi(s)\pi(a_1\mid s)\big(\gamma \widehat{P}_\pi((s,a_1),(s,a_1))-1\big).
\end{align*}
Since $0\leq \widehat{P}_\pi((s,a_1),(s,a_1))\leq1$, we have
\begin{align*}
e(s,a_1)^\top A_\pi e(s,a_1)\geq -\mu_\pi(s)\pi(a_1\mid s)\geq -\pi(a_1\mid s).
\end{align*}
Because $\pi^*(a_1\mid s)=0$, there exists $\delta>0$ such that $\mathrm{d_{TV}}(\pi,\pi^*)<\delta$ implies $\pi(a_1\mid s)<\lambda/2$. Therefore,
\begin{align*}
e(s,a_1)^\top (A_\pi+\lambda I)e(s,a_1)
=
e(s,a_1)^\top A_\pi e(s,a_1)+\lambda
>
\frac{\lambda}{2}>0.
\end{align*}
Thus $A_\pi+\lambda I$ is not negative semidefinite for any such policy $\pi$. The final claim follows because any algorithm that globally converges to $\pi^*$ must eventually enter this neighborhood, where Assumption~\ref{assum:exploration} fails.
\end{proof}

Since our analysis focuses on the tabular case, we do not impose such spectral exploration conditions. Instead, exploration is ensured through the positive behavior policy. Lemma~\ref{lem:exploration_fail} shows that Assumption~\ref{assum:exploration} is incompatible with global convergence in general tabular MDPs whenever optimality requires eliminating suboptimal actions.

\paragraph{Policy smoothness assumptions.}

Most of the actor--critic literature adopts an optimization viewpoint and analyzes the actor through policy gradients. Such analyses typically require Lipschitz continuity of policies and log-policy gradients. The following assumption appears frequently \cite[Assumption 3.3]{chen2023actorcritic}, \cite[Assumption 3]{olshevsky2023small}, \cite[Assumption 4.3]{wu2020finite}.

\begin{assumption}\label{assum:smooth}
The following hold for all $\theta,\theta_1,\theta_2\in\Theta$, $s\in\mathcal{S}$, and $a\in\mathcal{A}$:
\begin{enumerate}[(1)]
\item There exists $B>0$ such that
$\|\nabla_\theta\log\pi_\theta(a\mid s)\|\leq B$.
\item There exists $L>0$ such that
$\|\nabla_\theta\log\pi_{\theta_1}(a\mid s)-\nabla_\theta\log\pi_{\theta_2}(a\mid s)\|
\leq L\|\theta_1-\theta_2\|$.
\item There exists $L'>0$ such that
$|\pi_{\theta_1}(a\mid s)-\pi_{\theta_2}(a\mid s)|
\leq L'\|\theta_1-\theta_2\|$.
\end{enumerate}
\end{assumption}

In the tabular case, where the policy is parameterized by itself, i.e., $\theta=\pi$, the first condition fails because
$\frac{\partial}{\partial \pi(a\mid s)}\log\pi(a\mid s)=1/\pi(a\mid s)$ can be arbitrarily large as $\pi(a\mid s)$ approaches zero. Similarly, since the map $x\mapsto 1/x$ is not Lipschitz near zero, the second condition also fails. Although the third condition is benign under this parametrization, the first two conditions rule out policies approaching the boundary of the simplex. This is restrictive for global optimality analysis in the tabular setting, where optimal policies are often deterministic and hence lie on the boundary of the policy simplex.

\end{document}